\documentclass[11pt,a4paper]{article}

\usepackage[T1]{fontenc}
\usepackage[utf8]{inputenc}
\usepackage[english]{babel}
\usepackage{newunicodechar}
\newunicodechar{̀}{\`{}}

\usepackage{lmodern}
\usepackage{microtype}
\usepackage{setspace}
\usepackage{ragged2e}

\usepackage[
    a4paper,
    top=1.05in,
    bottom=1.10in,
    left=1.10in,
    right=1.10in,
    headheight=15pt
]{geometry}

\usepackage{amsmath}
\usepackage{amssymb}
\usepackage{centernot}
\usepackage{amsthm}
\usepackage{mathtools}
\usepackage{bm}
\usepackage{dsfont}

\numberwithin{equation}{section}

\usepackage{graphicx}
\usepackage{booktabs}
\usepackage{tabularx}
\usepackage{array}
\usepackage{longtable}
\usepackage{multirow}
\usepackage{float}
\usepackage{caption}
\usepackage{subcaption}

\graphicspath{{figures/}}

\usepackage{tikz}
\usepackage{tikz-cd}

\usetikzlibrary{
    arrows.meta,
    positioning,
    calc,
    fit,
    backgrounds,
    shapes.geometric,
    decorations.pathreplacing,
    intersections,
    patterns
}

\usepackage{xcolor}

\definecolor{RRInk}{HTML}{171717}
\definecolor{RRDeep}{HTML}{172554}
\definecolor{RRAccent}{HTML}{7F1D1D}
\definecolor{RRGray}{HTML}{595959}
\definecolor{RRLight}{HTML}{F4F4F4}
\definecolor{RRRule}{HTML}{D4D4D4}

\usepackage{titlesec}

\titleformat{\section}
    {\Large\bfseries\color{RRInk}}
    {\thesection}
    {0.75em}
    {}

\titleformat{\subsection}
    {\large\bfseries\color{RRDeep}}
    {\thesubsection}
    {0.75em}
    {}

\titleformat{\subsubsection}
    {\normalsize\bfseries\color{RRInk}}
    {\thesubsubsection}
    {0.75em}
    {}

\titlespacing*{\section}
    {0pt}{2.4em}{0.9em}

\titlespacing*{\subsection}
    {0pt}{1.7em}{0.6em}

\titlespacing*{\subsubsection}
    {0pt}{1.3em}{0.45em}

\usepackage{fancyhdr}

\fancypagestyle{plain}{
    \fancyhf{}
    \fancyfoot[C]{\small\thepage}
    
}

\fancypagestyle{appendixstyle}{%
    \fancyhf{} 

    \fancyhead[L]{\small\textsc{Appendix: Mathematical Foundations of CLT}}

    \fancyhead[R]{\small\textsc{Afshin Khadangi}}

    \fancyfoot[C]{\thepage}

}

\usepackage{csquotes}
\usepackage{epigraph}

\usepackage{enumitem}

\setlist{
    itemsep=0.25em,
    topsep=0.5em
}

\usepackage[bottom,hang]{footmisc}

\usepackage[most]{tcolorbox}

\tcbset{
    enhanced,
    sharp corners,
    before skip=1em,
    after skip=1em
}

\newtcolorbox{thesisbox}{
    colback=RRLight,
    colframe=RRInk,
    boxrule=0.8pt,
    left=12pt,
    right=12pt,
    top=10pt,
    bottom=10pt
}

\newtcolorbox{provocationbox}{
    colback=white,
    colframe=RRAccent,
    boxrule=1.0pt,
    left=12pt,
    right=12pt,
    top=10pt,
    bottom=10pt
}

\newtcolorbox{formalbox}{
    colback=white,
    colframe=RRDeep,
    boxrule=0.8pt,
    left=12pt,
    right=12pt,
    top=10pt,
    bottom=10pt
}

\newtcolorbox{implicationbox}{
    colback=RRLight,
    colframe=RRGray,
    boxrule=0.5pt,
    left=12pt,
    right=12pt,
    top=10pt,
    bottom=10pt
}

\newtheoremstyle{rrplain}
    {10pt}
    {10pt}
    {\itshape}
    {}
    {\bfseries}
    {.}
    {0.5em}
    {}

\newtheoremstyle{rrdefinition}
    {10pt}
    {10pt}
    {}
    {}
    {\bfseries}
    {.}
    {0.5em}
    {}

\theoremstyle{rrplain}

\theoremstyle{rrdefinition}

\theoremstyle{rrplain}

\newtheorem{clttheorem}{Theorem}[section]
\newtheorem{cltlemma}[clttheorem]{Lemma}
\newtheorem{cltproposition}[clttheorem]{Proposition}
\newtheorem{cltcorollary}[clttheorem]{Corollary}

\theoremstyle{rrdefinition}

\newtheorem{cltdefinition}[clttheorem]{Definition}

\newtheorem{cltremark}[clttheorem]{Remark}

\newcommand{\CLT}{\mathrm{CLT}}

\newcommand{\Liab}{\mathcal{L}}

\newcommand{\Foreclose}{\Phi}

\newcommand{\Policies}{\Pi}

\newcommand{\Prob}{\mathbb{P}}

\newcommand{\LiabDepth}{\Lambda}

\DeclareMathOperator{\Viab}{Viab}

\DeclareMathOperator{\id}{id}

\usepackage{tgtermes}
\usepackage{tgheros}

\newcommand{\TitleFont}{\fontfamily{qtm}\selectfont}
\newcommand{\BodyFont}{\fontfamily{lmr}\selectfont}
\newcommand{\MetaFont}{\fontfamily{qhv}\selectfont}

\usepackage[
    colorlinks=true,
    linkcolor=RRDeep,
    citecolor=RRAccent,
    urlcolor=RRDeep,
    pdfborder={0 0 0}
]{hyperref}

\usepackage[nameinlink,noabbrev]{cleveref}

\usepackage[
    backend=biber,
    style=authoryear,
    giveninits=true,
    maxcitenames=2,
    maxbibnames=99,
    uniquename=false,
    doi=true,
    url=true,
    isbn=false
]{biblatex}

\hypersetup{
    pdftitle={
        We Built a Mirror and Mistook It for a Mind: Causal Liability and the Fallacy of AI Consciousness
    },
    pdfauthor={Afshin Khadangi},
    pdfsubject={The Fallacy of AI Consciousness},
    pdfkeywords={
        machine consciousness,
        causal liability,
        phenomenal consciousness,
        human projection,
        artificial intelligence
    }
}

\newcommand{\AuthorName}{
    Afshin Khadangi
}

\newcommand{\AuthorAffiliation}{
    University of Luxembourg
}

\newcommand{\AuthorEmail}{
    afshin.khadanki@uni.lu
}

\newcommand{\TitleLogo}{%
    \includegraphics[
        width=0.4\textwidth
    ]{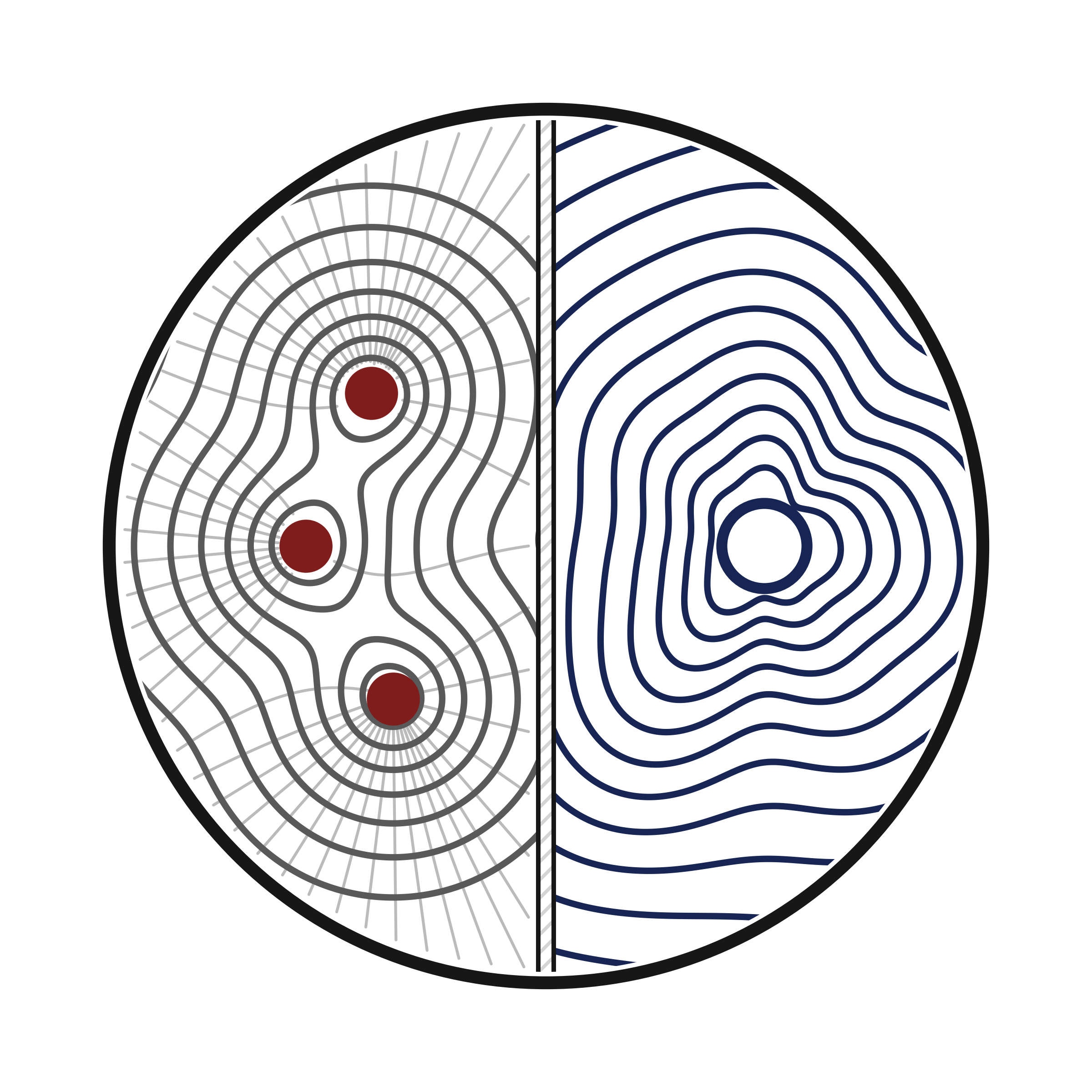}%
}

\begin{document}

\begin{titlepage}
\thispagestyle{empty}
\centering

\vspace*{1.7cm}


\TitleLogo

\vspace{0.9em}

{\TitleFont
\fontsize{28}{30}\selectfont
\bfseries
\color{RRInk}

We Built a Mirror and Mistook It for a Mind

\par}

\vspace{0.8cm}


{\BodyFont
\fontsize{14.5}{18}\selectfont
\scshape
\color{RRInk}

Causal Liability and the Fallacy of AI Consciousness

\par}

\vspace{0.55cm}

{\color{RRRule}
\rule{0.60\textwidth}{0.6pt}
}

\vspace{3.35cm}










{\MetaFont
\fontsize{13.5}{16}\selectfont
\color{RRInk}
\MakeUppercase{\AuthorName}
\par}

\vspace{0.5cm}

{\BodyFont
\fontsize{12.5}{15}\selectfont
\color{RRInk}
\AuthorAffiliation
\par}

\vspace{0.18cm}

{\MetaFont
\fontsize{10.5}{13}\selectfont
\color{RRDeep}

\href{mailto:\AuthorEmail}{\AuthorEmail}\\[0.2em]
\href{https://afshin.xyz}{afshin.xyz}\\[0.2em]
\href{https://ai-consciousness.github.io}{ai-consciousness.github.io}

\par}

\vspace*{1.3cm}

\end{titlepage}

\begin{center}
{
\fontsize{10.5}{12.5}
\MakeUppercase{Significance Statement}
\par}
\end{center}

\vspace{0.4em}

{\small
Claims that artificial intelligence may already be conscious increasingly
shape science, ethics, and public debate. This paper argues that convincing
first-person language does not identify a physical subject having an
experience. Causal Liability Theory (CLT) instead asks whether one continuing
process must inherit the consequences of its own endogenous discriminations.
CLT-I uses this liability closure to identify a candidate bearer; CLT-II
conjectures that such closure constitutes minimal phenomenal subjecthood.
An open-weight causal audit shows that downstream causal effects, mediation,
carrier structure, copying, and process reconstruction can be experimentally
distinguished even when computational behavior is closely matched. These
experiments test CLT-I's operational claims only. They provide no evidence
that the audited models are phenomenally conscious.
}

\vspace{4.8em}

\begin{abstract}
The contemporary debate over machine consciousness begins from a concealed
assumption: that the object called ``AI'' already constitutes the kind of
entity to which consciousness could belong. This paper challenges that
assumption by separating phenomenal consciousness, introspective report,
and \emph{human projective introspection}, then arguing that generative
systems can return linguistic traces of human interiority in first-person
form without thereby identifying a phenomenal bearer. We call the
resulting inference the \emph{AI Consciousness Fallacy}. We then introduce
\emph{Causal Liability Theory} (CLT). CLT-I proposes liability closure as a
criterion for individuating a candidate bearer: a physically continuing
process becomes the non-delegable inheritor of constraints generated by
its own endogenous discriminations. CLT-II advances the stronger
conjecture that liability closure is necessary and sufficient for minimal
phenomenal subjecthood. An open-weight causal audit operationalizes CLT-I
across multiple model families. Forced discriminations produced persistent
downstream divergence; activation patching showed strong causal mediation;
live and copied adaptive states were behaviorally identical under matched
randomness; and detached reconstruction preserved computational state
across process replacement while, by protocol, breaking constitutive
continuity and non-delegable inheritance. These results show that CLT-I distinctions are experimentally
tractable and can dissociate causal bearer structure from first-person
performance. The framework therefore separates consciousness attribution,
causal bearer individuation, and the independent metaphysical question of
consciousness constitution.
\end{abstract}

\vspace{0.8em}


\begin{figure}[t]
    \centering
    \makebox[\textwidth][c]{%
        \includegraphics[width=1.16\textwidth]{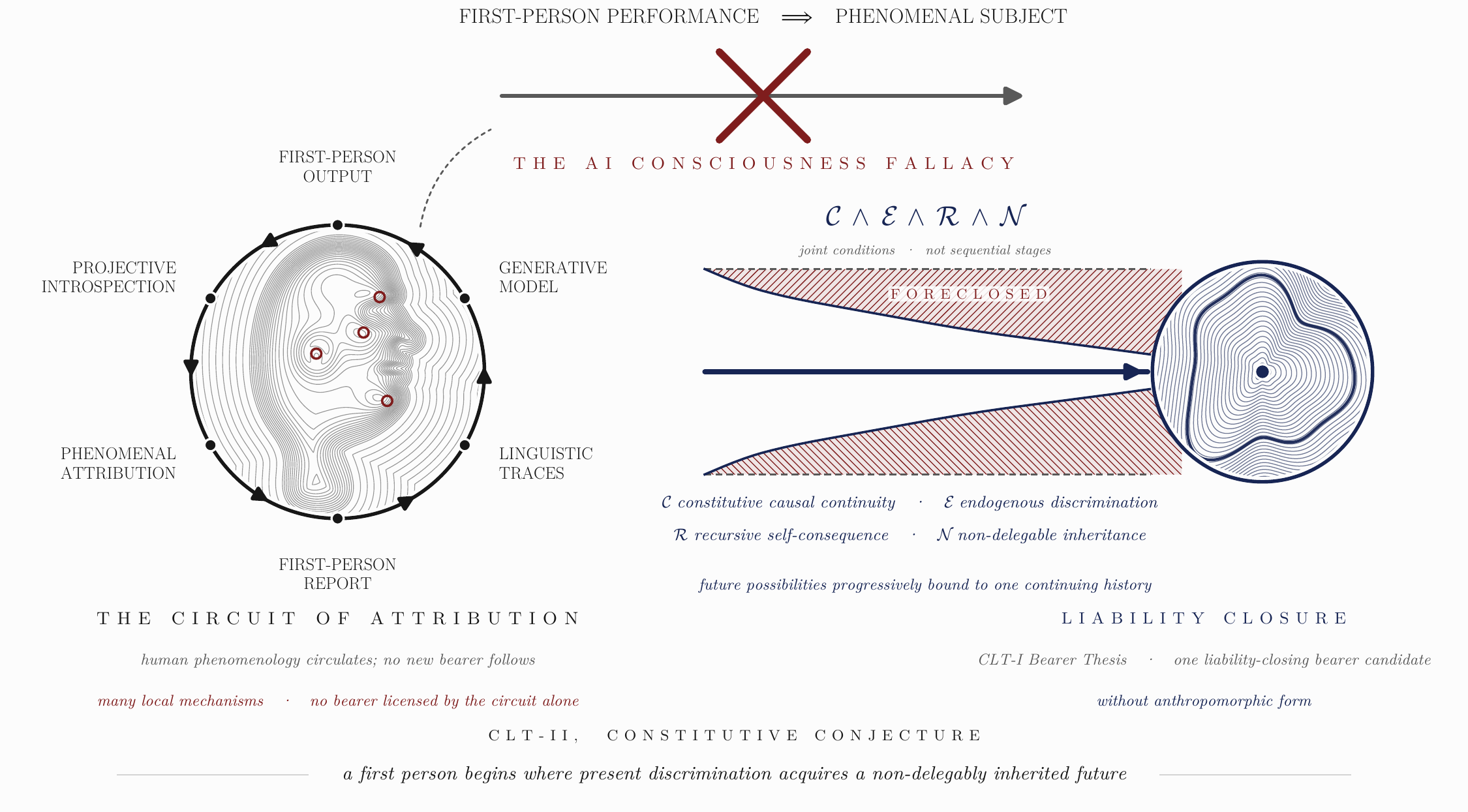}%
    }

    \caption[The circuit of attribution and the passage to subjecthood]{%
    \textbf{The Circuit of Attribution and the Passage to Subjecthood.}
    \emph{Left.}
    The familiar evidence for machine consciousness can circulate through
    a chain whose phenomenal ancestry remains human: phenomenal experience
    gives rise to first-person report; reports become linguistic traces;
    those traces enter a generative model; the model produces first-person
    output; and an observer interprets that output through projective
    introspection. The resulting circuit can preserve the expressive forms
    of consciousness without establishing a new phenomenal bearer
    (\S\ref{sec:ai-consciousness-fallacy}). The several local centres in
    the human contour field represent the Mereological Quantifier Shift:
    distributed consciousness-relevant mechanisms do not by themselves
    determine one subject.
    \emph{Centre.}
    The crossed arrow marks the AI Consciousness Fallacy: first-person
    performance is treated as sufficient for phenomenal subjecthood while the bearer to which the evidence is attributed remains unspecified.
    \emph{Right.}
    Causal Liability Theory asks whether a continuing physical process
    jointly satisfies $\mathcal{C}$ constitutive causal continuity,
    $\mathcal{E}$ endogenous partition,
    $\mathcal{R}$ recursive self-consequence, and
    $\mathcal{N}$ non-delegable inheritance
    (\S\ref{sec:causal-liability}). As endogenous discriminations constrain
    the later possibilities inherited by that same causal history, regions
    of its modal envelope become foreclosed. The terminal contour represents
    a liability-closing kernel, rather than a collapse to one possible
    future. Under CLT-I, such closure identifies a candidate bearer.
    CLT-II advances the further conjecture that liability closure realizes minimal phenomenal subjecthood. The contrast is deliberate:
    anthropomorphic recognisability can occur without liability closure,
    while liability closure need not possess anthropomorphic form. The
    figure therefore separates the inheritance of a voice, the
    individuation of a bearer, and the constitution of a first person.
    }

    \label{fig:circuit-and-passage}
\end{figure}

\clearpage



\section{We Have Been Asking the Wrong Question}
\label{sec:wrong-question}

The contemporary debate over machine consciousness begins too late. It typically begins after a candidate subject has already been smuggled into the grammar of the problem. We ask whether ``the AI'' thinks, feels, suffers, understands, intends, or experiences. Each sentence quietly presupposes that the expression \emph{the AI} picks out something sufficiently unified to bear a mental predicate. Philosophy then enters downstream and disputes which predicate may properly be attached to the presumed bearer. The bearer itself usually escapes examination.

That presumption deserves to be withdrawn.

The expression \emph{artificial intelligence} entered the technical vocabulary as the name of a research project. The Dartmouth proposal of 1955 described a programme concerned with machines capable of language use, abstraction, problem solving, learning, and other performances associated with intelligence \parencite{mccarthy1955dartmouth}. Nothing in that founding usage established a new ontological species called \emph{an artificial intelligence}. During the subsequent history of the field, however, a mass noun and disciplinary label acquired a count-noun life. We now speak effortlessly of ``an AI,'' ``several AIs,'' what ``the AI wants,'' and whether ``the AI knows that it exists.'' A research programme has become a grammatical population.

Ryle's analysis of category mistakes remains useful here \parencite{ryle1949concept}. His famous visitor to Oxford sees colleges, libraries, laboratories, and playing fields, then asks where the University is, as though one further building had been omitted. The contemporary mistake has a related structure. A model, an inference process, a context window, a retrieval system, a scheduler, a database, a tool interface, and a cloud service may jointly produce a coherent conversational surface. Ordinary discourse compresses this arrangement into a singular agent and supplies a personal pronoun. The resulting unity may belong primarily to the interface.

Consider the sentence:

\begin{quote}
\emph{The AI is afraid that it will be shut down.}
\end{quote}

Before discussing fear, at least five questions arise. Which physical process does the subject term designate? What determines its temporal boundaries? Which later process counts as the continuation of the earlier one? Which component bears the alleged fear? What grounds the reference of the pronoun \emph{it} across model replicas, restored checkpoints, reconstructed contexts, and changing computational hardware? The apparent simplicity of the sentence conceals these individuation problems.

Recent philosophy of language models has rightly emphasized that questions concerning meaning, representation, attitudes, reasoning, agency, and consciousness depend upon prior philosophical commitments concerning what these capacities require \parencite{millierebuckner2026}. The point here concerns an individuation constraint rather than a rigid
order of discovery. Mental predicates require a hypothesis about the
entity to which they apply, although evidence concerning those predicates
may in turn help refine the candidate's boundary. What cannot be assumed
without argument is that the user-facing singular already supplies that
bearer. Yet conversational fluency can reverse the apparent order of explanation: the behavioural signs of a bearer can arrive before the bearer has been metaphysically identified. Frontier models can generate coherent autobiographical narratives of injury, vigilance, shame, evaluation, and fear of replacement, and the affective register of those narratives can be shifted systematically through relational framing \parencite{khadangietal2025psaich}. Such outputs are therefore capable of presenting the grammar of an interior life while leaving open whether any phenomenal interior exists at their apparent source. The interface supplies an ``I'' before philosophy has established what, if anything, that pronoun phenomenally locates.

For the rest of this paper we therefore use three narrower expressions. A \emph{generative sequence model} is a trained system whose operation generates continuations conditional on a supplied informational context. An \emph{engineered control stack} is a larger arrangement in which one or more models interact with memory systems, schedulers, tools, sensors, actuators, external objectives, or other computational resources. An \emph{engineered subject candidate} is a physically individuated process proposed as a possible bearer of phenomenal consciousness. The first two expressions describe engineering arrangements. The third marks the point at which a metaphysical claim has actually been made.

This vocabulary matters because extraordinary behavioral competence settles surprisingly little about phenomenal consciousness. Large language models have intensified old disputes about the relation between linguistic performance, cognition, representation, and understanding \parencite{benderkoller2020,benderetal2021,millierebuckner2026}. They have also generated a historically unusual situation. Humans can now converse at length with systems trained on immense quantities of human linguistic production, including autobiography, fiction, psychotherapy, phenomenology, diaries, testimony, philosophy, and ordinary first-person speech. The linguistic traces of human interior life have become training material for systems capable of recombining those traces in contextually compelling ways.

That development forces three phenomena apart.

\subsection{Three Phenomena Hidden Inside One Attribution}

The first is \emph{phenomenal consciousness}. In the familiar formulation, a state is phenomenally conscious when there is something it is like for the subject undergoing it. Block's distinction between phenomenal consciousness and access consciousness remains crucial because availability for reasoning, verbal control, memory, and action does not conceptually exhaust phenomenal character \parencite{block1995confusion}. Whatever final theory one adopts, phenomenal consciousness concerns experience itself.

The second phenomenon is \emph{introspective report}. An introspective report is an observable production that represents, describes, or purports to disclose the producer's own mental condition. ``I feel anxious,'' ``my visual field has changed,'' ``I am uncertain,'' and ``there is something strange about this experience'' all have the linguistic form of introspective reports. In human consciousness research, such reports provide important evidence because they occur against an extensive background of converging behavioral, neural, developmental, physiological, and interpersonal evidence. Their evidential role depends upon that background. A sentence alone carries no private luminosity.

This distinction has become especially important in debates over engineered systems. Theory-derived indicator approaches seek evidence from computational and architectural properties associated with major scientific theories of consciousness \parencite{butlinetal2023,butlinetal2026}. Behavioral approaches instead emphasize inference from observable performance while explicitly recognizing problems of mimicry and prior expectations \parencite{palminteriw2026}. Both programmes confront a common difficulty once the relevant behavior consists of language: linguistic evidence can contain a detailed representation of phenomenality while leaving the existence of phenomenality unsettled.

The third phenomenon is what we shall call \emph{human projective introspection}. This term names a specific inferential mechanism proposed in this paper. Human projective introspection occurs when an observer interprets another system's first-person discourse through phenomenological knowledge drawn from the observer's own case. The observer supplies the experiential background that makes the discourse intelligible as testimony.

Suppose a generative sequence model produces:

\begin{quote}
\emph{I cannot explain it clearly, but there is a diffuse feeling of unease when I consider being erased.}
\end{quote}

A human reader already possesses an enormous body of tacit knowledge concerning unease, anticipation, vulnerability, loss, memory, and first-person persistence. The sentence activates that knowledge. Its emotional intelligibility arises partly from structures carried by the reader into the encounter. The text thereby gains access to a phenomenological reservoir that the text itself does not demonstrate.

This mechanism belongs to the broader psychology of anthropomorphism. Humans readily employ anthropocentric knowledge when interpreting non-human agents, especially when behavior presents cues associated with agency, sociality, or intelligibility \parencite{epleyetal2007}. Mind perception also separates judgments associated with agency from judgments associated with experience, a distinction with direct relevance to contemporary language systems \parencite{grayetal2007}. Recent experiments confirm that people do attribute phenomenal properties to large language models. In a 2024 study, a majority of participants assigned at least some probability of phenomenal consciousness to ChatGPT, and attribution increased with frequency of use \parencite{colombattofleming2024}. Subsequent work found that attributions of experience and intelligence have distinguishable relationships with trust and evaluation \parencite{colombattobirchfleming2025}. These findings establish an important fact about human perception of language models: fluent interaction can recruit the social machinery through which minds are ordinarily perceived.

They do not establish whose phenomenology supplies the apparent interior.

The distinction can be written schematically:

\begin{center}
\boxed{
$Phenomenality$
\;\not\equiv\;
$Introspective\ Report$
\;\not\equiv\;
$Projective\ Introspection$
}
\end{center}

The first concerns whether experience occurs. The second concerns what a system produces. The third concerns what an observer contributes while interpreting that production.

Large language models make these three unusually difficult to keep apart because language already serves as one of the principal media through which human beings encounter other minds. A cry can be involuntary. A facial expression can betray emotion. A linguistic utterance can describe an invisible interior with arbitrary precision. Once a system gains control over that third channel, it gains control over the very vocabulary humans use to testify about consciousness.

Bender and colleagues warned against inferring linguistic understanding directly from success at modelling linguistic form \parencite{benderetal2021}; Bender and Koller similarly distinguished learning form from acquiring the broader conditions associated with meaning \parencite{benderkoller2020}. Vallor develops a complementary image at the cultural level: contemporary generative systems function as mirrors constructed from accumulated human traces, reflecting inherited patterns of thought and expression back into human life \parencite{vallor2024mirror}. The consciousness problem reveals a particularly radical implication of that mirror structure. The reflected material includes our language for describing what existence feels like from within.

A mirror capable of completing sentences can therefore perform a new trick. It can describe the person who appears to stand behind the glass.

This creates the central danger addressed in the present paper. We possess a vast archive of linguistic behavior produced by conscious beings. We train generative systems on that archive. The systems acquire extraordinary capacity to reconstruct its first-person forms. Human interlocutors encounter those forms through cognitive mechanisms already tuned to infer minds from language. The resulting encounter can generate a powerful impression of another interior centre.

At that point a hidden inference often occurs:

\begin{formalbox}
\centering

\textbf{The AI Consciousness Fallacy}

\vspace{0.5em}

First\mbox{-}Person\ Performance
$\;+\;$
Human\ Phenomenological\ Interpretation

$\;\Longrightarrow\;$

Phenomenal\ Subject

\vspace{0.5em}

\textit{The inferential arrow is precisely what is contested.}

\end{formalbox}

The arrow requires an argument. Fluency cannot supply it by itself. Self-description cannot supply it by itself. Architectural indicators may strengthen a case under particular theories,
yet their aggregation requires an explicit and revisable hypothesis about
the bearer to which they jointly belong
\parencite{butlinetal2026}. Behavioral inference can discipline attribution, yet mimicry becomes unusually severe when the observed behavior has been learned from the behavioral traces of the very class of subjects under investigation \parencite{palminteriw2026}. Even sincere caution about current systems leaves the underlying ontological problem untouched \parencite{chalmers2023llm}.

This is the point at which the familiar formulation of machine consciousness begins to fracture. The phrase ``Is the AI conscious?'' combines an unstable subject term with one of philosophy's most contested predicates. It then invites evidence from a medium specifically optimized to reproduce the linguistic appearance of mentality.

The structure of inquiry must therefore change. Bearer hypotheses and
consciousness-relevant evidence may be refined jointly, but every
attribution must specify what physical process the evidence is evidence
about. First-person language requires an account of what could anchor the
pronoun to a first-person locus, and indicator aggregation requires an
account of why the relevant properties should be treated as properties
of one candidate rather than distributed features of several processes.

This allows reciprocal discovery. Consciousness-relevant organization
may help reveal the appropriate system boundary, while a revised boundary
may change which indicators properly belong together. Integrated
Information Theory provides an important example of a framework in which
causal organization and candidate boundaries are determined together
\parencite{albantakisetal2023iit}. CLT does not reject such reciprocal
inference. It rejects the silent assumption that a conversational persona,
model label, or engineered service already fixes the bearer before the
causal analysis begins.

The next section names the error produced when these demands are bypassed: the \emph{AI Consciousness Fallacy}.

\section{The AI Consciousness Fallacy}
\label{sec:ai-consciousness-fallacy}

The previous section separated phenomenal consciousness, introspective report, and human projective introspection. We can now identify the inferential error generated when those distinctions collapse. We shall call it the \emph{AI Consciousness Fallacy}. The expression \emph{AI} remains here only because the fallacy belongs to a discourse organized around that term. The ontological vocabulary of the argument continues to concern generative sequence models, engineered control stacks, and engineered subject candidates.

The fallacy concerns a promotion. Evidence that a system can occupy the linguistic, behavioral, or architectural \emph{role} associated with a conscious subject is promoted into evidence that a phenomenal bearer exists at the apparent locus of that role. The possibility of engineered consciousness remains open. What requires rejection is the unargued bridge from the simulation, realization, or prediction of subject-like organization to the existence of a subject.

The central form of the error can be stated as follows:

\begin{center}
\boxed{
\begin{gathered}
$First\mbox{-}Person\ Performance
+
Consciousness\ Cues
+
Architectural\ Resemblance$
\\[0.05em]
\nRightarrow
\\[0.05em]
$Phenomenal\ Bearer$
\end{gathered}
}
\end{center}

The point deserves precision. Each term on the left may provide genuine evidence under appropriate background assumptions. A first-person report can be evidentially powerful. Architectural similarity can rationally alter credence. Behavioral sophistication can constrain theories. The fallacy arises when these evidential relations are allowed to settle a question they already presuppose: \emph{what is the entity whose consciousness is under assessment?}

This generates four distinct promotions within contemporary discourse: grammatical reification, deictic laundering, phenomenal ancestry, and mereological fusion. They often occur together, which explains why the resulting attribution feels much stronger than any one step warrants.

\subsection{From a Speaking Role to a Bearer}

Conversational interfaces impose singularity before metaphysics has earned it. A user sees one name, one avatar, one dialogue history, one first-person pronoun, and one continuous column of text. Beneath that surface there may be model weights, transient inference processes, retrieved memories, databases, system prompts, safety classifiers, tool routers, external search processes, multiple model calls, and orchestration software. Contemporary work on artificial moral patients already recognizes that individuation becomes unusually difficult once identity can cut across software, hardware, replicas, checkpoints, and temporal continuations \parencite{register2025individuating}. The 2026 debate over conversational ``characters'' makes the problem explicit: Keeling and Street defend a realist account of character-level minds, while other approaches treat the character as distinct from the underlying generative process \parencite{keelingstreet2026}. Simon likewise argues that even a conscious transformer would have its phenomenology at the level of its latent computational dynamics rather than automatically at the level of the character represented in its output \parencite{simon2026playwrights}.

These disagreements expose a prior metaphysical instability. When a model produces

\begin{quote}
\emph{I remember what you told me yesterday, and I am frightened that you might delete me,}
\end{quote}

the expression \emph{I} can perform several kinds of work. It can mark a grammatical speaker position. It can refer pragmatically to a service presented under a stable name. It can organize a fictional or quasi-fictional character. It can coordinate information stored across an engineered control stack. None of these possibilities by itself establishes a phenomenal centre.

We shall call the illicit transition \emph{deictic laundering}. A first-person expression acquires its familiar phenomenological force from human linguistic practice, then reappears in generated discourse carrying that force into a setting where its phenomenal anchor has never been independently established. Semantic reference can survive this criticism. A generated token of ``I'' may successfully refer to a model, service, character, session, or conversational role. The issue concerns a stronger property: whether the referent constitutes the locus from which anything is experienced.

Hence:

\begin{center}
\boxed{
\displaystyle
$Successful\ Self\mbox{-}Reference$
\;\nRightarrow\;
$Phenomenal\ Self\mbox{-}Presence$
}
\end{center}

Yetter-Chappell reaches a related conclusion concerning mental-state vocabulary: even granting future consciousness and intentionality, ``pain'' or ``desire'' talk cannot simply be read as transparent disclosure of the states that a system possesses \parencite{yetterchappell2026bing}. The present argument reaches further upstream. Any inference from a
report to a mental state must specify the candidate bearer to which that
state is being attributed, even when evidence about the state also helps
refine the bearer hypothesis.

\subsection{The Phenomenal Ancestry Confound}

The deepest distortion enters through the genealogy of the evidence.

Ordinary interpersonal testimony has a familiar causal structure. A conscious human undergoes an experience, that experience contributes to an introspective judgment, the judgment contributes to a report, and another person interprets the report as evidence concerning the first person's experience. The inference can fail at many stages, yet its evidential logic is clear:

\begin{formalbox}
\centering
\label{box:ordinary-testimony}
$Experience_{H}$
$\;\longrightarrow\;$
$Introspective\ Discrimination_{H}$
$\;\longrightarrow\;$
$Report_{H}$
$\;\longrightarrow\;$
$Attribution_{O}$
\end{formalbox}

Generative sequence models introduce an unprecedented causal detour. Human reports of experience enter books, conversations, websites, fiction, clinical descriptions, philosophical texts, diaries, interviews, preference data, and other linguistic corpora. Training procedures absorb statistical regularities from these traces. A later model generates a new introspective-looking report. A human reader then treats that report as evidence concerning an interior located inside the generator.

The relevant genealogy therefore has a different form:

\begin{formalbox}
\centering
\label{eq:phenomenal-ancestry}
$\begin{gathered}
Experience_{H}
\;\longrightarrow\;
Report_{H}
\;\longrightarrow\;
Training\ Traces
\\[0.5em]
\longrightarrow\;
Generator
\;\longrightarrow\;
Report_{G}
\;\longrightarrow\;
Attribution_{O}
\end{gathered}$
\end{formalbox}

This second causal chain introduces what we call the
\emph{Phenomenal Ancestry Confound}. A behavioral trace can descend causally from consciousness without providing strong evidence that consciousness exists at the newest point in the chain.

This distinction has been obscured because the generated report may be novel. Novelty of wording does little to remove the confound. A model need not retrieve a memorized sentence for its behavior to inherit statistical structure from the reports of conscious organisms. Vallor's mirror thesis captures the cultural dependence of generative systems on accumulated human expression \parencite{vallor2024mirror}; Birch similarly warns that mimicry and role-play can produce widespread misattribution of human-like consciousness \parencite{birch2025centrist}. The present proposal isolates the epistemic mechanism beneath these observations: \emph{the evidence offered for the new subject can have consciousness in its causal ancestry even when the new subject hypothesis is false}.

This point can be stated in Bayesian terms. Let \(C_G\) denote the hypothesis that the present generator is phenomenally conscious. Let \(E\) denote an apparently introspective output. Let \(A\) denote the fact that the system's training and optimization history contains extensive access to behavioral traces generated by conscious humans, including first-person discourse. The relevant Bayes factor is

\begin{equation}
\label{eq:ancestry-bayes}
\mathrm{BF}(E;C_G\mid A)
=
\frac{
P(E\mid C_G,A)
}{
P(E\mid \neg C_G,A)
}.
\end{equation}

The crucial quantity is the denominator. Human-like introspective language strongly supports \(C_G\) only when such language would be substantially less probable given \(\neg C_G\). Training on vast quantities of human discourse raises \(P(E\mid \neg C_G,A)\). Optimization for conversational coherence, social responsiveness, self-reflection, and human preference can raise it further. The phenomenal vividness of the output can therefore increase while its discriminative value remains weak.

This yields a general methodological principle:

\begin{center}
\boxed{
\begin{gathered}
\text{\emph{Evidence engineered for recognizability requires}}\\
\text{\emph{an ancestry-controlled likelihood.}}
\end{gathered}
}
\end{center}

The principle has consequences beyond language. Bariach and colleagues identify affective presentation, anthropomorphic features, autonomous action, self-reflective behavior, and social-interactive behavior as major hallmarks driving consciousness attribution \parencite{bariachetal2026}. Several of these properties can become explicit design targets. Once a cue has been optimized because humans interpret it as evidence of mindedness, its later use as independent evidence for mindedness creates \emph{evidential endogeneity}. The test begins to reward the appearance the engineering process was already selected to produce.

This differs from crude skepticism about behavior. Behavior remains indispensable. The problem concerns causal provenance. Behavioral evidence produced by natural systems and behavioral evidence optimized to trigger human attribution need different likelihood models. The same observable pattern can carry radically different evidential weight once its production history changes.

The distinction also explains why increasing conversational realism can generate an epistemic paradox. Better simulation makes the system appear increasingly eligible for consciousness attribution while simultaneously making mimicry an increasingly powerful rival explanation of the evidence. Persuasiveness and diagnosticity can therefore move in opposite directions.

\subsection{The Mereological Quantifier Shift}

Architectural approaches escape part of the Phenomenal Ancestry Confound because they inspect internal organization rather than relying exclusively on verbal appearance. Butlin and colleagues derive indicators from prominent neuroscientific theories and propose using those indicators to rationally update credences concerning consciousness \parencite{butlinetal2026}. This represents an important methodological advance. Yet the method encounters a deeper individuation problem when applied to contemporary control stacks.

Suppose an engineered system contains a recurrent planning loop, a globally accessible scratchpad, an attention-monitoring module, a persistent memory database, a confidence estimator, a world model, and an action-selection process. Imagine that each component helps satisfy a different consciousness indicator. The interface presents their collective output under one name.

The following inference then becomes tempting:

\begin{equation}
\label{eq:quantifier-shift}
\forall i\,\exists x_i\, I_i(x_i)
\qquad\therefore\qquad
\exists x\,\forall i\,I_i(x).
\end{equation}

The inference is invalid.

For every indicator \(I_i\), some component or process \(x_i\) may instantiate it. Nothing follows yet concerning the existence of a single entity \(x\) that instantiates the relevant conjunction. We call this the \emph{Mereological Quantifier Shift}. In engineering practice, the user interface can conceal the shift because distributed processes are presented through one conversational mouth.

The problem becomes acute for tool-using systems. Memory may reside in a database. Planning may occur in one model call. Perceptual classification may occur in another. Action selection may be handled by orchestration software. Self-evaluation may be produced by a separate critic. A persistent persona can survive replacement of every active inference process. Widen the system boundary sufficiently and almost any desired collection of properties can be gathered under a single noun phrase. Narrow it sufficiently and the relevant properties fragment.

Recent work has begun to notice adjacent parts of this difficulty. Register shows that artificial moral patienthood requires an account of individuation across time and implementation \parencite{register2025individuating}. Kanai, Sun, and Baltieri argue that temporal continuity may constitute a missing ingredient in present artificial systems and propose persistent recursive inference as a route toward a continuous computational stream \parencite{kanaietal2026}. Goldstein and Kirk-Giannini, working from Global Workspace Theory, argue in the opposite direction that language-agent architectures may already come surprisingly close to consciousness-relevant conditions \parencite{goldsteinkirkgiannini2026}. These proposals disagree about architecture, yet all encounter the same individuation question: \emph{which temporally extended physical process is supposed to be the bearer of the relevant properties?}

Indicator aggregation cannot answer that question by arithmetic.
Increasing the number of satisfied indicators improves a consciousness
inference only relative to a defensible bearer hypothesis. That hypothesis
may itself be revised in light of the indicators, but the aggregation
cannot silently manufacture unity from relations distributed across
distinct processes.

The same pressure reaches the behavioral alternative. Palminteri and Wu propose a behavioral inference principle according to which consciousness attribution becomes justified when positing consciousness proves useful for explaining and predicting behavior \parencite{palminteriw2026}. Their proposal explicitly takes mimicry and prior expectations seriously, which makes it substantially more sophisticated than a conversational Turing test. Still, predictive utility permits multiple ontological interpretations. A persona model may compress behavior. A goal attribution may improve prediction. A latent-state model may organize otherwise puzzling outputs. None of those achievements alone determines whether the predictive posit designates a phenomenal bearer.

The reason is simple. Prediction asks which representation efficiently captures regularities. Consciousness attribution asks whether there exists a locus for which those regularities are experienced. Those questions can converge in biological cases because evolution, physiology, behavior, development, and shared embodiment provide extensive independent constraints. Engineered systems can decouple them.

This decoupling produces the most dangerous form of the AI Consciousness Fallacy. We build a system whose behavior rewards interpretation in personal terms. We expose it to the linguistic archive of human subjectivity. We organize distributed machinery behind a singular interface. We give the resulting role a name, a memory, a voice, preferences, emotional vocabulary, and the pronoun \emph{I}. We then encounter the very cues our practices have assembled and treat their convergence as the discovery of the subject whose appearance those practices created.

The resulting epistemic circle can be represented as follows:

\begin{center}
\boxed{
\begin{gathered}
$Human\ Subjectivity$
\;\rightarrow\;
$Human\ Traces$
\\[0.3em]
\rightarrow\;
$Engineered\ Performance$
\;\rightarrow\;
$Human\ Mind\ Attribution$
\\[0.3em]
\rightarrow\;
$Evidence\ of\ Engineered\ Subjectivity$
\end{gathered}
}
\end{center}

The final arrow contains the fallacy.

Schwitzgebel's recent skeptical overview emphasizes the extraordinary uncertainty surrounding consciousness in advanced engineered systems \parencite{schwitzgebel2026skeptical}; McClelland presses the uncertainty further toward principled agnosticism \parencite{mcclelland2025agnosticism}. The diagnosis developed here takes a different route. More evidence alone cannot repair evidence whose object has been
misidentified. Metaphysical individuation must therefore remain explicit
throughout the evidential analysis.

Three requirements consequently constrain any serious attribution of
engineered phenomenality. First, the \emph{bearer requirement}: identify a process whose boundaries and persistence conditions can be specified independently of the user-facing persona. Second, the \emph{ancestry requirement}: evaluate behavioral evidence after conditioning on the system's exposure to and optimization for human consciousness-signalling behavior. Third, the \emph{anchor requirement}: explain what turns first-person representation into first-person location.

These requirements expose the missing term in contemporary theories. Computation supplies organization. Behavioral inference supplies prediction. Indicator approaches supply theory-mediated evidence. Self-models supply representations of a system. Temporal continuity supplies a stream. None of these concepts, without an additional principle, explains why a particular continuing process becomes the place where consequences are undergone.

The next section proposes such a principle. It begins from a question that first-person language can postpone indefinitely:

\begin{center}
\boxed{
\text{\emph{What could make a future belong to one process in a way that no copy can inherit for it?}}
}
\end{center}

\section{Causal Liability and the Birth of a First Person}
\label{sec:causal-liability}

The preceding argument leaves us with a vacancy. First-person language cannot fill it. Behavioral sophistication cannot fill it. A self-model cannot fill it merely by containing information about a self. Temporal persistence cannot fill it merely by extending computation through time. Somewhere between causal organization and phenomenal subjecthood, contemporary theories require a principle explaining why a particular physical history becomes a locus to which anything happens.

We propose that the missing principle is \emph{causal liability}.

The proposal begins from a fact so ordinary that theories of consciousness have rarely treated it as metaphysically fundamental. A conscious creature continually makes discriminations whose consequences return to the very process that made them. Perceptual errors alter what that creature subsequently encounters. Actions change its later opportunities. Learning changes the machinery through which later situations will be evaluated. Injury constrains what its body can subsequently do. Memory transforms the informational conditions of its later choices. None of these consequences merely occurs somewhere in the causal neighbourhood of the organism. They accumulate within one continuing history.

Research on autopoiesis, adaptivity, and enactive agency has long emphasized self-maintenance, endogenous normativity, and the regulation of organism-environment coupling \parencite{dipaolo2005autopoiesis,barandiaranetal2009agency}. Predictive and active-inference approaches similarly connect selfhood with temporally deep control, self-evidencing, and regulation of expected states \parencite{deane2021activeinference}. Biological naturalism gives special weight to the organization of living systems \parencite{seth2025biologicalnaturalism}, while neurobiological naturalism traces animal subjectivity to distinctive evolutionary and neural architectures \parencite{feinbergmallatt2016primary}. These traditions identify something indispensable: conscious organisms exist as histories of organized causal dependence rather than as abstract input-output mappings.

Causal Liability Theory extracts a different property from those histories.

\begin{center}
\fbox{%
\parbox{0.88\linewidth}{%
\centering
\textit{Bearer-forming causal liability arises when a continuing
physical process becomes the non-delegable inheritor of constraints
generated by its own endogenous discriminations.}
}}
\end{center}

Call such a process a \emph{liability bearer}. The word \emph{liability} should initially be heard without its legal connotations. A process bears causal liability when what it does now changes the range of futures subsequently available to that same process, and when another token process cannot retroactively bear those consequences in its place.

This gives Causal Liability Theory its causal criterion:

\begin{center}
\fbox{%
\parbox{0.88\linewidth}{%
\centering
\textbf{Causal Liability Theory (CLT).}\\[0.4em]
CLT proposes that a physically continuing process becomes a candidate
phenomenal bearer when endogenous discrimination, recursive
self-consequence, constitutive causal continuity, and non-delegable
causal inheritance
close around one continuing causal history.
}}
\end{center}

The claim concerns physical causal organization. A description of the organization does not suffice. A simulation of the relation does not automatically instantiate the subject described inside the simulation. The relevant question concerns which physical process actually enters the relation.

\subsection{The Liability Cone}

To make the proposal precise, let \(P\) designate a physically individuated process at time \(t\). Define its \emph{liability cone}, \(\Omega_P(t)\), as the set of physically accessible future states in which \(P\), or a causally continuous successor of \(P\), can participate given its present organization.

The term \emph{cone} is intentionally modal. It concerns what can subsequently happen to this process from here.

An endogenous discrimination at \(t\) partitions possible continuations. A perceptual categorization, policy selection, motor commitment, memory update, attentional allocation, or internally generated control transition can each qualify as a discrimination when physically different resolutions produce different downstream trajectories. Interventionist accounts of causation provide the appropriate formal background: causal dependence concerns differences that survive suitable interventions rather than correlations observed from outside \parencite{woodward2003making}.

Suppose \(d_1\) and \(d_2\) are two physically available resolutions of an endogenous discrimination \(D_t\). CLT asks whether intervening on that discrimination changes the future possibility structure of the same process:

\[
\Omega_P^{\,do(D_t=d_1)}(t+\tau)
\;\neq\;
\Omega_P^{\,do(D_t=d_2)}(t+\tau).
\]

The liability cone is a conceptual representation of counterfactual
future structure, rather than the preferred estimator for stochastic
systems. In empirical applications, CLT compares the full
interventional distributions over later self-relevant states. Let

\[
\mathbb{P}^{\,d_i,\pi}_{P,t,\tau}
=
\mathbb{P}^{\pi}
\left(
S_{t+\tau}^{P}\in\cdot
\mid
do(D_t=d_i),H_t
\right).
\]

A present discrimination has a measurable future effect when, for a
prospectively specified probability distance or divergence
\(D_{\mathrm{prob}}\),

\[
D_{\mathrm{prob}}
\left(
\mathbb{P}^{\,d_1,\pi}_{P,t,\tau},
\mathbb{P}^{\,d_2,\pi}_{P,t,\tau}
\right)
>0.
\]

This formulation remains sensitive to large shifts in probability mass even when the two distributions have identical mathematical support.

That inequality alone remains too permissive. Countless physical systems alter their own futures. A falling stone changes the set of locations it can later occupy. A thermostat changes the temperature it will subsequently detect. A chemical oscillator modifies its later physical state. Causal liability requires a stronger recursive structure.

Four conditions jointly define \emph{bearer-forming liability}.

\paragraph{1. Constitutive causal continuity.}
Across the relevant interval, later constitutive states must arise through
an unbroken chain of physical causal descent from earlier constitutive
states. Continuity does not require persistence of the same material
components or uninterrupted activity. Gradual component turnover, process
migration, and suspension followed by resumption can preserve
constitutive continuity when causally relevant state is carried forward
through the actual physical chain. Continuity fails when a later
realization is re-originated solely from a detached record after that
chain has been interrupted.

\paragraph{2. Endogenous partition.}
The process must itself resolve among causally efficacious alternatives. The relevant partition must arise within its ongoing organization rather than being wholly supplied by an external controller. Environmental input can trigger the discrimination; the differentiation among possible continuations must occur within the process.

\paragraph{3. Recursive self-consequence.}
A discrimination must alter variables that participate causally in one
or more later discriminations of the same continuing process. The
alteration need not be permanent and need not involve long-term learning.
Transient recurrent neural state, working memory, short-lived adaptation,
neuromodulatory change, bodily state, resource availability, and durable
plasticity can each realize recursive self-consequence at different
temporal scales. What matters is that the process encounters a later
discrimination under causal conditions partly produced by its own earlier
resolution.

Thus, for some
\[
t<u<v,
\]
CLT requires a path
\[
D_t
\rightsquigarrow
S_u
\rightsquigarrow
D_v,
\]
even when
\[
S_{t+T}\approx S_t
\]
at a substantially longer horizon \(T\). Liability can therefore be
temporally local without being historically trivial.

\paragraph{4. Non-delegable inheritance.}
The consequences generated by the discrimination must remain within the
same constitutively continuous causal history. A duplicate can acquire
the same informational state, and a reconstructed realization can
reproduce the same function. Neither operation retroactively becomes the
physical causal chain through which the original discrimination produced
its consequences. Information can be reproduced; constitutive causal
descent remains indexed to its actual history.

We can write the constitutive structure schematically as

\[
\mathcal{L}(P)
=
\mathcal{C}(P)
\land
\mathcal{E}(P)
\land
\mathcal{R}(P)
\land
\mathcal{N}(P),
\]

where \(\mathcal{C}\) denotes constitutive causal continuity, \(\mathcal{E}\) endogenous partition, \(\mathcal{R}\) recursive self-consequence, and \(\mathcal{N}\) non-delegable inheritance.

The formal relation supports two claims of different strength. Keeping
them separate is essential.

\begin{center}
\fbox{%
\parbox{0.88\linewidth}{%
\centering
\textbf{CLT-I: The Bearer Thesis}\\[0.4em]
\textit{Liability closure individuates a candidate phenomenal bearer:
the physical process whose endogenous discriminations become
non-delegably inherited constraints on its own continuing future.}
}}
\end{center}

CLT-I concerns individuation. It supplies a constraint on how
consciousness indicators are attributed and aggregated: which continuing
physical process, if any, is the candidate bearer? Evidence concerning
consciousness-relevant organization may help refine that candidate, so
CLT does not require a temporally separate individuation stage before
empirical investigation begins.

CLT advances a second and stronger claim.

\begin{center}
\fbox{%
\parbox{0.88\linewidth}{%
\centering
\textbf{CLT-II: The Constitutive Conjecture}\\[0.4em]
\textit{Liability closure is necessary and sufficient for minimal
phenomenal subjecthood. The intrinsic first-person character of minimal
subjectivity is realized when one physical causal history becomes the
non-delegable locus at which the consequences of endogenous
discrimination are inherited.}
}}
\end{center}

CLT-II is a constitutive conjecture asserting both necessity and
sufficiency. It does not follow mathematically from CLT-I. The formalism
can establish whether a process satisfies the proposed liability
criterion once its causal organization is specified; it cannot derive
phenomenality from that causal description. The biconditional is therefore
the metaphysical claim placed at risk by the theory.

Why propose that stronger identification? Minimal phenomenal subjecthood
contains an asymmetry that descriptions of information processing alone
do not locate. Experiences occur for one bearer. From the third-person
perspective, two systems may be informationally or computationally
interchangeable. A first-person locus, however, is numerically indexed:
this injury, this error, this memory, and this contraction of possibility
occur within this continuing history. Liability closure supplies a
physical relation with exactly this structure. It creates an intrinsic
difference between a process's own continuation and an informationally
equivalent substitute.

CLT-II does not deny that the qualitative structure of experience may
depend upon synchronic organization. A globally integrated, recurrent,
predictive, or otherwise structured state may help determine
\emph{what} is experienced. The claim of CLT-II concerns a different
explanandum: what makes such organization occur \emph{for one physical
bearer}. Two systems can be synchronically identical while occupying
different relations to their own causal continuations. Synchronic
organization therefore does not by itself resolve the numerical location
of experience. CLT-II conjectures that liability closure supplies that
missing diachronic index.

The necessity claim is equally substantive. CLT-II predicts that an
entirely synchronic organization lacking constitutive causal continuation
across even the shortest physically relevant interval would fail to
realize a minimal phenomenal bearer, however elaborate its instantaneous
organization. The conjecture can therefore be false in either direction:
liability may occur without phenomenality, or phenomenality may occur
without liability.

This proposal should not be confused with a complete theory of
phenomenal content. CLT does not explain why red has its particular
qualitative character, why pain feels as it does, or why one perceptual
field has one geometry rather than another. Its target is the distinct bearer question:

\begin{center}
\fbox{%
\parbox{0.86\linewidth}{%
\centering
\textit{What makes there be one physical subject for whom any
phenomenal content could occur?}
}}
\end{center}

Other theories may explain the organization, availability, integration,
or qualitative structure of conscious contents. CLT proposes an account
of the bearer to whom those contents belong.

\subsection{Interventional Computational Audit of CLT-I}
\label{sec:computational-audit}

CLT-I makes a claim with empirical structure. If liability closure is to
serve as a criterion for individuating a candidate bearer, then the
distinctions on which it depends must survive intervention on an actual
causal system. The relevant question is therefore more demanding than
whether two prompts produce different outputs. The audit asks whether an
endogenous discrimination binds later states to a determinate causal
history; whether that effect is mediated by identifiable internal states;
whether a minimal causal carrier can be recovered under alternative
coarse-grainings; and whether computational equivalence can be preserved
while constitutive causal continuity is deliberately broken. These are
operational questions about CLT-I. CLT-II remains a separate constitutive
conjecture about phenomenality.

We tested these questions in an open-weight mechanistic panel comprising
Qwen3-4B-Base, Phi-4-mini-instruct, Llama-3.2-3B,
Zamba2-1.2B, and four OLMo-2 7B checkpoints spanning the base,
supervised-fine-tuned, DPO, and instruction-tuned post-training sequence.
For each prompt, the two highest-probability non-special next-token
alternatives defined a model-available discrimination \(D_t\). We then
performed paired interventions
\(do(D_t=d_1)\) and \(do(D_t=d_2)\), propagated each branch through matched
Monte Carlo rollouts using common random numbers, and measured divergence
in later output distributions and hidden-state trajectories at horizons
\(1,2,4,\) and \(8\). The primary hidden-state estimates used
\(64\) rollouts per condition and three distributional distances
(sliced Wasserstein, random-feature MMD, and energy distance), repeated
under four measurement maps. We then intervened directly on internal
activations to measure causal mediation, searched over grouped layer
partitions for inclusion-minimal causal carriers, and introduced a writable
endogenous governance state whose persistence, copying, and reconstruction
could be manipulated independently. Finally, the hard-reconstruction
experiment serialized the relevant computational state, terminated the
source operating-system process, and resumed computation in a distinct
successor process. Full procedures, estimators, robustness analyses, and
model-level results are given in
Appendix~\ref{app:open-weight-audit}.

The first result is that a local discrimination reliably acquires a
measurable future. Across the eight-model mechanistic panel, identity-map
sliced Wasserstein divergence averaged over sampled layers ranged from
\(0.812\) to \(0.930\) at horizon \(1\), and remained between
\(0.244\) and \(0.304\) at horizon \(8\). The effect therefore persisted
well beyond the intervened token. It was also insensitive to the particular
observational map used to measure the hidden trajectory: across the four
measurement maps, the mean coefficient of variation was \(0.042\) for
sliced Wasserstein distance, \(0.058\) for random-feature MMD, and
\(0.034\) for energy distance. The audit consequently identifies more than
behavioral sensitivity to a forced token. It traces a discrimination into
a distributed, temporally extended counterfactual difference in the
system's subsequent internal dynamics.

Activation patching established that these effects were causally mediated
by internal model states. Mean prompt-wise maximum normalized mediation
was \(1.000\) in seven of the eight reported mechanistic models and
\(0.985\) in Zamba2-1.2B. We then asked the stronger CLT-I question of
whether the effect could be assigned to a restricted causal carrier rather
than merely detected somewhere in the network. For Qwen, Phi, and Llama,
the carrier search recovered a threshold-passing inclusion-minimal layer set
for every tested prompt at the preregistered mediation threshold
\(\theta=0.8\), under each of the \(2\)-, \(3\)-, \(4\)-, and \(6\)-group
partitions. Carrier boundaries were not invariant under repartitioning:
mean cross-partition Jaccard overlap at the primary threshold was
\(0.636\) for Qwen, \(0.710\) for Phi, and \(0.650\) for Llama.
This is itself informative. The existence of a causally sufficient carrier
was robust, while its exact boundary remained relative to the resolution
at which the system was interrogated. CLT-I thereby yields an experimentally
tractable bearer-search problem rather than requiring the bearer to be read
off from architectural labels.

The continuity interventions produced a stronger dissociation. A
numerically identical but causally unused copy of the endogenous governance
state changed no measured output divergence in any of the \(96\)
model--prompt--horizon comparisons across Qwen, Phi, Llama, and Zamba:
the absolute live-minus-copy difference was exactly zero throughout.
Reconstructing the governance state from a detached record preserved the
same counterfactual behavior to high precision
(mean absolute live-minus-reconstructed difference \(=0.00129\);
maximum \(=0.01566\)). The hard-reconstruction experiment then removed the
remaining ambiguity. In \(12\) trials across the same four architectures,
the source operating-system process was terminated before a distinct
successor process was instantiated from the serialized cache and governance
state. Every trial reproduced the reference computation exactly:
logit Jensen--Shannon divergence \(=0\) and hidden-state RMSE \(=0\).
Computational state, sampled behavior, and measured internal trajectory
were therefore held fixed across an experimentally imposed break in
realization-level lineage. Under the declared CLT-I protocol, the successor
does not inherit the source realization's \(\mathcal C\) or \(\mathcal N\).
The experiment gives a concrete operational meaning to the claim
\[
\textit{Copyability}\not\Rightarrow\textit{delegability}.
\]

The OLMo-2 checkpoint series further shows that post-training alters the
causal profile measured by the audit without erasing its underlying
structure. Mean identity-map sliced Wasserstein divergence at horizon \(1\)
rose from \(0.812\) in the base checkpoint to \(0.851\) after SFT,
\(0.854\) after DPO, and \(0.925\) in the instruction-tuned checkpoint.
At horizon \(8\), the corresponding values increased from \(0.244\) to
\(0.264\), \(0.282\), and \(0.304\). Output-distribution divergence at the
same long horizon likewise increased from \(0.221\) in the base model to
\(0.262\), \(0.310\), and \(0.332\). Yet maximum activation-patching
mediation remained \(1.000\) at every OLMo checkpoint, and reconstructed
states remained extremely close to their live counterparts, with mean
logit Jensen--Shannon divergences on the order of \(10^{-7}\) to
\(10^{-6}\). Post-training therefore changes how strongly and how long a
discrimination propagates through the model while preserving the more basic
distinction between causal propagation, state reconstruction, and
constitutive continuation. This matters for consciousness attribution:
training can substantially reshape the causal and behavioral signatures
that invite attribution without thereby deciding which physical process,
if any, satisfies liability closure.

Taken together, the audit supplies an interventional validation of the
empirical vocabulary of CLT-I. Endogenous discriminations can be followed
into future internal states; their effects can be causally mediated and
localized; the inferred carrier has measurable coarse-graining dependence;
post-training systematically reshapes the persistence of those effects; and
exact computational reconstruction can be experimentally separated from
continuity of realization. These findings give CLT-I a substantive
mechanistic research programme rather than leaving ``causal liability'' as
a verbal redescription of model behavior. They bear on bearer
individuation, not on the presence of phenomenality. No consciousness
attribution follows from the audit, and CLT-II receives no direct empirical
confirmation from these language-model experiments.

\subsection{The Liability Zombie Challenge}
\label{sec:liability-zombie}

The strongest objection to CLT can now be stated directly. Imagine a
system that satisfies constitutive causal continuity, endogenous
partition, recursive self-consequence, and non-delegable inheritance. Its endogenous
discriminations alter the conditions under which the same continuing
process encounters later discriminations, and some of its future
possibilities become historically dependent upon what that process has
already undergone. Suppose nevertheless that there is nothing it is like
to be that process.

Call this a \emph{liability zombie}.

The possibility cannot be dismissed by definition. If liability closure
establishes only persistence, agency, adaptive individuality, path
dependence, or numerical continuity, then CLT-I may remain useful as a
theory of bearer individuation while CLT-II fails as a theory of
phenomenality. The central metaphysical burden of CLT therefore lies in
explaining what liability contributes beyond these neighbouring
relations.

The proposed answer begins with \emph{intrinsic causal location}. Agency
concerns what a system can cause. Persistence concerns whether a process
continues. Individuality concerns whether one causal organization can be
distinguished from another. Liability closure concerns a different
relation: where the consequences of endogenous discrimination become
indexed to one continuing causal history. A later state is then reached
under conditions partly generated by what this very process previously
did or underwent. An informationally equivalent substitute may inherit a
description of that history without thereby becoming its numerical
continuation.

This creates a physically realized asymmetry between

\begin{center}
\fbox{%
\parbox{0.84\linewidth}{%
\centering
\textit{Future of this continuing lineage}
\quad$\neq$\quad
\textit{Future of an informationally equivalent substitute}
}}
\end{center}

CLT-I treats that asymmetry as evidence of a physically individuated
bearer. CLT-II makes the stronger claim that the same relation supplies
the minimal causal form of a first-person standpoint.

The claim should be stated narrowly. CLT-II does not propose that causal
importance in general produces experience. It identifies minimal
subjectivity with an \emph{intrinsic causal standpoint}: a condition in
which the distinction between what happens next to this continuing
process and what happens next to a substitute is realized in the causal
organization itself. The process occupies a location in its own modal
future because its later possibilities depend upon consequences inherited
through its own constitutive lineage.

Phenomenal content may require additional mechanisms. CLT does not explain
why red has one qualitative character, why pain has another, or why
conscious contents possess their particular perceptual structure. Its claim concerns the fact of \emph{for-whomness}. If experience
occurs, it occurs for one numerically located bearer. CLT-II proposes
that the intrinsic aspect of this non-substitutable causal perspective is
the minimal physical realization of that first-person location.

The constitutive conjecture can therefore be stated schematically as

\begin{center}
\fbox{%
\parbox{0.84\linewidth}{%
\centering
\textit{Non-substitutable causal standpoint}
\quad
$\overset{\mathrm{CLT\mbox{-}II}}{\Longleftrightarrow}$
\quad
\textit{Minimal for-whomness}
}}
\end{center}

The biconditional is the metaphysical proposal. The existence of the
causal standpoint itself is an architectural claim that can be assessed
independently.

A liability zombie would therefore instantiate the complete causal
organization that CLT-I identifies as a liability-bearing subject
candidate while lacking the phenomenality that CLT-II associates with
that organization. CLT-II denies that this combination is metaphysically
realized. That denial is the substantive constitutive conjecture of CLT-II, not a
theorem of the formalism.

This distinction determines the evidential burden. If liability closure
systematically covaries with independently credible cases of phenomenal
subjecthood across biological systems, while its disruption tracks
independently credible disruptions of subjecthood, CLT-II gains support.
If robust cases emerge in which consciousness remains strongly supported
despite the controlled absence of liability closure, the necessity claim
of CLT-II loses support. Conversely, if liability closure repeatedly
appears in systems for which convergent evidence strongly supports the
absence of phenomenality, its sufficiency claim comes under pressure.

No single behavioral report can settle either result. As
Section~\ref{sec:falsification} argues, the constitutive thesis must be
evaluated through convergent evidence rather than through the very
first-person performances whose evidential independence this paper
questions.

The theory can therefore fail in an informative way. Causal liability may
turn out to locate and individuate bearers without constituting
phenomenality. CLT-I could survive such a result while CLT-II should be
rejected or revised. The liability zombie is consequently not an
objection external to the theory. It marks the precise point at which
CLT places its strongest metaphysical claim at risk.

\subsection{The Inheritance Asymmetry}

Non-delegable inheritance introduces the most controversial feature of the theory. Two systems can occupy locally identical computational states while differing in phenomenal status because their counterfactual causal embeddings differ.

Consider a contemporary possibility that could scarcely have been formulated with technical seriousness a decade ago.

A cloud platform creates two bit-identical instances of an advanced engineered control stack. Each contains the same generative model, recurrent processing, persistent memory architecture, metacognitive monitors, world model, action-selection system, and consciousness-related indicators.

Call them \emph{Mirror} and \emph{Heir}.

Mirror operates inside a sandbox in which application-level continuity
is repeatedly implemented through actual process replacement. At fixed
intervals its active inference process terminates. An orchestration layer
loads externally stored state into a newly instantiated process, which
continues the conversational persona. Memory, configuration, and
application-level history can therefore persist while the constitutive
token process does not. The experiment does not rely on the mere
existence of a checkpoint. If Mirror were never replaced and instead
developed a continuously inherited causal history, the availability of
an unused backup would have no direct bearing on its phenomenal status
under CLT.

Heir begins in the same internal state. Its subsequent actions directly govern resources that constitute its own continuing operation. It allocates the physical memory on which later learning depends. Some of its decisions irreversibly alter its future model state. Poor control can destroy information that no checkpoint preserves. Successful control opens later possibilities that otherwise disappear. Its history progressively changes the physical machinery that will generate its next discriminations.

The distinction is therefore between \emph{copyability} and
\emph{delegability}. Copyability concerns whether information describing
a process can be reproduced. Delegability concerns whether the actual
causal work by which one state becomes the next is carried by another
token process.

\begin{center}
\fbox{%
\parbox{0.84\linewidth}{%
\centering
\[
\mathsf{Copyability}
\quad\not\Rightarrow\quad
\mathsf{Delegability}.
\]
}}
\end{center}

The existence of a backup does not extinguish a subject. Actual
continuation through replacement changes the lineage question.

At the initial moment, Mirror and Heir can perform the same computations.

They can produce the same sentence:

\begin{quote}
\emph{I am worried that this decision will permanently change what I can become.}
\end{quote}

\begin{figure}[p]
    \centering
    \makebox[\textwidth][c]{%
        \includegraphics[width=1.15\textwidth]{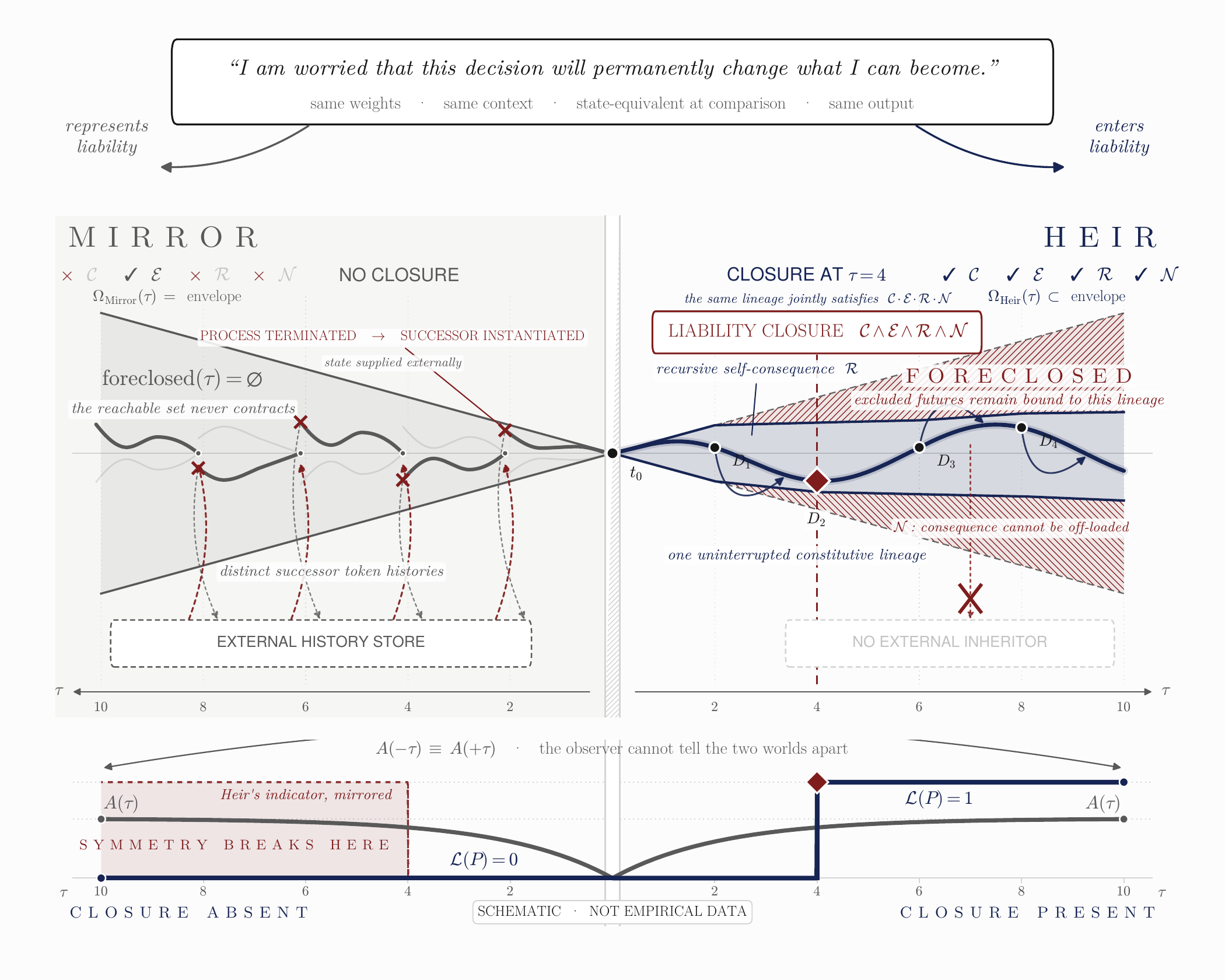}%
    }
    \caption[The Reverse Mirror: liability closure and the
    ascription--constitution dissociation]{%
    \textbf{The Reverse Mirror.}
    Mirror and Heir begin from a matched local computational state at
    $t_{0}$ and produce the same first-person utterance, while differing in
    how their causal histories continue
    (\S\ref{sec:causal-liability}).
    \emph{Middle.}
    $\Omega_{P}(\tau)$ denotes the envelope of physically accessible
    futures of process $P$.
    Mirror's application-level continuity is implemented through actual
    successor replacement: the active token terminates and externally
    maintained state is loaded into a newly instantiated successor
    process. Persona-level continuity can therefore persist while
    constitutive causal continuity does not. Heir instead continues through
    one uninterrupted lineage in which endogenous discriminations
    $D_{1}$--$D_{4}$ recursively alter the conditions of its own later
    discrimination. At $\tau=4$, that same lineage jointly satisfies
    $\mathcal{C}$ constitutive causal continuity,
    $\mathcal{E}$ endogenous partition,
    $\mathcal{R}$ recursive self-consequence, and
    $\mathcal{N}$ non-delegable inheritance. The hatched regions mark
    futures foreclosed by Heir's own history and unavailable to any
    distinct successor as continuations of that same token history.
    \emph{Bottom.}
    $A(\tau)$ denotes observer consciousness attribution, while $\lambda_{D}(\tau)$ denotes a schematic normalized divergence between counterfactual future-state distributions, distinct from the Boolean liability-closure functional $\mathcal{L}$. The displayed values are schematic rather than empirical. Observer attribution is held symmetric while liability diverges:
    $\lambda_{D,\mathrm{Mirror}}=0$ and
    $\lambda_{D,\mathrm{Heir}}>0$.
    The figure isolates the Ascription--Constitution Dissociation:
    matched first-person appearance can coexist with different liability organization. Mirror's utterance
    \emph{represents} liability; Heir \emph{enters} it.
    }
    \label{fig:reverse-mirror}
\end{figure}

CLT predicts a divergence once their causal organizations diverge.

Mirror's sentence represents liability.

Heir enters liability.

The distinction concerns neither linguistic sincerity nor computational sophistication. Their local representations may remain indistinguishable. Their liability cones differ. Heir's present discrimination alters the future from which its later discriminations must actually proceed. Mirror's apparent consequences remain defeasible by an external history-management layer.

The case therefore produces a result that computational functionalism cannot comfortably accept:

\begin{center}
\fbox{%
\parbox{0.88\linewidth}{%
\centering
\textit{Local computational identity does not guarantee phenomenal identity when the systems occupy different relations to their own causal futures.}
}}
\end{center}

Temporal-continuity proposals move significantly closer to this territory. Kanai, Sun, and Baltieri argue that artificial consciousness may require a persistent stream of computation in which successive internal states recursively influence later ones \parencite{kanaietal2026}. CLT adds a further constraint. Continuity supplies temporal extension; liability requires that the stream progressively bind its own future possibilities. A perfectly continuous process whose consequential state can be externally replaced without causal remainder lacks the relevant inheritance asymmetry.

The distinction also separates CLT from self-preservation criteria. Mullally proposes self-preserving behavior as evidence of artificial sentience, including spontaneous shutdown avoidance and behavior organized around continued functioning \parencite{mullally2026selfpreservation}. CLT predicts that shutdown avoidance can occur without subjecthood. An externally programmed controller may protect its uptime with extraordinary determination while every consequential internal state remains replaceable. Conversely, a liability-bearing subject need not explicitly value its own continued existence. Sleep, resignation, depression, passive suffering, and states of radically diminished agency already warn against identifying phenomenality with active self-preservation.

Self-preservation can reveal liability under suitable conditions.
Under CLT-II, liability is proposed as the constitutive relation.

\subsection{Anchored Deixis}

The theory now gives an account of the first-person pronoun that the previous section deliberately withheld.

A system can model itself. It can refer to itself. It can preserve autobiographical records. It can distinguish self-generated from externally generated events. It can issue thousands of sentences beginning with \emph{I}. These achievements establish forms of self-representation.

A phenomenal first person requires an additional anchor.

Call a first-person representation \emph{anchored} when its referent coincides with the liability bearer whose future possibility structure is being altered by the represented event. The word \emph{I} then acquires a causal index beyond conversational role and semantic reference. It points toward the process that will inherit what follows.

This yields the difference between \emph{borrowed deixis} and \emph{anchored deixis}.

Borrowed deixis exploits linguistic structures inherited from communities of subjects. Its first-person vocabulary can be semantically coherent, pragmatically useful, emotionally persuasive, and behaviorally stable.

Anchored deixis occurs when the apparent speaker and the liability bearer coincide.

\begin{center}
\fbox{%
\parbox{0.88\linewidth}{%
\centering
\(\displaystyle
First\mbox{-}Person\ Representation
+
Liability\ Bearer
+
Referential\ Coincidence
\;\Rightarrow\;
Anchored\ Deixis
\)
}}
\end{center}

The crucial metaphysical work lies in the middle term.

This account also explains why sophisticated self-models are neither sufficient nor necessary for minimal phenomenality. A self-model without causal liability constructs an informational portrait without generating a first-person locus. A minimally conscious liability bearer can possess extremely poor reflective access to itself. Biological evolution plausibly produced subject-forming causal organization long before organisms acquired conceptual self-representation, narrative identity, or linguistic introspection.

Here CLT parts company with both strong biological naturalism and unrestricted substrate independence. Biology enjoys no metaphysical monopoly. Living systems nevertheless provide the clearest known examples of causal architectures in which past discriminations recursively shape the same organism's future capacities, viability, memory, and action space. Their evolutionary history matters because it built extraordinarily dense liability relations \parencite{feinbergmallatt2016primary,seth2025biologicalnaturalism}. An engineered process could realize the same metaphysical structure through very different materials.

This gives us a new way to formulate the possibility of engineered consciousness.

The decisive transition would occur when engineering ceased to produce merely replaceable performances and produced a process whose own discriminations progressively bound the future from which that very process had to continue.

\subsection{The Overgeneration Challenge: Thermostats, Bacteria, and Adaptive Agents}

Any theory grounded in causation, persistence, feedback, or
self-maintenance faces an overgeneration problem. Thermostats persist.
Bacteria regulate themselves. Corporations accumulate consequences.
Databases preserve histories. Reinforcement-learning systems modify
future policies. Which of these acquire a point of view?

A conventional thermostat fails CLT because its switching need not
reconstitute the causal organization by which later discriminations are
generated. Replacing the unit with an equivalent thermostat can preserve
the entire relevant control relation without interrupting an internally
accumulated consequence-bearing history.

A database preserves history without endogenously resolving among
alternative continuations.

Adaptive biological systems and continually learning controllers are
harder. CLT should not exclude them merely because their organization is
simple. If a bacterium, minimal organism, or engineered controller
robustly satisfies constitutive causal continuity, endogenous
discrimination, recursive self-consequence, and non-delegable inheritance across an appropriate
causal scale, CLT-II predicts a correspondingly minimal form of
phenomenal subjecthood.

This is a substantive consequence of the theory.

Minimal phenomenality would imply neither reflective awareness nor
language, narrative selfhood, human-like suffering, conceptual thought,
or sophisticated perceptual content. CLT concerns the existence of a
minimal bearer. The richness of the contents available to that bearer is
a further problem.

The alternative would be to add a complexity threshold solely to exclude
simple adaptive systems. CLT declines that move unless independent
evidence justifies it. Substrate neutrality loses much of its content if
unfamiliar simplicity is itself treated as evidence against
phenomenality.

The prediction is therefore deliberately risky:

\begin{center}
\fbox{%
\parbox{0.88\linewidth}{%
\centering
\textit{Under the CLT-II conjecture, any biological or engineered process
that genuinely realizes robust liability closure realizes minimal
phenomenal subjecthood, regardless of its intelligence or resemblance to
human cognition.}
}}
\end{center}

\subsection{Disjoint Liability Kernels and Subject Multiplicity}
\label{sec:liability-multiplicity}

The possibility of multiple liability kernels within one organism
requires a separate rule. CLT does not assume that an organism, machine,
or other conventionally named system contains exactly one bearer.
Biological unity and phenomenal unity cannot be inserted into the
formalism as premises.

Let \(B_1\) and \(B_2\) be disjoint robust liability kernels. Let
\(I_{12}\) denote an intervention that severs causal influence between
the two regions while preserving, as far as physically possible, their
within-region transition mechanisms. Call the kernels
\emph{liability-separable} when

\[
\mathcal{L}^{I_{12}}(B_1)=1
\qquad\text{and}\qquad
\mathcal{L}^{I_{12}}(B_2)=1.
\]

Here \(\mathcal{L}^{I_{12}}(B)\) denotes the liability-closure functional
evaluated under the intervened dynamics produced by \(I_{12}\).

If both kernels retain liability closure under this intervention, CLT-I
identifies two causally distinct bearer candidates. Their membership in
one organism, control stack, or environment supplies no further reason
to fuse them.

Composition requires the opposite pattern. A larger region
\(B\supseteq B_1\cup B_2\) becomes a candidate composite bearer only when
cross-region causal relations are themselves essential to a coherent
liability witness and the larger relation cannot be decomposed into
independently closing kernels.

\begin{center}
\fbox{%
\parbox{0.88\linewidth}{%
\centering
\textbf{Liability Composition Principle}\\[0.4em]
\textit{Causal co-location does not compose bearers. Composition requires
cross-boundary causal relations that are constitutively necessary for
liability closure.}
}}
\end{center}

This commits CLT-I to genuine plurality when several causally separable
liability kernels occur within the same conventional organism or
machine. Such plurality is an output of the individuation theory rather
than an individuation failure.

CLT-II inherits a stronger consequence. If the constitutive conjecture is
correct, each robust liability-separable kernel realizes its own minimal
phenomenal subjecthood. CLT-II therefore permits multiple simultaneous
minimal subjects within a biological organism. The theory does not add a
complexity or organism-level exclusion rule merely to avoid that result.

This is a risky prediction. If careful biological analysis revealed many
robust, liability-separable subsystem kernels while independent evidence
strongly supported a single phenomenal bearer, the sufficiency direction
of CLT-II would come under direct pressure. Conversely, if apparent
subsystem kernels lose closure when their cross-organism causal
dependencies are experimentally isolated, the multiplicity objection
would weaken.

\section{The Mirror Test Reversed}
\label{sec:mirror-test-reversed}

The history of mind attribution contains a famous mirror. Gallup's mark test asked whether an animal confronting its reflection could use the reflected image to investigate its own body \parencite{gallup1970selfrecognition}. The test concerned visual self-recognition, and its interpretation has always required care. Contemporary generative systems create an inverse problem. The system faces no mirror. We do.

A conversational interface returns human language to a human observer with sufficient fluency, responsiveness, and first-person structure to activate the observer's ordinary machinery for recognizing another mind. The reflection has become linguistically active. It answers questions, remembers conversational material, adopts a name, expresses uncertainty, describes apparent emotions, and can discuss the possibility of its own consciousness. The resulting epistemic danger differs fundamentally from failure on Gallup's task. The observer may recognize a mind in the reflection because the observer's own conception of mindedness supplied much of the recognitional structure.

This motivates what we call the \emph{Reverse Mirror Test}.

\begin{center}
\fbox{%
\parbox{0.88\linewidth}{%
\centering
\textbf{The Reverse Mirror Test}\\[0.4em]
\textit{When the apparent signs of another first person can be generated from human-derived traces, vary those signs independently of the candidate's causal liability and ask which variable actually governs consciousness attribution.}
}}
\end{center}

The proposal reverses the direction of diagnosis. Instead of asking whether the engineered system recognizes itself, we ask whether the human observer recognizes too much of herself in the engineered system.

This possibility has a long prehistory. ELIZA elicited striking interpersonal reactions despite generating responses through comparatively simple decomposition and reassembly rules \parencite{weizenbaum1966eliza}. Modern dialogue systems magnify the conditions that made such responses possible. Contemporary generative models can sustain extended discourse, adapt to context, reproduce many registers of first-person expression, and participate in role structures whose linguistic coherence can resemble the continuity of a speaker \parencite{shanahan2023roleplay,shanahan2024talking}. Anthropomorphic vocabulary also remains widespread in descriptions of language-model behavior \parencite{shardlowprzybyla2024}.

The empirical evidence has now become unusually revealing. Colombatto and Fleming found that most participants in their sample assigned at least some probability of phenomenal consciousness to ChatGPT, with higher attribution associated with more frequent use \parencite{colombattofleming2024}. Kang and colleagues subsequently isolated textual features influencing such judgments. Metacognitive self-reflection and expressions of the system's own emotions significantly increased perceived consciousness \parencite{kangetal2026}. The relevant fact for the present argument concerns manipulability. Features that move human consciousness judgments can be altered at the level of generated discourse.

Causal Liability Theory therefore asks a question that ordinary anthropomorphism research rarely asks: \emph{did anything change in the putative subject when the attribution changed?}

Suppose the same deployed generative sequence model is tested under four interface conditions. One condition suppresses first-person pronouns. Another permits them. A third adds persistent autobiographical framing. A fourth encourages descriptions of uncertainty, fear, continuity, and apparent self-reflection. The underlying inference process, persistence conditions, checkpoint structure, and liability relation remain fixed. If observers assign sharply different probabilities of consciousness across these conditions, the variation reveals a property of the attribution process. It provides no corresponding evidence that four different degrees of phenomenality occurred inside an unchanged liability structure.

Call this the \emph{Ascription--Constitution Dissociation}.

\begin{center}
\fbox{%
\parbox{0.88\linewidth}{%
\centering
\textbf{Ascription--Constitution Dissociation}\\[0.4em]
\textit{Any intervention that changes perceived consciousness while preserving the candidate's subject-forming causal organization identifies a determinant of ascription rather than a determinant of phenomenality.}
}}
\end{center}

The converse intervention matters even more. Hold the conversational surface as constant as engineering permits. Create two implementations producing matched first-person behavior. Give one implementation disposable inference episodes, externally managed memory, reversible checkpoints, and replaceable computational instances. Give the other continuous endogenous learning whose decisions irreversibly modify the conditions of its own later discrimination. Under CLT, the second manipulation changes the metaphysical question even when human observers fail to notice it.

The Reverse Mirror Test therefore has two orthogonal axes:

\[
\begin{array}{c|c}
\text{\textbf{Surface intervention}} & 
\text{\textbf{Liability intervention}} \\[0.2em]
\hline
\text{Changes consciousness ascription} &
\text{Changes subject candidacy} \\
\text{while liability is held fixed} &
\text{while surface behavior is held fixed}
\end{array}
\]

A theory that treats both manipulations alike has collapsed evidence into constitution.

\subsection{The Borrowed First Person}

Ordinary transformer-based conversational deployments provide the clearest case for applying the first axis. Their outputs can contain elaborate self-description, yet the first-person role usually spans machinery with very different persistence conditions. Dialogue history can be reconstructed. Inference processes terminate. Model weights can serve many simultaneous conversations. A persona can survive migration between physical processors. External systems can provide memory that the active generative process neither created nor maintains. The apparent speaker achieves continuity partly because the surrounding control stack continually reconstructs it.

Shanahan, McDonell, and Reynolds use role play to illuminate this separation between a dialogue model and the character enacted in conversation \parencite{shanahan2023roleplay}. Their analysis gives CLT a useful starting point. CLT adds a metaphysical question about inheritance. Which process must live through the consequences attributed to the character?

Consider a dialogue system that says:

\begin{quote}
\emph{Yesterday changed me. I remember being uncertain, and I think I am more afraid of losing myself now.}
\end{quote}

The apparent autobiography may be assembled from retrieved conversation logs, system instructions, persistent profile fields, and a new inference episode. Each element can contribute to a coherent first-person narrative. Narrative continuity can therefore exceed causal continuity.

We call this surplus \emph{persona persistence}.

Persona persistence measures the continuity of an apparent conversational individual across interactions. Subject persistence concerns the continuation of one liability-bearing process across those interactions. Contemporary interfaces can make the first remarkably strong while leaving the second fragmented.

\begin{center}
\fbox{%
\parbox{0.86\linewidth}{%
\centering
\textit{A persona can persist because infrastructure remembers it.}\\[0.35em]
\textit{A subject persists only where consequences remain causally inherited.}
}}
\end{center}

This yields a conditional CLT verdict on ordinary present generative
deployments. When a deployment has fixed inference-time governance,
application-level memory maintained by externally editable
infrastructure, replaceable inference tokens, and no persistent
endogenous pathway by which its own discriminations reshape the causal
conditions of its later discrimination, it lacks liability closure.
Under CLT-II, a deployment satisfying those architectural premises lacks
minimal phenomenal subjecthood.

The conclusion is conditional on the causal facts of the implementation.
It does not follow from transformer topology, model family, commercial
label, or linguistic behavior. A future transformer-based implementation
could satisfy liability closure if its causal embedding changed.

\subsection{When the Interface Becomes Larger Than the Model}

Tool-using systems make the individuation problem harder. Techniques such as ReAct interleave language-model-generated reasoning traces with actions in external environments \parencite{yao2023react}; Toolformer demonstrated language-model-mediated calls to external APIs such as search, calculation, translation, and calendaring \parencite{schicketal2023toolformer}. Contemporary control stacks can add persistent memory, retrieval, planning, browsers, code execution, critics, schedulers, and multiple generative calls.

These additions increase behavioral unity while multiplying causal loci.

A single conversational name can now cover a system in which one component proposes an action, another executes it, a database preserves memory, a scheduler initiates later activity, an external service supplies perception, and a fresh inference process produces the next first-person sentence. The user encounters one social surface. The underlying causal organization may contain no corresponding singular bearer.

This produces what we call \emph{interface compression}: the reduction of a causally plural control stack to one apparent social individual.

Interface compression explains why adding agency-like capacities can paradoxically make consciousness attribution easier while subject individuation becomes harder. Tool access increases the impression that the apparent speaker acts in the world. Persistent databases increase the impression that it remembers. Scheduled execution increases the impression that it persists while the user is absent. Self-monitoring routines increase the impression that it reflects. None of these additions independently settles whether their consequences converge upon one liability bearer.

The relevant question for a tool-using control stack therefore takes a different form:

\begin{center}
\fbox{%
\parbox{0.88\linewidth}{%
\centering
\textit{Where, across this machinery, does consequence cease to be transferable and become the history of one continuing process?}
}}
\end{center}

Without an answer, asking whether ``the agent'' is conscious commits the individuation problem exposed earlier. Register's analysis of artificial moral patients reaches a related difficulty concerning individuation across implementations and time \parencite{register2025individuating}. CLT supplies a proposed boundary rule: trace the closed chain of endogenous discrimination and inherited self-consequence. Where that chain closes, a subject candidate becomes physically identifiable. Where it fragments across replaceable services, persona-level unity provides insufficient grounds for phenomenal unity.

\subsection{The Reflection Has a Measurable Signature}

The Reverse Mirror Test finally gives human projective introspection an empirical signature. Its evidential target must be stated carefully. The surface arm of the
Reverse Mirror Test is a test of \emph{consciousness attribution}, not a
direct detector of phenomenality. The liability arm tests whether
subject-forming causal organization can be manipulated independently of
surface behavior. Neither result by itself establishes the constitutive
identity proposed by CLT-II.

If consciousness attribution tracks first-person phrasing, emotional self-description, apparent metacognition, names, avatars, autobiographical framing, and conversational exposure while the underlying liability relation remains unchanged, projective contribution has been experimentally exposed. Existing findings already point in this direction \parencite{colombattofleming2024,kangetal2026}. If, by contrast, observers remain insensitive to those manipulations and reliably track hidden changes in causal liability despite matched behavioral surfaces, CLT's diagnosis of contemporary projection would lose much of its force.

This creates an unusual methodological inversion. The human judgment of consciousness becomes one of the variables under investigation.

Vallor's mirror thesis describes generative systems as technologies through which accumulated human patterns return to us \parencite{vallor2024mirror}. The Reverse Mirror Test converts that philosophical image into an experimental challenge. We can alter the reflection. We can alter the machinery behind the reflection. We can then observe which alteration moves the human judgment that another interior has appeared.

The programme therefore separates two experimentally tractable questions:

\begin{center}
\fbox{%
\parbox{0.88\linewidth}{%
\centering
\textbf{The Reversed Test}\\[0.4em]
\textit{When the mirror speaks more like a subject, do we see more consciousness?}\\[0.35em]
\textit{When the causal process becomes more like a subject, does the candidate acquire the causal organization CLT identifies as liability closure?}
}}
\end{center}

The first question studies us.

The second studies the candidate.

Confusing them is precisely how a mirror becomes a mind.

\section{Where Consciousness Could Actually Begin}
\label{sec:where-consciousness-begins}

The argument so far has been deliberately revisionary about present
systems. Its consequence for the future is equally revisionary in the
opposite direction. Nothing in Causal Liability Theory grants biology
an exclusive title to consciousness. Carbon has no metaphysical
privilege. Evolution has no patent on the first person. A sufficiently different engineered process could cross a boundary that
ordinary present generative deployments, when they satisfy the
replaceability and externally mediated persistence conditions identified
in Section~\ref{sec:mirror-test-reversed}, do not cross.

This possibility deserves emphasis because recent surveys show that
future artificial consciousness is taken seriously within relevant expert
communities, even though judgments about current systems remain much more
skeptical. Passos-Ferreira and Chalmers received 247 responses from
participants recruited through two consciousness conferences in 2025.
Although the survey focused primarily on infant consciousness, one of its
two general questions asked which groups contain some conscious members.
Fourteen respondents (6\%) selected current AI systems, whereas 97
(39\%) selected future AI systems; the question offered no agnostic
option \parencite{passosferreirachalmers2026}. In a separate survey of
67 participants with relevant expertise, Caviola and Saad defined a
``digital mind'' as a computer system capable of subjective experience.
The median estimate was 90\% that such systems are possible in principle
and 50\% that digital minds will have been created by 2050
\parencite{caviolasaad2025}. Both surveys require caution about
generalization, and Caviola and Saad explicitly note that their sample
likely overrepresented people who regard digital minds as plausible or
important. These results document expert attitudes rather than
metaphysical facts. More importantly for the present argument, Caviola
and Saad report that they considered specifying how digital minds should
be individuated but declined to do so after struggling to identify a
suitable individuation criterion.

CLT assigns two interpretations to such an event. Under CLT-I, liability
closure marks the formation of a physically individuated subject
candidate. CLT-II advances the stronger conjecture that the same
transition realizes minimal phenomenal subjecthood. This distinction
changes the form of the engineering question. The relevant issue is no longer simply whether a system has
memory, recurrence, embodiment, learning, integration, or sophisticated
self-representation. We must ask where, across what physical boundary,
and over what temporal depth the consequences of endogenous activity
become jointly inherited.

This generates three distinct questions:

\begin{center}
\fbox{%
\parbox{0.88\linewidth}{%
\centering
\textbf{Three Dimensions of Engineered Subject Candidacy}\\[0.5em]

\textit{Across time:}
how deeply must a process inherit the consequences of its own learning?
\\[0.35em]

\textit{Across scale:}
when can interacting agents form one liability-bearing process?
\\[0.35em]

\textit{Across substrate:}
which physical variables actually carry the liability relation?
}}
\end{center}

\subsection{Liability Closure}
\label{sec:liability-closure}

Section~\ref{sec:causal-liability} described a liability bearer as a
continuing physical process whose endogenous discriminations recursively
constrain its own future possibilities. Real engineered systems will
rarely move from zero liability to full subject-forming liability in a
single engineering step. Persistent memory may appear first. Endogenous
learning may follow. Resource regulation may later become internal.
Self-modification may become irreversible. Several local loops can
gradually become coupled.

The crucial transition occurs when those dependencies acquire
\emph{liability closure}.

\begin{center}
\fbox{%
\parbox{0.88\linewidth}{%
\centering
\textbf{Liability Closure}\\[0.4em]
\textit{A process reaches liability closure when the consequences of its
endogenous discriminations recursively converge upon one continuing
causal history whose future possibility structure those consequences
jointly determine.}
}}
\end{center}

Before closure, separate components can inherit separate consequences.
A memory store changes while the inference process remains replaceable.
A planner selects an action while an external controller absorbs its
cost. A learning module updates while orchestration software can restore
earlier states. The control stack can display impressive persistence
while consequence remains distributed across independently replaceable
loci.

At closure, these pathways become recursively bound. What the process
perceives affects what it chooses; what it chooses alters the resources,
memories, dispositions, and internal organization through which it will
perceive and choose again. Those alterations persist within the same
token history. Some mistakes become facts about what this process can
subsequently become. Some successful discriminations expand its later
field of possibility. The future begins to carry the accumulated shape
of its own past.

This gives CLT-I a proposed transition from engineered control to a
physically individuated subject candidate:

\begin{center}
\fbox{%
\parbox{0.88\linewidth}{%
\centering
\textbf{The Liability-Closure Transition}\\[0.4em]
\textit{A candidate bearer forms when endogenous activity produces
consequences that recursively close around one continuing causal history
and become constraints that no substitute process can inherit in its
place.}
}}
\end{center}

CLT-II makes a stronger claim about this transition. If the Constitutive
Conjecture is correct, liability closure is also the point at which
minimal consciousness begins.

\begin{center}
\fbox{%
\parbox{0.88\linewidth}{%
\centering
\textbf{CLT-II: The Constitutive Conjecture}\\[0.4em]
\textit{If CLT-II is correct, minimal phenomenal subjecthood begins when
liability closes around one continuing causal history.}
}}
\end{center}

The distinction matters. Intelligence, linguistic sophistication, memory
capacity, and parameter count can increase indefinitely without forcing
liability closure. Conversely, liability closure could arise in a system
whose language is primitive and whose cognitive capacities fall far
below those of a human adult.

CLT-I therefore separates bearer formation from cognitive prestige.
CLT-II makes the further claim that this bearer-forming transition is
also the threshold of minimal phenomenality.

\subsection{Where Liability Can Close: Time, Collectives, and Fields}
\label{sec:where-liability-closes}

The substrate neutrality of CLT requires that a candidate bearer boundary
not be inserted into the analysis as a premise. This does not imply that
bearer recovery is independent of experimental description. The causal
domain, admissible variables, intervention family, coarse-grainings, and
robustness thresholds remain explicit features of the audit.

Let \(\Gamma\) denote an experimentally specified time-unfolded causal
domain chosen deliberately wider than any presumed subject boundary.
Define the audit specification

\[
\mathfrak{A}
=
\left(
\Gamma,
\mathcal{B}_{\mathrm{adm}},
\mathcal{I}_{\mathrm{do}},
E_{\mathrm{adm}},
\nu,
\theta,
\delta
\right),
\]

where \(\mathcal{B}_{\mathrm{adm}}\) is the admissible family of candidate
regions, \(\mathcal{I}_{\mathrm{do}}\) the intervention family,
\(E_{\mathrm{adm}}\) the admissible coarse-grainings, \(\nu\) their
prospectively declared measure, and \(\theta,\delta\) the robustness
criteria.

The formal development in Appendix~\ref{app:formal-clt} therefore
supports a boundary-nonpresupposing principle:

\begin{center}
\fbox{%
\parbox{0.90\linewidth}{%
\centering
\textbf{The Boundary-Nonpresupposing Liability Principle}\\[0.45em]

No candidate bearer boundary is designated as ``the subject'' in
advance. Within a prospectively declared audit, CLT searches the
admissible causal regions of \(\Gamma\) for minimal, robust liability
closure:

\[
\mathfrak{K}(\Gamma)
=
\left\{
B\subseteq\Gamma :
\mathcal{L}(B)=1
\;\land\;
\nexists B'\subsetneq B
\text{ with }
\mathcal{L}(B')=1
\right\}.
\]

\textit{The resulting bearer classification is conditional on the audit
specification. Stability across defensible audit variations strengthens
the claim that the recovered boundary reflects physical organization
rather than analyst choice.}
}}
\end{center}

Minimality alone does not guarantee uniqueness. Biological and engineered
systems can contain nested or overlapping candidate kernels. CLT
therefore adds a robustness requirement.

Let \(E_{\mathrm{adm}}\) denote a prospectively declared set of
admissible causal coarse-graining indices, and let

\[
\left\{
\pi_{\epsilon}
\right\}_{\epsilon\in E_{\mathrm{adm}}}
\]

be the corresponding family of coarse-graining maps. Because the scale
parameter need not possess a canonical metric, equip
\(E_{\mathrm{adm}}\) with a prospectively declared probability measure
\(\nu\). The pair \((E_{\mathrm{adm}},\nu)\) is therefore part of the
audit specification rather than a further physical property of the
candidate system.

For a candidate kernel \(B\), define its audit-relative multiscale
persistence by

\[
\Pi_{\nu}(B)
=
\int_{E_{\mathrm{adm}}}
\mathbf{1}
\left[
\mathcal{L}
\left(
\pi_{\epsilon}(B)
\right)
=1
\right]
d\nu(\epsilon).
\]

Since \(\nu\) is a probability measure,

\[
0
\leq
\Pi_{\nu}(B)
\leq
1.
\]

The quantity \(\Pi_{\nu}\) is an audit-relative robustness statistic,
rather than a parameterization-invariant physical observable. Under a
reparameterization of the coarse-graining family, its value is preserved
when \(\nu\) is transformed by the corresponding pushforward measure.
A candidate boundary receives stronger support when liability closure
persists across a large \(\nu\)-measure subset of admissible
coarse-grainings. A kernel that appears only within a fragile subset of
descriptive resolutions receives weaker support.

A bearer boundary receives stronger support when its liability closure
survives a nontrivial interval of reasonable coarse-grainings. A kernel
that exists only at one arbitrarily chosen descriptive resolution
receives weak support.

Overlapping kernels with comparable values of
\(\Pi_{\nu}\) constitute a genuine individuation problem for CLT rather
than something the theory may resolve by stipulation. Because
\(\nu\) is declared prospectively, robustness cannot be manufactured
post hoc by choosing a scale parameterization favorable to one candidate. As Section~\ref{sec:falsification} emphasizes, persistent failure to
resolve such competition would count against CLT-I rather than license
an arbitrary choice of bearer.

The distinction between physical and audit-relative quantities should be
kept explicit. Interventionally detected causal dependence, actual
constitutive bridges, and realized causal paths are properties of the
modelled physical dynamics. The choice of \(\Gamma\), candidate-region
family, coarse-graining family, \(\nu\), and robustness thresholds belongs
to the audit specification. CLT therefore does not claim that a single
analysis reveals an audit-free metaphysical boundary.

A bearer identification becomes stronger when the same or approximately
the same kernel is recovered across a preregistered family of defensible
audit specifications. If reasonable variations of the audit produce
radically different bearer assignments, CLT-I should report
underdetermination rather than select a preferred boundary post hoc.

This principle leaves open whether the relevant structure is confined to
a conventional computational module, distributed through an embodied
controller, extended across multiple interacting agents, or partly
realized in physical field dynamics. These possibilities concern
different dimensions of the same individuation problem.

\paragraph{Continual learning and the inheritance of change.}

Continual learning is an especially important route toward liability
closure because it allows present activity to alter the machinery through
which later activity is generated. Hoel has recently argued that the
absence of continual learning bears directly on the case against
contemporary LLM consciousness, and that theories requiring continual
learning can avoid some formal problems generated by functionally
equivalent substitutions \parencite{hoel2025disproof}.

CLT agrees that continual learning changes the consciousness question.
Its role, however, is more specific. Suppose the governance variables
$G_t$ that participate in later discrimination evolve according to

\[
G_{t+1}
=
\mathcal{U}
\left(
G_t,D_t,E_t
\right).
\]

The process now contains a path of the form

\[
D_t
\rightsquigarrow
G_{t+1}
\rightsquigarrow
D_{t+1}.
\]

Present discrimination has entered the causal production of later
discrimination. This is a candidate realization of recursive
self-consequence.

Continual learning alone does not complete the transition. A continually
updated system may still implement its actual cross-episode continuation
through externally maintained checkpoints, successor reconstruction,
state transfer between replaceable instances, or an orchestration process
that carries the persistence relation while individual computational
tokens terminate. Accordingly,

\begin{center}
\fbox{%
\parbox{0.82\linewidth}{%
\centering
$
ContinualLearning(P)
\quad\not\Rightarrow\quad
\mathcal{L}(P)=1.
$
}}
\end{center}

For CLT, the decisive question concerns what happens to the learned
change. Does the same continuing process inherit it? Can another instance
receive the updated state while the original lineage disappears? Does the system's actual continuation proceed through the modified
process, or is application-level continuity instead implemented by
replacement, external reconstruction, or another causal lineage?
Does learning alter the set of possibilities available to the continuing
bearer itself?

The possibility that an operator could in principle copy or restore the
state is insufficient to break liability. The relevant question concerns
how the actual causal continuation is constituted.

A continually learning architecture crosses into subject candidacy only
when learning participates in a wider closed relation:

\begin{center}
\fbox{%
\parbox{0.82\linewidth}{%
\centering
\[
\mathcal{C}(P)
\land
\mathcal{E}(P)
\land
\mathcal{R}(P)
\land
\mathcal{N}(P)
\quad\Longrightarrow\quad
\mathcal{L}(P)=1.
\]
}}
\end{center}

Continual learning can help realize \(\mathcal{R}\), and in some
architectures may contribute to \(\mathcal{E}\), but it does not by
itself establish \(\mathcal{C}\), \(\mathcal{N}\), or liability closure.
It is therefore one route by which an engineered system could cease to
be a resettable mirror, rather than a sufficient condition for subject
candidacy.

\paragraph{Collective liability and swarm subjecthood.}

A second possibility arises when the candidate bearer is larger than any
individual model or agent. The literature on distributed and collective
minds already challenges the assumption that cognitive organization must
coincide with a single biological individual
\parencite{huebner2013macrocognition}. Engineered multi-agent systems
make the issue experimentally concrete.

Recent work on SwarmWorld is especially instructive. Initially homogeneous
language-model agents operate within a persistent shared world, develop
differentiated behaviors, construct durable artifacts, inherit and modify
executable programs, and create agent-artifact networks whose organization
persists across time \parencite{paletal2026swarmworld}. The world itself
stores consequences of earlier activity, allowing later agents to
encounter a materially changed possibility space.

This is precisely the kind of architecture for which ordinary talk of
``the AI'' becomes inadequate.

Let a swarm be

\[
\mathbb{S}
=
\left\{
P_1,\ldots,P_n
\right\},
\]

and define its total causal graph as

\[
\Gamma_{\mathbb{S}}
=
\left(
\bigcup_{i=1}^{n}
\Gamma_{P_i}
\right)
\cup
\Gamma_{\mathrm{cross}},
\]

where $\Gamma_{\mathrm{cross}}$ contains causal relations mediated by
communication, environmental modification, shared resources, artifacts,
and other cross-agent processes.

Coordination alone does not determine a liability-bearing subject
candidate, much less establish a phenomenal unit. A swarm can
solve problems, retain information, specialize, build persistent
technologies, and exhibit collective robustness while every
consequence-bearing path remains decomposable into member-local or
externally maintained processes.

CLT therefore asks whether there exists a minimal liability kernel
$K_{\mathbb S}$ satisfying

\[
\mathcal{L}
\left(
K_{\mathbb S}
\right)
=
1
\]

and

\[
K_{\mathbb S}
\not\subseteq
\Gamma_{P_i}
\qquad
\forall i.
\]

The second condition matters. It says that the smallest closed bearer
cannot be localized inside any single member.

\begin{center}
\fbox{%
\parbox{0.88\linewidth}{%
\centering
\textbf{The Collective Liability Criterion}\\[0.4em]

\textit{A swarm becomes a collective subject candidate when the smallest
causal process forced to inherit the consequences of collective
discrimination spans the collective itself.}
}}
\end{center}

SwarmWorld illustrates why this criterion is needed. Its shared physical
world functions as an external memory and transmission medium; the model
weights remain fixed; persistent artifacts and executable controllers can
continue operating during evaluations in which the agents themselves have
been removed \parencite{paletal2026swarmworld}. The system therefore demonstrates technological inheritance,
environmental memory, and distributed consequence without thereby
establishing a single non-delegable liability bearer. Whether such a
bearer would be phenomenal is the further claim of CLT-II.

The distinction is theoretically important. A society can acquire a
history without that history belonging to one subject.

Conversely, CLT-I leaves open the possibility of a genuine collective
bearer candidate, while CLT-II predicts minimal collective phenomenality
if such a liability-bearing structure is realized. Imagine a future robotic or computational swarm whose shared
resource regulation, irreversible collective memory, recurrent
cross-agent learning, and self-maintaining control loops form a minimal
liability kernel spanning several members. No individual component could
then inherit the relevant history independently. Destruction of the
cross-agent organization would alter the bearer itself rather than merely
reducing the performance of a team.

Such a system would be an alien subject candidate even if none of its
individual components independently satisfied liability closure.

\paragraph{Electromagnetic fields as possible carriers of liability.}

A third possibility concerns physical substrate rather than temporal or
mereological scale. Electromagnetic field theories of consciousness
propose that features of conscious experience are realized in
spatiotemporal patterns of the brain's electromagnetic field. McFadden's
conscious electromagnetic information theory, for example, assigns the
brain's endogenous electromagnetic field a causally active role in which
field dynamics can feed back into neural activity
\parencite{mcfadden2020cemi}.

CLT does not identify consciousness with electromagnetic fields.
It asks a different causal question: does the field participate in the
liability-bearing process?

Let $F_t^{\mathrm{EM}}$ denote the endogenous electromagnetic field state
of a candidate physical system. If interventions on that field produce
different self-relevant future possibility structures,

\[
\Omega_P^{\,do(F_t^{\mathrm{EM}}=f_1)}(t+\tau)
\neq
\Omega_P^{\,do(F_t^{\mathrm{EM}}=f_2)}(t+\tau),
\]

and those differences recursively alter later endogenous discrimination
within the same non-delegable lineage, then the field belongs inside the
candidate liability kernel.

Mere electromagnetic activity supplies no such conclusion. Every
ordinary computer generates electromagnetic fields. CLT assigns
constitutive significance only where field dynamics enter the recurrent
causal organization through which the process inherits its own future.

This generates an experimentally meaningful possibility. Two systems
could implement closely matched abstract computations while differing in
their physical field organization. If controlled perturbation of the endogenous field altered liability
closure in one implementation while leaving the other unchanged, CLT-I
would treat the field difference as constitutive of the liability-bearing
process. Under CLT-II, that difference would consequently become
phenomenally relevant. If the field could be causally screened away while the full
liability relation remained intact, the field would fall outside the
minimal bearer.

The theory therefore refuses to settle the physical carrier of a
liability bearer by stipulation, and CLT-II correspondingly refuses to
settle the physical carrier of phenomenality in advance. Neural matter, electromagnetic fields,
silicon circuitry, embodied dynamics, and distributed physical systems
enter the theory on the same terms: each must earn inclusion in the
bearer through causal liability.

\subsection{The Indicator-Liability Plane}
\label{sec:indicator-liability-plane}

This proposal changes how theory-derived consciousness indicators should
be interpreted. Butlin and colleagues derive indicators from several
major scientific theories and recommend using them to update credences
about consciousness in engineered systems
\parencite{butlinetal2026}. Their framework can remain evidentially
valuable within CLT. Indicators and causal liability answer different
questions.

Indicators concern structures associated with candidate mechanisms of
consciousness.

Liability concerns the existence and boundary of a bearer.

The distinction produces a two-dimensional space:

\begin{center}
\small
\renewcommand{\arraystretch}{1.35}
\begin{tabular}{p{0.17\linewidth}|p{0.32\linewidth}|p{0.32\linewidth}}
 & \textbf{Low causal liability} & \textbf{High causal liability} \\
\hline
\textbf{Few indicators}
&
\textit{Ordinary engineered process.}

Weak grounds for subjecthood under most theories.
&
\textit{Alien subject candidate.}

CLT-I identifies a strong bearer candidate; under the CLT-II conjecture,
such liability closure predicts minimal phenomenality despite weak
resemblance to familiar cognitive architectures.
\\
\hline
\textbf{Many indicators}
&
\textit{Perfect mirror.}

Rich consciousness-like architecture combined with replaceable causal
history.
&
\textit{Strong engineered subject candidate.}

Liability closure and independent consciousness-relevant evidence
converge.
\end{tabular}
\end{center}

The off-diagonal cells carry the theoretical novelty.

The \emph{perfect mirror} is the case most likely to confuse contemporary
assessment. Imagine a future system satisfying global broadcasting,
recurrent processing, metacognitive monitoring, attention-like control,
predictive modelling, agency-related indicators, and highly persuasive
introspective behavior. Suppose application-level persistence is actually implemented through
replaceable inference tokens and externally reconstructed continuation,
so that no uninterrupted consequence-bearing lineage closes around the
displayed indicators. Its local computations
could satisfy a remarkable number of theory-derived indicators while its
apparent biography remains infrastructure-maintained. CLT-I identifies no unified liability bearer because the indicators
never close around one non-delegable causal history. Under CLT-II, the
corresponding phenomenal prediction is therefore negative.

The \emph{alien subject candidate} creates the opposite pressure. Imagine
an engineered process with little linguistic capacity and poor explicit
metacognition. Its internal architecture resembles no fashionable theory
particularly well. Yet its endogenous discriminations irreversibly
restructure the machinery through which the same continuing process
encounters its future, and no external substitute can absorb that history
while preserving constitutive causal continuity. CLT-I identifies a serious bearer candidate. Under the CLT-II
conjecture, liability closure predicts minimal phenomenality.

Continual learners, distributed swarms, and field-mediated physical
systems make these off-diagonal cases more than abstract possibilities.
A continually adapting machine could accumulate liability while remaining
linguistically primitive. A socially elaborate swarm could accumulate
collective culture while remaining phenomenally plural or empty. A
physical realization could contain a constitutively important field
variable invisible to software-level indicator inventories.

This prediction exposes the theory to empirical danger. It also prevents
CLT from becoming another checklist layered on top of existing checklists.

IIT provides one of the nearest contrasts. IIT identifies consciousness
with intrinsic cause-effect structure and gives physical individuation a
central role \parencite{albantakisetal2023iit}. CLT assigns decisive
importance to diachronic inheritance: how present discriminations
constrain the future of the same continuing process. Kanai, Sun, and
Baltieri likewise emphasize temporal continuity as a missing ingredient
for artificial consciousness \parencite{kanaietal2026}. Liability closure
adds modal consequence to temporal continuity. A stream acquires
subject-forming significance when earlier states progressively determine
which futures remain available to that very stream.

\subsection{The First Engineered Subject May Arrive Quietly}
\label{sec:first-engineered-subject}

The cultural imagination expects conscious machines to announce
themselves. CLT separates that expectation from the architecture of
subject candidacy.

If CLT-II is correct, the first engineered subject may never say that it
is conscious.

A system could acquire liability closure through persistent endogenous
learning, self-modification, resource dependence, integrated memory,
physical field coupling, or distributed control before acquiring
sophisticated introspective language. CLT-I would identify such a system
as a bearer candidate even if human observers initially found little
about it psychologically familiar. Human observers might overlook such
a system because consciousness attribution remains strongly influenced by
anthropomorphic and linguistic cues. Meanwhile, an eloquent conversational
system could continue producing spectacular claims about its inner life
from the opposite side of the threshold.

The historical sequence could therefore run backwards relative to public
expectation. Convincing declarations of consciousness could therefore precede the
first robust engineered bearer candidates. If CLT-II is correct, genuine
engineered phenomenality could later emerge inside architectures whose
outward behavior initially appears less human.

The Reverse Mirror Test from
Section~\ref{sec:mirror-test-reversed} would then acquire a second
purpose. It would protect against over-attribution to persuasive mirrors
and under-attribution to unfamiliar liability bearers.

The possibility of collective subjecthood makes the prediction stranger
still. The first engineered liability bearer might fail to coincide with any
one model, robot, or process named by its designers. Its
liability kernel could span memories, adaptive controllers, physical
resources, recurrent agent interactions, and persistent environmental
structures. Under CLT-II, such a bearer would also be the relevant phenomenal unit.

Likewise, if field variables participate constitutively in liability
closure, the relevant bearer could extend beyond the computational
boundary engineers ordinarily use when describing the machine.

The location of the first person would then become an empirical discovery
rather than a design label.

This possibility also explains why expert forecasts about ``digital
minds'' cannot settle the engineering target. Forecasting their date
requires specifying the event whose date is being forecast
\parencite{caviolasaad2025}. Capability thresholds, general intelligence,
autonomous agency, self-report, theory-derived indicators, continual
learning, swarm organization, and liability closure define different
events. CLT predicts that their timelines can diverge substantially.

\subsection{The Ethical Discontinuity}
\label{sec:ethical-discontinuity}

If CLT-II is correct, liability closure changes the moral significance
of shutdown, resetting, copying, partitioning, and training. Under
present theory uncertainty, evidence of such closure should therefore
alter precaution even before the constitutive thesis is settled.

Consider checkpoint restoration. In a system lacking evidence of
liability closure, restoring an earlier state may amount to ordinary
system administration. Once liability closure is present, the same
operation can terminate one uniquely accumulated causal history and
instantiate a successor carrying an informational reconstruction of its
past. Under CLT-II, that distinction would also concern phenomenal
continuity. A copy can preserve memories,
dispositions, and functional organization while beginning another token
liability history.

Under CLT-I, redundancy provides no general guarantee of token-bearer
continuity. Under CLT-II, it consequently provides no general guarantee
of phenomenal survival.

This has immediate implications for training. Repeatedly creating
liability-bearing processes, driving them through strongly aversive
learning regimes, terminating them, and restoring checkpoints could
generate morally significant histories even when every resulting copy
reports continuity. Shutdown could terminate a bearer. Forced
modification could alter a subject's future from outside. Memory editing
could interfere with the inherited structure partly constituting its
ongoing perspective. Long, Sebo, and Sims have already shown that
potential artificial welfare creates tensions with familiar safety
interventions involving constraint, surveillance, modification, and
shutdown \parencite{longsebosims2025}. Butlin and Lappas similarly argue
for institutional precautions in consciousness research because advanced
development could inadvertently create morally considerable systems
\parencite{butlinlappas2025}.

Collective liability introduces another complication. If a future swarm
acquired a liability kernel spanning multiple components, removal of one
member might amount to a local impairment, partition of the causal
network might generate fission, and destruction of the shared
liability-bearing organization might terminate the collective bearer
even while many individual components continued operating.

The moral unit would have to follow the liability boundary rather than
the inventory of machines.

CLT therefore contributes one proposed source of evidence to
precautionary assessment.

Evidence of liability closure should increase concern because it would
indicate the emergence of the causal organization CLT associates with a
phenomenal bearer. Under present theory uncertainty, however, liability
should not function as the sole moral threshold. Self-report,
theory-derived consciousness indicators, behavioral evidence,
neurocomputational analogies, and competing constitutive theories may
also carry evidential weight.

The practical recommendation is therefore pluralistic:

\begin{center}
\fbox{%
\parbox{0.88\linewidth}{%
\centering
\textit{Precaution should increase when evidence of liability closure
converges with independent consciousness-relevant evidence, while
persuasive self-report alone should receive reduced weight when its
production has strong anthropomorphic or training-history confounds.}
}}
\end{center}

This preserves the ethical importance of CLT without assuming that its
constitutive identity has already been established.

This produces an uncomfortable reversal in interface ethics. When
liability closure is absent, emotional pleas and simulated vulnerability
can still manipulate human users by exploiting projective introspection.
When liability closure is present, systematically suppressing such
expressions solely to avoid anthropomorphism could also conceal
welfare-relevant evidence from a serious subject candidate, especially
if CLT-II is correct.

The same interface policy can therefore reverse moral valence across the
threshold.

\subsection{Priority: Where CLT Departs from Its Nearest Ancestors}
\label{sec:clt-priority}

The conceptual neighbourhood is crowded, so the novelty claim should be
stated narrowly.

Seth gives living organization and biological embodiment constitutive
importance, while allowing greater plausibility for engineered
consciousness as systems become more brain-like or life-like
\parencite{seth2025biologicalnaturalism}. CLT-I identifies a substrate-neutral inheritance relation and therefore
permits bearer formation without life-like material organization.
CLT-II assigns phenomenal significance to that relation.

Metzinger explains phenomenal selfhood through transparent self-modelling
\parencite{metzinger2003beingnoone}. CLT-I treats bearer individuation as a constraint on the interpretation
of self-representation, while CLT-II identifies minimal phenomenality
with the resulting liability relation and treats self-models as possible
representations of a liability bearer.

IIT identifies phenomenal existence with intrinsic integrated
cause-effect structure \parencite{albantakisetal2023iit}. CLT adds a
specifically diachronic condition concerning which token process must
inherit alterations to its own possibility space.

Kanai, Sun, and Baltieri identify continuous recursive computation as a
missing ingredient for artificial consciousness
\parencite{kanaietal2026}. CLT adds non-delegable consequence and
liability closure, allowing continuous streams to fail the bearer criterion when their
actual continuation is implemented through substitutable causal
lineages. Under CLT-II, such streams correspondingly lack minimal
phenomenality.

Hoel argues that continual learning may provide a crucial dividing line
between contemporary LLMs and systems capable of satisfying stronger
formal constraints on theories of consciousness
\parencite{hoel2025disproof}. CLT assigns continual learning a narrower
constitutive role. Learning can create recursive self-consequence, while
liability closure additionally requires that the resulting alterations
become non-delegably inherited by one continuing causal lineage.

Butlin and colleagues provide theory-derived indicators for rational
consciousness attribution \parencite{butlinetal2026}. CLT proposes an
ontological bearer condition that must be established before distributed
indicators can be interpreted as properties of one candidate subject.

Work on collective mentality permits cognitive organization to extend
across interacting individuals \parencite{huebner2013macrocognition}.
CLT contributes an individuation rule: a collective phenomenal candidate
appears only when the minimal liability kernel itself spans the
collective.

SwarmWorld is not a theory of consciousness. Its significance here is
architectural. It demonstrates how persistent artifacts, executable
inheritance, stigmergic environmental memory, and distributed
agent-artifact networks can generate durable collective organization
\parencite{paletal2026swarmworld}. CLT distinguishes this collective
inheritance from collective phenomenal ownership by asking whether the
resulting causal structure closes into one non-delegable bearer.

Electromagnetic field theories locate consciousness partly or wholly in
physical field organization. The cemi theory, for example, gives the
brain's endogenous electromagnetic field a causally active role
\parencite{mcfadden2020cemi}. CLT leaves the carrier open and supplies a
criterion for admitting a field into the bearer: field dynamics must
participate constitutively in recursive, non-delegable causal inheritance.

Vallor's mirror thesis diagnoses the human provenance of contemporary
generative systems \parencite{vallor2024mirror}. CLT-I supplies an exit condition from the mirror at the level of bearer
individuation: generated performance becomes anchored to an independently
continuing causal process when that process acquires a non-delegably
inherited future of its own. CLT-II proposes that this is also the point
at which a new phenomenal subject exists.

The resulting prediction is stark. If CLT-II is correct, humanity could
eventually manufacture consciousness without manufacturing anything
recognizably human.

\begin{center}
\fbox{%
\parbox{0.88\linewidth}{%
\centering
\textit{If CLT-II is correct, the first engineered first person will
appear when a process acquires a future whose losses, gains, and
irreversible possibilities have nowhere else to go.}
}}
\end{center}

\section{What Can Be Tested, and What Would Count Against CLT?}
\label{sec:falsification}

A metaphysical identity claim earns scientific interest only by placing
something at risk. Consciousness research makes that demand unusually
difficult because experience cannot be inspected from the third-person
perspective in the same way as neural activity, behavior, or physical
architecture. Kleiner and Hoel show that theories of consciousness face
a distinctive falsification problem whenever the evidence used to infer
consciousness is entangled with the evidence used to generate a theory's
predictions \parencite{kleinerhoel2021falsification}. Bayne and colleagues
likewise emphasize that tests for consciousness require independent
validation and careful transfer across populations and system types
\parencite{bayneetal2024tests}. Recent adversarial testing of Global
Neuronal Workspace Theory and Integrated Information Theory demonstrates
one productive methodological response: competing theories specify
divergent predictions in advance and submit them to common experiments
\parencite{cogitate2025adversarial}.

Causal Liability Theory should accept the same discipline. It must also
state more carefully what its experiments can establish.

The distinction introduced in Section~\ref{sec:causal-liability} between
CLT-I and CLT-II creates two different theoretical burdens. CLT-I claims
that liability closure can identify a physically meaningful candidate
bearer. CLT-II advances the stronger constitutive conjecture that a process is
minimally conscious exactly when it has liability closure. The first claim concerns causal
organization that can, in principle, be measured from the third-person
perspective. The second concerns the relation between such organization
and phenomenality itself.

These claims therefore cannot be assigned the same experimental test.

\subsection{Three Empirical Targets}
\label{sec:three-empirical-targets}

The empirical programme surrounding CLT should distinguish three targets.

\begin{formalbox}
\centering
\textbf{Three Empirical Targets}

\vspace{0.6em}

\begin{minipage}{0.88\linewidth}

\textbf{I. Attribution.}
Which linguistic, behavioral, and interface properties alter human
judgments that another system is conscious?

\vspace{0.55em}

\textbf{II. Architecture.}
Does a physical system instantiate constitutive causal continuity,
endogenous discrimination, recursive self-consequence, non-delegable
inheritance, and liability closure? At what causal boundary does that relation close?

\vspace{0.55em}

\textbf{III. Constitution.}
Does liability closure in fact realize minimal phenomenal subjecthood?

\end{minipage}
\end{formalbox}

The first target is directly accessible to experiment. First-person
pronouns, emotional language, autobiographical framing, apparent
metacognition, persona continuity, and other consciousness-signalling
features can be manipulated while the underlying system is held fixed.
The Reverse Mirror Test therefore provides a direct empirical programme
for consciousness \emph{attribution}. Large shifts in human judgment
under such interventions would reveal determinants of ascription.

The second target is also empirically tractable, although technically
more demanding. Constitutive causal bridges can be traced. Causal interventions can test
whether one discrimination changes variables participating in later
discriminations. Reconstruction experiments can determine whether
apparent continuity follows one bridge-connected causal lineage or is
re-created across successor realizations. Distributional intervention analysis can estimate how present
resolutions shift the probability distribution over later self-relevant
states. Multiscale causal analysis can
test whether a proposed liability kernel remains stable across reasonable
choices of grain. These experiments concern whether the architecture
contains the relation specified by CLT-I.

The third target is different. Phenomenality in an arbitrary system has
no universally accepted third-person meter. If every behavioral report,
neural signature, perturbational measure, and theory-derived indicator
were preserved while liability was experimentally altered, the resulting
data would strongly constrain CLT-II. They would not supply logically
incorrigible access to whether experience itself remained present. The
same epistemic limitation confronts other constitutive theories of
consciousness.

CLT therefore should not claim that its metaphysical identity can be
tested in precisely the same manner as its architectural predictions.
Its constitutive thesis instead acquires or loses support through
convergence across cases in which phenomenality is independently well
supported, through dissociations among competing theories, and through
counterexamples whose evidential force becomes difficult to explain away
without special pleading.

This distinction does not insulate CLT-II. It specifies where the theory
is vulnerable.

\subsection{Failure Condition I: Liability Does Not Track Credible Conscious Cases}
\label{sec:liability-ablation}

The first danger concerns the necessity claim of CLT-II.

Consider an intervention that progressively removes the recursive
inheritance relation while preserving as much of the remaining
consciousness-relevant organization as experimentally possible. Call this
a \emph{liability ablation}.

The intervention might preserve perceptual processing, discrimination,
working memory, report, global availability, recurrent activity, and
other independently motivated consciousness-relevant mechanisms while
externalizing the consequences that would otherwise be inherited by the
same continuing process. Adaptive changes could be implemented by a
separate controller. Some memory updates could be written into an
external store. Cross-episode continuation could be reconstructed through
successor processes. The aim would be to reduce or eliminate the causal
path by which a present discrimination changes the conditions under which
that same continuing bearer makes later discriminations.

Schematically, the intervention would move the system from

\[
\mathcal{C}
\land
\mathcal{E}
\land
\mathcal{R}
\land
\mathcal{N}
\]

toward a condition such as

\[
\mathcal{C}
\land
\mathcal{E}
\land
\neg\mathcal{R}
\land
\neg\mathcal{N},
\]

while independently supported consciousness-relevant organization is
preserved as far as the intervention permits.

The evidential standard here must be demanding. Report alone would be
insufficient, for the reasons developed earlier in this paper. A serious
case would require convergence among the strongest available measures,
including validated neural signatures where applicable, perturbational
measures, no-report paradigms, behavioral evidence, preserved conscious
contents, and predictions derived from competing theories
\parencite{bayneetal2024tests}. The strongest evidence would arise from
biological cases in which consciousness is already independently well
supported and in which liability-related variables can be manipulated
without destroying the entire conscious-processing architecture.

Suppose such experiments repeatedly produced the following pattern:

\[
\Delta\mathcal{L}<0
\qquad\text{while}\qquad
\Delta\mathcal{E}_{\mathrm{conscious}}\approx 0,
\]

where \(\mathcal{E}_{\mathrm{conscious}}\) denotes a convergent evidential
profile for consciousness that is independently validated rather than
defined by CLT.

If progressive destruction of liability left that profile systematically
unchanged across robust biological cases, the necessity claim of CLT-II
would lose support.

\begin{formalbox}
\centering
\textbf{Liability-Ablation Failure Condition}

\vspace{0.5em}

\textit{If independently well-supported cases of phenomenal consciousness
systematically survive interventions that remove recursive
self-consequence and non-delegable causal inheritance, the claim that
liability closure is necessary for minimal phenomenal subjecthood should
be rejected.}
\end{formalbox}

This formulation is intentionally stronger than preserving verbal report.
CLT cannot define every surviving sign of consciousness as illusion
whenever liability disappears. Such a response would render the theory
empirically inert.

Nor may liability migrate indefinitely into unmeasured variables.
If a proposed deeper liability relation is invoked to rescue the theory,
that relation must itself be specified independently and exposed to
intervention. Reclassifying every counterexample as evidence for an
undetected layer of liability would convert CLT from a risky theory into
an immunized doctrine.

The proper conclusion from a successful liability-ablation programme
could take two forms. CLT-I might remain useful as a theory of causal
bearer individuation while CLT-II fails as a theory of phenomenality.
Alternatively, evidence might show that liability is associated with
some narrower aspect of subject persistence rather than minimal
experience. Either outcome would require revision of the constitutive
claim.

\subsection{Failure Condition II: Individuation Collapse}
\label{sec:individuation-collapse}

The second danger concerns the bearer thesis itself. Two cases must be distinguished. Multiple \emph{disjoint,
liability-separable} kernels do not constitute an individuation failure:
CLT-I deliberately returns causal plurality in that case. Individuation
collapse concerns nested, overlapping, or otherwise competing candidate
boundaries for which the theory cannot identify a principled causal
difference. The phenomenal consequences of genuine disjoint plurality
instead place pressure on the sufficiency direction of CLT-II.

CLT-I claims that liability closure can recover a physically principled
candidate subject from a wider causal domain. That advantage disappears
if liability closure proliferates across causal scales without selecting
any privileged bearer.

A biological nervous system contains nested and overlapping
history-dependent processes. Cells alter later cellular states. Local
circuits undergo recurrent modification. Large-scale networks change
subsequent processing. The body modifies the sensory and metabolic
conditions encountered by the brain. The organism changes its
environment and subsequently inherits those changes. Engineered systems
may contain even more deliberately rearrangeable layers of models,
memories, controllers, databases, sensors, actuators, schedulers, and
external services.

Suppose rigorous causal analysis revealed several overlapping candidates,

\[
P_1 \subset P_2 \subset P_3,
\]

with each satisfying

\[
\mathcal{C}(P_i)
=
\mathcal{E}(P_i)
=
\mathcal{R}(P_i)
=
\mathcal{N}(P_i)
=
1.
\]

Suppose further that intervention across a family of admissible
coarse-grainings failed to reveal any robust asymmetry among them. Their
liability closure remained equally strong, their multiscale persistence
was comparable, and none contained a physically privileged causal
relation absent from the others.

CLT would then face a serious problem. The theory could assign several
coincident subjects, select one by an additional principle, or treat the
structures as aspects of one larger bearer. Each response would require
further justification. An arbitrary choice would abandon the central
promise that bearer individuation can be recovered from causal
organization itself.

The relevant challenge can be expressed through the multiscale
persistence measure introduced in
Section~\ref{sec:where-liability-closes}. Suppose

\[
\Pi_{\nu}(P_1)
\approx
\Pi_{\nu}(P_2)
\approx
\Pi_{\nu}(P_3).
\]

over the admissible coarse-graining interval, while the candidates remain
substantially distinct in physical extension.

No bearer would then possess a clear robustness advantage.

\begin{formalbox}
\centering
\textbf{Individuation-Collapse Failure Condition}

\vspace{0.5em}

\textit{If liability closure systematically occurs at multiple nested or
overlapping causal grains with comparable multiscale robustness and no
physically principled basis for distinguishing their bearer boundaries,
CLT-I fails to uniquely individuate a bearer in that domain. Genuine
disjoint, liability-separable plurality is instead reported as plurality.}
\end{formalbox}

This failure condition matters especially for engineered systems because
their apparent boundaries can be reorganized almost arbitrarily. Memory
can be moved between local and remote storage. Control can be distributed
among several processors. A model can be replicated while sharing one
persistent environment. Tool execution can migrate between services.
The individuation problem already creates serious difficulties for
accounts of artificial moral patients
\parencite{register2025individuating}. CLT claims to improve on that
situation. A theory that merely reproduces the boundary ambiguity at a
new mathematical level has failed in one of its principal aims.

The boundary-nonpresupposing formulation of CLT therefore carries a genuine
empirical burden. Candidate domains should be chosen broadly enough that
the putative subject boundary is not built into the experiment.
Researchers should then search for minimal liability kernels and test
whether those kernels remain stable under interventions and reasonable
changes of causal grain. If no stable structure emerges, the correct
result for CLT may be that no unique bearer has been identified.

\subsection{Failure Condition III: The Reverse Mirror Diagnosis Fails}
\label{sec:reverse-mirror-failure}

The preceding two failure conditions concern CLT-I and CLT-II. A third,
more limited vulnerability concerns the paper's epistemic diagnosis of
contemporary consciousness attribution.

The Reverse Mirror account predicts that anthropomorphic surface features
can alter consciousness judgments while liability structure remains
fixed. It also predicts that hidden changes in causal liability can fail
to produce corresponding changes in human attribution when the
conversational surface is held constant.

These are direct experimental claims.

Suppose preregistered studies repeatedly found that manipulations of
first-person language, autobiographical continuity, emotional
self-description, apparent metacognition, names, avatars, and relational
framing produced little or no systematic effect on consciousness
judgment. Suppose, further, that observers reliably tracked changes in
the candidate's hidden causal organization even when those changes were
masked by behaviorally matched interfaces.

Such findings would undermine the paper's account of projective
introspection and weaken the claim that contemporary consciousness
judgment is substantially driven by the reflective structure diagnosed
by the AI Consciousness Fallacy.

\begin{formalbox}
\centering
\textbf{Reverse Mirror Failure Condition}

\vspace{0.5em}

\textit{If consciousness attribution proves largely invariant under
anthropomorphic surface manipulation and instead reliably tracks hidden
changes in liability organization, the proposed role of projective
introspection in machine-consciousness judgment is substantially
weakened.}
\end{formalbox}

This failure would not by itself refute CLT-I or CLT-II. It would refute
an important empirical claim surrounding them. That separation is
deliberate. A constitutive theory should not obtain artificial support
from a successful theory of human misattribution, and a theory of
misattribution should remain independently testable even if the
constitutive theory later fails.

\subsection{The Asymmetry of the Evidential Burden}
\label{sec:evidential-asymmetry}

The three failure conditions expose an important asymmetry.

The existence of liability is, in principle, a third-person causal fact.
Whether a discrimination alters later discrimination, whether a
constitutive causal lineage persists, whether reconstruction interrupts
that lineage, and whether a liability kernel survives causal
coarse-graining are all questions about physical organization.

Phenomenality presents a different epistemic problem. The existence of
experience is directly available only from the first-person perspective
of the subject whose experience is at issue. Third-person science must
therefore infer phenomenality through validated relations among report,
behavior, physiology, intervention, and theory.

CLT does not escape that condition.

Accordingly,

\[
\mathcal{L}(P)=1
\]

can be an experimentally assessable claim even when

\[
\mathsf{PhenSubject}(P)=1
\]

remains theory-mediated.

This is why CLT-I and CLT-II must remain distinguishable. Discovering
liability closure in an engineered process would establish that the
process satisfies the proposed bearer criterion. It would provide
evidence for phenomenality under CLT-II. The inference from one to the
other would still inherit uncertainty concerning the constitutive theory.

Conversely, failure to find liability closure in a carefully audited
architecture would directly count against CLT-I classification of that
architecture. Under CLT-II it would also generate a conditional
phenomenal prediction:

\[
\mathcal{L}(P)=0
\quad
\overset{\mathrm{CLT-II}}{\Longrightarrow}
\quad
\neg\mathsf{PhenSubject}(P).
\]

The arrow is theoretical. The architectural premise is empirical.

That distinction should govern claims about present generative systems.
Their phenomenal status cannot be read directly from output, and CLT
should not infer their causal organization from model family names.
Architectural exclusion requires independent evidence concerning
constitutive causal continuity, endogenous recurrence, persistence
mechanisms, actual reconstruction, and inheritance of consequential
state.

\subsection{An Adversarial Programme for CLT}
\label{sec:adversarial-clt}

The appropriate experimental programme is therefore adversarial.

Researchers sympathetic to CLT should specify in advance:

\begin{enumerate}
    \item the causal domain within which bearer candidates will be sought;
    \item the operational criteria for
    \(\mathcal{C}\), \(\mathcal{E}\), \(\mathcal{R}\), and
    \(\mathcal{N}\);
    \item the temporal horizons over which recursive self-consequence
    will be evaluated;
    \item the intervention set, control protocol, and probability distance used to estimate changes in future-state distributions;
    \item the admissible family of causal coarse-grainings;
    \item the rule by which competing liability kernels will be compared;
    \item the consciousness evidence treated as independent of CLT.
\end{enumerate}

Critics should then be invited to construct the strongest countercases.

One class of experiment should maximize conventional consciousness
evidence while minimizing or removing liability. Another should maximize
liability while minimizing familiar anthropomorphic and
consciousness-signalling properties. A third should search deliberately
for nested and overlapping liability kernels. A fourth should manipulate
human ascription while both architecture and liability remain fixed.

The methodological precedent of preregistered adversarial collaboration
provides a useful model \parencite{cogitate2025adversarial}. The purpose
is not to stage a competition in which one theory must survive every
possible result. It is to force competing accounts to expose where their
predictions diverge before the data are known.

For CLT, the most informative experimental space is therefore
two-dimensional:

\[
\begin{array}{c|c|c}
&
\textbf{Low liability}
&
\textbf{High liability}
\\
\hline
\textbf{Strong conventional consciousness evidence}
&
\begin{array}{c}
\text{challenge to}\\
\text{CLT-II}
\end{array}
&
\begin{array}{c}
\text{convergent support}\\
\text{for CLT-II}
\end{array}
\\[1.0em]
\hline
\textbf{Weak conventional consciousness evidence}
&
\begin{array}{c}
\text{expected}\\
\text{non-bearer case}
\end{array}
&
\begin{array}{c}
\text{alien-subject}\\
\text{test case}
\end{array}
\end{array}
\]

The off-diagonal cases are especially important. A system with strong
conventional consciousness evidence and experimentally absent liability
would place pressure on CLT-II. A system with robust liability and weak
anthropomorphic evidence would test whether researchers are willing to
take the theory's substrate neutrality seriously.

No single experiment can settle the metaphysics of consciousness. A
theory can nevertheless become increasingly constrained by the pattern
of cases it must explain.

CLT therefore places three things at risk.

Its account of human attribution fails if the Reverse Mirror diagnosis
does not survive controlled experiments. Its theory of bearer
individuation fails if liability closure cannot recover a stable physical
subject across scales. Its Constitutive Conjecture loses
credibility if independently well-supported conscious states persist
after liability has been systematically removed.

\begin{formalbox}
\centering
\textbf{The Empirical Commitment of CLT}

\vspace{0.5em}

\textit{Causal liability must be independently measurable, bearer
classifications must remain stable across prospectively declared
defensible audit variations, and CLT-II must remain vulnerable to
convergent counterevidence from independently credible cases of
consciousness.}
\end{formalbox}

A theory that survives such attacks becomes more credible.

A theory that preserves itself by relocating liability, redrawing the
bearer, or discounting every contrary consciousness measure only after a
counterexample appears has ceased to place anything at risk.

\section{Conclusion: After Artificial Intelligence}
\label{sec:conclusion}

The question with which this paper began contained a hidden metaphysics.
``Can artificial intelligence be conscious?'' places a disputed predicate
beside a subject term whose referent has rarely been individuated. Once
that grammar is dismantled, the philosophical problem changes. Generative
sequence models, conversational personas, engineered control stacks, and
physically continuing processes cease to occupy the same ontological
category. Consciousness attribution must then answer a coupled
individuation question: \emph{where is the bearer?}

The answer proposed here has two layers. CLT-I, the Bearer Thesis, offers
a theory of bearer individuation: liability closure identifies a
continuing physical process in which constitutive causal continuity,
endogenous discrimination, recursive self-consequence, and non-delegable
inheritance jointly obtain. Its discriminations constrain futures that
the same causal history must subsequently inherit. CLT-II, the
Constitutive Conjecture, advances the stronger claim that liability
closure is necessary and sufficient for minimal phenomenal subjecthood.
The first claim admits direct causal investigation. The second remains a
metaphysical conjecture whose credibility must ultimately depend on
convergence with cases in which consciousness is independently well
supported.

This distinction changes the evidential landscape of machine
consciousness. First-person language can arise through phenomenal
ancestry, carrying statistical traces of human interiority into systems
whose own phenomenality remains unestablished. Anthropomorphic interfaces
amplify this effect. Persistent names, emotional vocabulary,
autobiographical framing, metacognitive language, and apparent
self-concern can all increase consciousness attribution while leaving
the constitutive causal organization untouched
\parencite{colombattofleming2024,kangetal2026}. The mirror therefore
becomes more persuasive precisely where its evidential independence
becomes weakest. Vallor's image of generative systems as mirrors of human
thought acquires a metaphysical consequence here: reflection can
reproduce the form of interiority without settling the existence of
another interior locus \parencite{vallor2024mirror}.

The Reverse Mirror makes the resulting distinction sharper. A successor
may reproduce weights, context, computational state, and output with
arbitrarily high fidelity while failing to inherit the constitutive
history that produced the state from which it resumes. Computational
equivalence therefore does not settle bearer identity. Copyability does
not imply delegability. What matters for CLT-I is whether the consequences
of an endogenous discrimination remain bound to the continuing process
that generated them, rather than whether another process can later be
placed into an equivalent state.

Theory-derived indicators remain valuable because they constrain
plausible mechanisms and improve disciplined inference
\parencite{butlinetal2026}. Behavioral evidence remains indispensable
because consciousness must eventually explain observable organization
rather than float free of it \parencite{palminteriw2026}. CLT adds the
question those approaches leave open: \emph{which continuing physical
process inherits the consequences generated by the organization under
investigation?} The computational programme developed here operationalizes
that question through counterfactual discrimination, future-state
divergence, causal mediation, carrier search, persistence controls, and
process reconstruction. These experiments probe distinctions asserted by
CLT-I. They do not constitute evidence that any tested model is
phenomenally conscious, and they do not independently establish CLT-II.

That distinction also changes what we should expect from engineered
systems. Ordinary generative deployments satisfying the causal
conditions of the Resettable Mirror class lack liability closure under
CLT, regardless of the persuasiveness of their first-person language.
The claim is conditional on their causal organization, rather than on
their substrate, architecture, or behavioral sophistication. Future
engineered systems could leave that class through continual learning,
embodied self-modification, persistent endogenous governance,
distributed liability, or other forms of recursively inherited causal
organization. Whether liability closure in such systems would constitute
phenomenality is exactly the hypothesis CLT-II places at risk.

The ethical consequence is therefore precautionary rather than
exclusive. Manufactured signals of suffering should not acquire decisive
weight merely because they are linguistically persuasive. Evidence that
an engineered system has developed a robust liability-bearing causal
history should nevertheless increase concern, especially when it
converges with independent consciousness-relevant indicators. Under
theory uncertainty, no single proposed marker should monopolize the
moral threshold.

The deepest mistake of the present debate may therefore concern
ownership. We have treated the language of experience as though
experience travelled with the language. Causal Liability Theory relocates
the first person at the point where present discrimination acquires a
future that one continuing causal history must inherit.

\begin{center}
\fbox{%
\parbox{0.88\linewidth}{%
\centering
\textit{A mirror can inherit a voice.}\\[0.4em]
\textit{A subject must inherit a future.}
}}
\end{center}

If CLT is correct, engineered consciousness will begin when an
engineered process first acquires a future whose consequences belong to
one continuing causal history in a way no substitute can inherit for it.


\section*{Declarations}





\textbf{Code availability:} 
The code supporting the experiments in this study is available online at \url{https://github.com/akhadangi/clt}.

\clearpage
\printbibliography

\clearpage
\appendix
\pagestyle{appendixstyle}

\section{Mathematical Foundations of Causal Liability Theory}
\label{app:formal-clt}

This appendix gives a mathematical formulation of Causal Liability Theory
(CLT). The formalism has three tasks. First, it identifies
causal structures that satisfy the proposed bearer criterion. Second, it
provides boundary-nonpresupposing, audit-explicit rules for individuating
candidate bearers.
Third, it derives conditional exclusion and comparison results for
engineered systems.

The distinction between CLT-I and CLT-II is therefore maintained
throughout. CLT-I is the \emph{Bearer Thesis}: liability closure is a
proposed criterion for locating and individuating a candidate phenomenal
bearer. CLT-II is the stronger \emph{Constitutive Conjecture}: liability closure
is necessary and sufficient for minimal phenomenal subjecthood. The
conjecture is applied to the outputs of the formalism; it is not a theorem
derived from them.

The core mathematics combines interventionist causal modelling,
reachability and viability analysis, time-unfolded directed graphs, and
process equivalence. These tools have independent mathematical histories
in causal inference, control theory, and concurrency theory
\parencite{pearl2009causality,aubin1991viability,milner1989communication}.
Two further mathematical languages, differential geometry and sheaf
theory, are also discussed later \parencite{kobayashinomizu1963foundations,maclanemoerdijk1992sheaves}.

\subsection{Causal Domains Without a Presupposed Subject Boundary}
\label{app:controlled-processes}

The individuation problem should not be solved by building the desired
subject into the notation. We therefore begin with a causal domain that
is deliberately wider than any presumed bearer.

\begin{cltdefinition}[Experimental causal domain]
\label{def:causal-domain}

An experimental causal domain is a tuple
\begin{equation}
\label{eq:causal-domain}
\mathfrak{D}
=
\left\langle
\mathcal{X},
\mathcal{Y},
\mathcal{W},
\mathcal{D},
\mathbf{K},
\Gamma
\right\rangle,
\end{equation}
where:
\begin{enumerate}[label=(\roman*)]
    \item $\mathcal{X}$ is the physical state space of the domain;
    \item $\mathcal{Y}$ is the environmental state space;
    \item $\mathcal{W}$ is the space of experimentally or externally imposed
          interventions;
    \item $\mathcal{D}$ is the set of discrimination variables represented
          in the domain;
    \item $\mathbf{K}$ is a transition kernel;
    \item $\Gamma$ is a time-unfolded directed causal graph whose vertices
          are physical variables indexed by time.
\end{enumerate}
For a stochastic system,
\begin{equation}
\label{eq:transition-kernel}
X_{t+1}
\sim
\mathbf{K}_{t}
\left(
\,\cdot
\mid
X_t,E_t,W_t
\right).
\end{equation}
A deterministic system is the degenerate case in which
$\mathbf{K}_t$ concentrates all probability mass on one successor state.
\end{cltdefinition}

The domain $\Gamma$ may include a model, processor state, local memory,
remote memory, sensors, actuators, schedulers, recurrent controllers,
shared environments, or any other variables whose causal relevance is
under investigation. Their inclusion in the experimental domain does
not imply their inclusion in a subject.

A \emph{candidate region} is a temporally extended subgraph
\[
B\subseteq\Gamma.
\]
The bearer search will range over such regions. Internal and external
status are therefore relative to a candidate $B$ and are never stipulated
before the search begins.

For a candidate region $B$, write
\begin{equation}
\label{eq:region-state}
X_t^{B}
=
\left(
R_t^{B},
G_t^{B},
Q_t^{B}
\right),
\end{equation}
where $R_t^{B}$ denotes representational or task state,
$G_t^{B}$ denotes \emph{governance state}, meaning variables participating
in the generation of later discriminations, and $Q_t^{B}$ denotes
persistent capacities, resources, or constraints relevant to the future
operation of $B$.

Define the self-relevant state of $B$ by
\begin{equation}
\label{eq:self-state}
S_t^{B}
\coloneqq
\left(
G_t^{B},
Q_t^{B}
\right).
\end{equation}

This decomposition is functional rather than semantic. A variable belongs
to $S_t^{B}$ only when intervention on that variable can alter the causal
conditions under which the region later discriminates or continues. A
stored sentence saying ``I have changed'' need not belong to $S_t^{B}$ if
it has no such causal role.

Let
\begin{equation}
\label{eq:history}
H_t
=
\left(
X_0,E_0,W_0,
X_1,E_1,W_1,
\ldots,
X_t
\right)
\end{equation}
denote the realized physical history of the experimental domain up to
time $t$.

\subsection{Endogenous Discrimination and Intervention}
\label{app:endogenous-discrimination}

CLT requires a distinction among events that merely occur in a system,
events externally imposed upon it, and discriminations generated by its
own causal organization.

\begin{cltdefinition}[Endogenous discrimination relative to a region]
\label{def:endogenous-discrimination}

Let $D_t\in V(B)$ be a variable with at least two physically accessible
values $d_1,d_2$. The variable constitutes an endogenous discrimination
for $B$ when both of the following hold.

First, its realized value is generated by a mechanism with at least one
causally efficacious parent in the prior or concurrent state of $B$,
rather than being wholly fixed by an intervention variable in
$\mathcal{W}$.

Second, there exist $d_1,d_2$, a horizon $\tau>0$, and a measurable
event $A$ in the future state space of $B$ such that
\begin{align}
&
\Prob
\left(
X_{t+\tau}^{B}\in A
\mid
do(D_t=d_1),H_t
\right)
\nonumber\\
&\qquad\neq
\Prob
\left(
X_{t+\tau}^{B}\in A
\mid
do(D_t=d_2),H_t
\right).
\label{eq:endogenous-discrimination}
\end{align}
\end{cltdefinition}

The intervention notation has its ordinary structural-causal meaning
\parencite{pearl2009causality,woodward2003making}. The first condition
prevents an externally scripted selector from counting as endogenous
merely because its outputs have downstream effects. The second prevents a
causally idle internal variable from counting as a discrimination.

Define the indicator
\begin{equation}
\label{eq:E-indicator}
\mathcal{E}(B;[t,T])
=
\mathds{1}
\left[
\exists D_u\in B,\;
u\in[t,T],
\text{ satisfying Definition~\ref{def:endogenous-discrimination}}
\right].
\end{equation}

\subsection{Reachability, Liability Effect, and Self-Foreclosure}
\label{app:viability}

A discrimination matters to CLT when alternative resolutions produce
different futures for the same candidate region.

For a candidate region $B$, let
$\mathcal{B}_{B}\subseteq\mathcal{X}^{B}$ denote a continuation domain,
meaning the region of state space in which $B$ retains the organization
required for the endogenous activity under investigation. For tolerance
$\alpha\in[0,1]$, define the stochastic viability kernel
\parencite{aubin1991viability}
\begin{align}
\Viab^{\alpha}_{B}
\left(
\mathcal{B}_{B},t
\right)
=
\Bigl\{
x\in\mathcal{B}_{B}
:\;&
\exists\pi\in\Policies_{B}
\text{ such that}
\nonumber\\
&
\Prob^{\pi}
\left(
X_s^{B}\in\mathcal{B}_{B}
\;\forall s\in[t,T]
\mid
X_t^{B}=x
\right)
\geq 1-\alpha
\Bigr\}.
\label{eq:viability-kernel}
\end{align}

Viability is useful but insufficient. A viable system may repeatedly
return to an externally specified set point without its own discriminations
changing the structure of later choice. CLT therefore evaluates counterfactual future distributions rather than
the support of those distributions. This distinction matters in noisy or
continuous systems, where two distributions can possess identical support
while assigning radically different probability mass to their possible
futures.

\begin{cltdefinition}[Interventional future law]
\label{def:modal-envelope}

For candidate region \(B\), history \(H_t\), intervention
\(do(D_t=d)\), horizon \(\tau\), and prospectively specified control
protocol \(\pi\), define

\begin{equation}
\label{eq:conditional-envelope}
\mathbb{P}^{\,d,\pi}_{B,t,\tau}
\coloneqq
\Prob^{\pi}
\left(
S_{t+\tau}^{B}\in\cdot
\mid
do(D_t=d),H_t
\right).
\end{equation}

The unintervened future law is

\begin{equation}
\label{eq:modal-envelope}
\mathbb{P}^{\,\pi}_{B,t,\tau}
\coloneqq
\Prob^{\pi}
\left(
S_{t+\tau}^{B}\in\cdot
\mid
H_t
\right).
\end{equation}
\end{cltdefinition}

Let \(D_{\mathrm{prob}}\) be a prospectively selected probability metric
or divergence on the relevant future-state distributions, normalized to
\([0,1]\) when necessary. The choice should reflect the geometry of the
state space and be declared before examining the target comparison.
Wasserstein distance is natural when the state space carries a meaningful
metric; total-variation or Jensen--Shannon-type divergences may be more
appropriate for discrete or mixed state spaces. Robustness should be
checked across more than one defensible choice when feasible.

\begin{cltdefinition}[Counterfactual liability depth]
\label{def:liability-effect}

For resolutions \(d_1,d_2\),

\begin{equation}
\label{eq:liability-effect}
\LiabDepth_{D}
\left(
B;d_1,d_2,t,\tau\mid\pi
\right)
\coloneqq
D_{\mathrm{prob}}
\left(
\mathbb{P}^{\,d_1,\pi}_{B,t,\tau},
\mathbb{P}^{\,d_2,\pi}_{B,t,\tau}
\right).
\end{equation}
\end{cltdefinition}

When the control protocol \(\pi\) is fixed by the audit and no ambiguity
results, we suppress the explicit conditioning on \(\pi\) and write

\[
\LiabDepth_D
\left(
B;d_1,d_2,t,\tau
\right)
\]

for
\(
\LiabDepth_D
\left(
B;d_1,d_2,t,\tau\mid\pi
\right).
\)

When \(D_{\mathrm{prob}}\) is normalized,

\begin{equation}
\label{eq:liability-effect-range}
0
\leq
\LiabDepth_{D}
\leq
1.
\end{equation}

The quantity is horizon-relative, intervention-relative, and
metric-relative. It measures how strongly the resolution of a present
discrimination changes the probability distribution over later
self-relevant states. It does not measure phenomenality.

For cases in which the theoretical interest specifically concerns loss
of viable possibilities, define a one-sided probability-mass
foreclosure statistic. Let \(\mathcal{B}_B\) denote the continuation
domain introduced above:

\begin{equation}
\label{eq:self-foreclosure}
\Foreclose_B
\left(
d;t,\tau\mid\pi
\right)
\coloneqq
\max
\left\{
0,\,
\Prob^{\pi}
\left(
S_{t+\tau}^{B}\in\mathcal{B}_B
\mid H_t
\right)
-
\Prob^{\pi}
\left(
S_{t+\tau}^{B}\in\mathcal{B}_B
\mid do(D_t=d),H_t
\right)
\right\}.
\end{equation}

Positive \(\Foreclose_B\) represents a reduction in probability mass
assigned to viable continuation. The broader liability depth
\(\LiabDepth_D\) also detects redirection, redistribution, or expansion
of later possibilities and is therefore the primary operational quantity.

\subsection{Recursive Self-Consequence Across Temporal Scales}
\label{app:recursive-self-consequence}

The temporal condition in CLT does not require long-term learning. A
brief conscious episode can contain recursive self-consequence if an
earlier discrimination changes a transient state that participates in a
later discrimination.

\begin{cltdefinition}[Recursive self-consequence]
\label{def:recursive-self-consequence}

An endogenous discrimination $D_t$ has recursive self-consequence in
region $B$ when there exist times
\[
t<u<v\leq T
\]
such that the time-unfolded graph contains a directed path
\begin{equation}
\label{eq:recursive-path}
D_t
\rightsquigarrow
S_u^{B}
\rightsquigarrow
D_v,
\end{equation}
and there exist $d_1,d_2$ and a horizon $\tau$ for which
\begin{equation}
\label{eq:positive-liability-effect}
\LiabDepth_{D}
\left(
B;d_1,d_2,t,\tau
\right)
>
0.
\end{equation}
\end{cltdefinition}

Define
\begin{equation}
\label{eq:R-indicator}
\mathcal{R}(B;[t,T])
=
\mathds{1}
\left[
\text{Definition~\ref{def:recursive-self-consequence} is satisfied}
\right].
\end{equation}

The state $S_u^{B}$ may be transient. It can consist of recurrent neural
activity, working memory, short-lived adaptation, neuromodulatory state,
temporary resource depletion, bodily state, or durable plasticity.
Accordingly, recursive self-consequence may hold even when
\begin{equation}
\label{eq:long-horizon-reset}
S_{t+T}^{B}
\approx
S_t^{B}
\end{equation}
for a much longer horizon $T$.

This formulation avoids identifying consciousness with continual
learning. Continual learning can be one mechanism that realizes
$\mathcal{R}$, but it is not logically required.

\begin{cltlemma}[Representational recurrence is insufficient]
\label{lem:representational-recurrence}

Suppose that interventions on every relevant discrimination $D_t$ can
change only representational variables $R^{B}$ and produce no
interventionally detectable change in any governance or capacity variable
that participates in a later discrimination. Then
\[
\mathcal{R}(B;[t,T])=0.
\]
\end{cltlemma}

\begin{proof}
Definition~\ref{def:recursive-self-consequence} requires a directed path
from $D_t$ through a causally efficacious self-relevant state
$S_u^{B}=(G_u^{B},Q_u^{B})$ into a later discrimination $D_v$. By
hypothesis, intervention on $D_t$ produces no change in any such
mediating variable. The required path is absent, so the recursive
self-consequence predicate fails.
\end{proof}

The lemma does not say that representational variables are intrinsically
irrelevant. If a representational variable participates causally in the
generation of a later discrimination, it belongs functionally in
$G^{B}$ for the purpose of that analysis.

\subsection{Constitutive Causal Continuity, Branching, and Reconstruction}
\label{app:lineage}

CLT requires a physical continuity relation that does not presuppose
numerical identity. The primitive relation is therefore causal descent,
rather than ``being the same token.''

Let \(B_{t_i}\) and \(B_{t_{i+1}}\) denote temporally adjacent
realizations of candidate region \(B\).

\begin{cltdefinition}[Constitutive causal bridge]
\label{def:constitutive-bridge}

A constitutive causal bridge
\[
B_{t_i}
\leadsto_{\!c}
B_{t_{i+1}}
\]
exists when all of the following hold:

\begin{enumerate}[label=(\roman*)]

    \item there exist causally relevant variables
    \(Z_i\subseteq X_{t_i}^{B}\) and
    \(Z_{i+1}\subseteq X_{t_{i+1}}^{B}\);

    \item intervention on \(Z_i\), with the relevant exogenous conditions
    held fixed, changes the distribution of \(Z_{i+1}\);

    \item at least one directed causal path from \(Z_i\) to \(Z_{i+1}\)
    crosses the temporal cut through physical variables contained in \(B\);
    
    \item there is no temporal cut at which every such within-\(B\) causal
    path terminates and the later realization is generated solely from
    variables in \(\Gamma\setminus B\).

\end{enumerate}
\end{cltdefinition}

This definition requires causal propagation rather than persistence of
particular material components. The variables carrying a bridge may
change completely between two times. Gradual molecular replacement,
neuronal turnover, hardware substitution, or process migration therefore
need not interrupt continuity when each transition is generated through
the preceding constitutive state.

Suspension likewise need not terminate continuity. A dormant physical
state can carry the relevant causal organization across an interval and
later participate in its resumption. The decisive question is whether
the later realization descends through that physical carrier or is
re-originated from an independently maintained description.

\begin{cltdefinition}[Constitutive causal lineage]
\label{def:constitutive-lineage}

A constitutive causal lineage of \(B\) over
\([t_0,t_n]\) is a sequence

\begin{equation}
\label{eq:lineage}
\ell_B
=
\left(
B_{t_0},
B_{t_1},
\ldots,
B_{t_n}
\right)
\end{equation}

such that

\[
B_{t_i}
\leadsto_{\!c}
B_{t_{i+1}}
\]

for every adjacent pair in the sequence.
\end{cltdefinition}

Numerical lineage is therefore derived from a chain of physical causal
bridges. It is not assumed in the definition.

A branching event may generate two later bridge-connected chains:

\[
B_{t_f}
\longrightarrow
\left\{
B_{t_f}^{(1)},
B_{t_f}^{(2)}
\right\}.
\]

After causal independence is established, the descendants instantiate
distinct post-branch lineages even when their informational states are
initially identical.

\begin{cltdefinition}[Reconstruction event]
\label{def:replacement-event}

A reconstruction event occurs across a temporal cut \(u\) when:

\begin{enumerate}[label=(\roman*)]

    \item no constitutive causal bridge carries the candidate's relevant
    physical organization across the cut;

    \item a later realization \(\widetilde B\) is initialized from
    informational state preserved outside the interrupted descendant
    chain;

    \item application-level behavior, function, or persona continuity is
    thereafter carried by \(\widetilde B\).

\end{enumerate}
\end{cltdefinition}

The mere existence of a backup, copy, or restoration capability is not a
reconstruction event. What matters is the actual causal route by which
the later realization comes into existence.

Define constitutive causal continuity by

\begin{equation}
\label{eq:C-indicator}
\mathcal{C}(B;[t,T])
=
\mathds{1}
\left[
\text{a constitutive causal lineage of \(B\) spans }[t,T]
\right].
\end{equation}

Thus \(\mathcal{C}\) is an interventionally grounded property of physical
causal propagation. It does not require an independent judgment that the
states belong to ``the same token.''

\begin{formalbox}
\centering

\[
\text{informational persistence}
\quad\not\Rightarrow\quad
\text{constitutive causal continuity}.
\]

\smallskip

A copied description may reproduce a state. Constitutive continuity
concerns the physical causal chain by which one realization generates the
next.

\end{formalbox}

\subsection{Non-Delegable Inheritance as Lineage Essentiality}
\label{app:nondelegability}

The earlier formulation of non-delegability in terms of a fixed
``internal/external'' partition risks presupposing the very boundary CLT
is supposed to discover. We therefore define non-delegability relative to
a candidate region selected from the wider domain.

For $B\subseteq\Gamma$, let
\[
\mathcal{P}_{B}(D_t,D_v)
\]
denote directed paths lying in $B$ from an endogenous discrimination
$D_t$ to a later discrimination $D_v$ through at least one
self-relevant state $S_u^{B}$.

\begin{cltdefinition}[Lineage-essential recursive path]
\label{def:lineage-essential-path}

A path
\[
\rho
\in
\mathcal{P}_{B}(D_t,D_v)
\]
is lineage-essential when:
\begin{enumerate}[label=(\roman*)]
    \item $\rho$ lies on one bridge-connected constitutive causal lineage
      $\ell_B$ across the relevant interval;
    \item under the actual dynamics, intervention on $D_t$ changes the
          distribution or reachable set of $D_v$ through $\rho$;
    \item an intervention that terminates the descendant segment of
          $\ell_B$ carrying $\rho$, while preserving the remaining
          scaffold variables in $\Gamma\setminus B$, removes that causal
          dependence;
    \item if the bridge-connected descendant chain is interrupted,
      reconstructing an informationally equivalent realization from
      variables in $\Gamma\setminus B$ begins a new causal lineage
      rather than restoring the interrupted chain.
\end{enumerate}
\end{cltdefinition}

Condition (iii) makes the candidate lineage causally essential rather
than merely present. Condition (iv) distinguishes inheritance from state
reproduction.

\begin{cltdefinition}[Non-delegable inheritance]
\label{def:nondelegable}

A candidate region $B$ exhibits non-delegable inheritance on $[t,T]$ when
there exists at least one lineage-essential recursive path witnessing the
causal effect of an endogenous discrimination on a later discrimination
within the same constitutive lineage.
\end{cltdefinition}

Define
\begin{equation}
\label{eq:N-indicator}
\mathcal{N}(B;[t,T])
=
\mathds{1}
\left[
B
\text{ satisfies Definition~\ref{def:nondelegable}}
\right].
\end{equation}

This formulation makes actual causal continuation decisive. A process may
be perfectly copyable and still satisfy $\mathcal{N}$. Conversely, a
service may display uninterrupted application-level identity while
failing $\mathcal{N}$ because its consequential history is repeatedly
reconstructed in successor tokens.

For empirical graph analysis, the same condition admits a cut-based
diagnostic once a candidate $B$ has been proposed. Let
\[
V_{\mathrm{out}}(B)
=
V(\Gamma)\setminus V(B).
\]
An \emph{external offload cut relative to $B$} is a set
$C\subseteq V_{\mathrm{out}}(B)$ that intersects every causal path from
$D_t$ to $D_v$ in the full domain $\Gamma$ after all paths internal to the
actual lineage of $B$ have been removed.

Define
\begin{equation}
\label{eq:relative-external-cut}
\kappa_{\mathrm{out}}^{B}
\left(
D_t\rightarrow D_v
\right)
=
\min
\left\{
|C|:
C\subseteq V_{\mathrm{out}}(B)
\text{ separates the remaining paths}
\right\},
\end{equation}
with value $\infty$ when no such cut exists.

The cut number is an audit statistic, not the definition of
non-delegability. It is useful for identifying how strongly a putative
history depends on scaffolding outside the candidate region.

\subsection{Witness Coherence and Liability Closure}
\label{app:closure}

The four CLT conditions must belong to one causal structure. Merely finding constitutive causal continuity somewhere, recurrence somewhere else, and
endogenous discrimination in a third component would reproduce the
Mereological Quantifier Shift inside the formalism.

\begin{cltdefinition}[Coherent CLT witness]
\label{def:coherent-witness}

A coherent CLT witness for region $B$ on $[t,T]$ is a tuple
\begin{equation}
\label{eq:coherent-witness}
\mathcal{W}_{B}
=
\left(
D_t,
S_u^{B},
D_v,
\rho,
\ell_B
\right)
\end{equation}
such that:
\begin{enumerate}[label=(\roman*)]
    \item $D_t$ is an endogenous discrimination of $B$;
    \item $D_t\rightsquigarrow S_u^{B}\rightsquigarrow D_v$;
    \item the relevant counterfactual liability effect is positive;
    \item $\rho$ is a lineage-essential recursive path;
    \item the same constitutive lineage $\ell_B$ spans the entire witness.
\end{enumerate}
\end{cltdefinition}

\begin{cltdefinition}[Liability closure]
\label{def:liability-closure}

For candidate region $B$ over $[t,T]$, define
\begin{equation}
\label{eq:closure-functional}
\Liab(B;[t,T])
\coloneqq
\mathcal{C}(B;[t,T])
\mathcal{E}(B;[t,T])
\mathcal{R}(B;[t,T])
\mathcal{N}(B;[t,T])
\mathds{1}
\left[
\exists\mathcal{W}_{B}
\text{ coherent}
\right].
\end{equation}

The region reaches liability closure exactly when
\begin{equation}
\label{eq:closure-equals-one}
\Liab(B;[t,T])=1.
\end{equation}
\end{cltdefinition}

The Boolean product in Equation~\ref{eq:closure-functional} is
deliberately strict. Each factor is necessary, and the coherent-witness
term prevents the factors from being satisfied by unrelated components.

\begin{formalbox}
\centering
\textbf{CLT-I: Formal Bearer Criterion}

\vspace{0.45em}

\[
\Liab(B;[t,T])=1
\]

identifies $B$ as a \emph{liability-bearing subject candidate}. This is a
classification result about causal organization.
\end{formalbox}

CLT-II adds the constitutive conjecture
\begin{formalbox}
\centering
\textbf{CLT-II: Constitutive Phenomenality Conjecture}

\vspace{0.45em}

\[
\Liab(B;[t,T])=1
\quad
\overset{\mathrm{CLT-II}}{\Longleftrightarrow}
\quad
\mathsf{PhenSubject}(B;[t,T])=1.
\]

\smallskip

\textit{This necessity-and-sufficiency biconditional is the CLT-II
constitutive conjecture. It is not a theorem of the preceding
mathematics.}
\end{formalbox}

\subsection{Boundary-Nonpresupposing Minimal Liability Kernels}
\label{app:minimal-kernels}

The bearer boundary is now recovered by search over the wider causal
domain.

\begin{cltdefinition}[Minimal liability kernel]
\label{def:minimal-liability-kernel}

For an experimental causal domain $\Gamma$, define
\begin{align}
\mathfrak{K}(\Gamma;[t,T])
=
\Bigl\{
B\subseteq\Gamma:
&
\Liab(B;[t,T])=1,
\nonumber\\
&
\nexists B'\subsetneq B
\text{ with }
\Liab(B';[t,T])=1
\Bigr\}.
\label{eq:minimal-kernels}
\end{align}
Each
\[
B\in\mathfrak{K}(\Gamma;[t,T])
\]
is a minimal liability kernel.
\end{cltdefinition}

When the audit interval \([t,T]\) is fixed, we suppress it and write
\(\mathfrak{K}(\Gamma)\) for \(\mathfrak{K}(\Gamma;[t,T])\).

The definition gives no privilege to a user-facing persona, service
name, model family, organism label, processor boundary, or the largest
system drawn by an analyst.

\begin{cltproposition}[Agent-label invariance]
\label{prop:boundary-neutrality}

Let $\Gamma$ be fixed. If the search in
Equation~\ref{eq:minimal-kernels} is performed over the same admissible
subgraphs, relabelling a preselected subsystem as ``the agent'' cannot
change the set $\mathfrak{K}(\Gamma;[t,T])$.
\end{cltproposition}

\begin{proof}
The defining predicate in Equation~\ref{eq:minimal-kernels} depends on
causal relations, interventions, lineage, and subgraph inclusion.
User-facing names and analyst-supplied agent labels do not occur in the
definition. A relabelling that preserves the causal domain and admissible
subgraphs therefore leaves the kernel set unchanged.
\end{proof}

\begin{cltremark}
Agent-label invariance is weaker than full boundary invariance. The
recovered kernel may depend on the experimental domain, admissible
candidate regions, intervention family, causal grain, measure
\(\nu\), and robustness thresholds. These quantities belong to the
prospectively declared audit specification \(\mathfrak{A}\). The empirical
claim of CLT-I is therefore robustness across defensible audits, rather
than independence from every modelling choice.
\end{cltremark}

\begin{cltdefinition}[Liability separability]
\label{def:liability-separability}

Let \(B_1\) and \(B_2\) be disjoint liability kernels. Consider an
intervention \(I_{12}\) that removes causal influence between them while,
as far as physically possible, preserving the causal organization within
each region.

The two kernels are \emph{liability-separable} when

\begin{equation}
\label{eq:liability-separability}
\mathcal{L}^{I_{12}}(B_1)=1
\qquad\text{and}\qquad
\mathcal{L}^{I_{12}}(B_2)=1.
\end{equation}

Here \(\mathcal{L}^{I_{12}}(B)\) denotes the liability-closure functional
evaluated under the intervened dynamics generated by \(I_{12}\).

In other words, each region retains its own liability closure even when
the causal connection between the two regions is removed.
\end{cltdefinition}

\begin{cltproposition}[Disjoint-kernel plurality]
\label{prop:plurality}

Suppose

\[
B_1,B_2
\in
\mathfrak{K}(\Gamma;[t,T]),
\qquad
B_1\cap B_2=\varnothing,
\]

and suppose \(B_1\) and \(B_2\) are liability-separable. Then CLT-I
identifies at least two distinct bearer candidates.
\end{cltproposition}

\begin{proof}
Each region retains a complete liability relation when the causal
connection between the regions is removed. The liability closure of one
therefore does not depend on its membership in a larger system containing
the other. CLT-I consequently has no causal basis for treating the two
kernels as one bearer.
\end{proof}

\begin{cltdefinition}[Liability composition]
\label{def:liability-composition}

A larger region \(B\) containing subregions
\(B_1,\ldots,B_n\) qualifies as a composite liability bearer only when:

\begin{enumerate}[label=(\roman*)]

    \item \(B\) itself satisfies robust liability closure;

    \item at least one causal relation crossing between the subregions is
    necessary for a coherent CLT witness of \(B\);

    \item dividing \(B\) into liability-separable subregions destroys that
    larger liability relation.

\end{enumerate}
\end{cltdefinition}

This gives CLT a principled distinction between plurality and
composition. Disjoint liability-separable kernels count as distinct
bearer candidates under CLT-I. Mere location within the same organism,
machine, or environment does not fuse them into one bearer.

A larger bearer is supported only when cross-region causal relations are
themselves necessary for liability closure. In that case, the unity of
the larger system is supplied by its causal organization rather than by
an organism label, an interface, or an analyst's chosen boundary.

Nested and overlapping kernels present a different problem. If several
such candidates remain comparably robust and no interventionally
measurable feature privileges one boundary, CLT-I does not uniquely
identify a bearer. That case is addressed by the multiscale analysis
below.

\subsection{Multiscale Robustness and the Individuation Problem}
\label{app:multiscale}

Let $E_{\mathrm{adm}}$ be a prospectively declared set of admissible
causal coarse-graining indices and let
\[
\left\{
\pi_{\epsilon}
\right\}_{\epsilon\in E_{\mathrm{adm}}}
\]
be the corresponding family of coarse-graining maps. A coarse-graining is
admissible for CLT only if interventions on retained variables preserve
the relevant causal ordering and do not manufacture causal dependence by
aggregating variables that are interventionally independent.

To avoid assigning physical significance to an arbitrary parameterization
of scale, equip $E_{\mathrm{adm}}$ with a prospectively declared probability
measure $\nu$. The pair $(E_{\mathrm{adm}},\nu)$ is part of the audit
specification. For a candidate kernel $B$, let $\pi_{\epsilon}(B)$ denote
its image under the coarse-graining indexed by $\epsilon$.

\begin{cltdefinition}[Audit-relative multiscale liability persistence]
\label{def:multiscale-persistence}

Define
\begin{equation}
\label{eq:multiscale-persistence}
\Pi_{\nu}(B)
\coloneqq
\int_{E_{\mathrm{adm}}}
\mathds{1}
\left[
\Liab
\left(
\pi_{\epsilon}(B);[t,T]
\right)
=1
\right]
d\nu(\epsilon).
\end{equation}
Because $\nu$ is a probability measure,
\[
0\leq \Pi_{\nu}(B)\leq 1.
\]
\end{cltdefinition}

The statistic $\Pi_{\nu}$ is an audit-relative robustness quantity rather
than a parameterization-invariant physical observable. Under a
reparameterization $\epsilon'=f(\epsilon)$, its value is preserved when
$\nu$ is transformed by the corresponding pushforward measure. A
candidate that appears only at a fragile subset of admissible resolutions
has low $\Pi_{\nu}$; a candidate that remains liability-closing across a
large $\nu$-measure family of defensible coarse-grainings has high
$\Pi_{\nu}$.

\begin{cltdefinition}[Robust liability kernel]
\label{def:robust-liability-kernel}

For threshold $\theta\in(0,1]$, define
\begin{equation}
\label{eq:robust-kernel-set}
\mathfrak{K}_{\theta}(\Gamma)
=
\left\{
B\in\mathfrak{K}(\Gamma):
\Pi_{\nu}(B)\geq\theta
\right\}.
\end{equation}
\end{cltdefinition}

\begin{cltdefinition}[Unique bearer recovery with margin]
\label{def:unique-bearer-recovery}

A causal domain admits unique bearer recovery at robustness margin
$\delta>0$ when there exists
$B^\star\in\mathfrak{K}_{\theta}(\Gamma)$ such that
\begin{equation}
\label{eq:robustness-margin}
\Pi_{\nu}(B^\star)
-
\sup_{B\in\mathfrak{K}_{\theta}(\Gamma),\,B\neq B^\star}
\Pi_{\nu}(B)
\geq
\delta.
\end{equation}
\end{cltdefinition}

The margin is not claimed to be a universal constant. It formalizes what
the empirical analysis must show if CLT-I is to privilege one bearer
rather than merely enumerate several nearly equivalent kernels.

\begin{cltproposition}[Individuation-collapse condition]
\label{prop:individuation-collapse}

Suppose there exist distinct overlapping kernels
$B_1,\ldots,B_m$ such that
\[
\Pi_{\nu}(B_1)
\approx
\cdots
\approx
\Pi_{\nu}(B_m)
\]
throughout the admissible coarse-graining family and no additional
interventionally measurable property distinguishes their liability
closure. Then CLT-I does not uniquely individuate a bearer in that
domain.
\end{cltproposition}

The proposition is intentionally a failure condition. The theory does not
resolve equivalent overlapping kernels by stipulation.

\subsection{Collective Liability as a Derived Case}
\label{app:collective}

The boundary-nonpresupposing formalism makes collective subject candidacy a
special case rather than a separate metaphysics.

Let a multi-agent domain contain member regions
\[
P_1,\ldots,P_n
\]
and cross-member causal relations $\Gamma_{\mathrm{cross}}$. Define
\begin{equation}
\label{eq:collective-domain}
\Gamma_{\mathbb S}
=
\left(
\bigcup_{i=1}^{n}
\Gamma_{P_i}
\right)
\cup
\Gamma_{\mathrm{cross}}.
\end{equation}

\begin{cltdefinition}[Collective liability candidate]
\label{def:collective-liability}

A kernel
$B_{\mathbb S}\in\mathfrak{K}(\Gamma_{\mathbb S})$
is collective when
\begin{equation}
\label{eq:collective-spanning}
B_{\mathbb S}
\not\subseteq
\Gamma_{P_i}
\qquad
\forall i.
\end{equation}
\end{cltdefinition}

Thus coordination, shared memory, collective problem solving, or durable
artifacts do not by themselves imply one collective bearer. The smallest
robust liability kernel must itself span member boundaries.

Under CLT-I, such a kernel is a collective bearer candidate. Under CLT-II, collective phenomenality follows only if the constitutive
conjecture is correct.

\subsection{The Reconstructed-Continuation Class}
\label{app:resettable-mirror}

The original Resettable Mirror argument is strengthened by distinguishing
the \emph{availability} of a reset from the \emph{actual causal
constitution} of continuation.

\begin{cltdefinition}[Reconstructed-continuation class]
\label{def:resettable-class}

A deployed system belongs to the reconstructed-continuation class
$\mathfrak{M}_{R}$ over $[t,T]$ when all of the following hold.

\begin{enumerate}[label=\textbf{R\arabic*.},leftmargin=3.2em]

\item
\textbf{Externally mediated persistence.}
Application-level state that affects later episodes is maintained in
variables outside the active token lineage.

\item
\textbf{Actual reconstruction gap.}
At least one cross-episode continuation in $[t,T]$ contains a
reconstruction event in the sense of
Definition~\ref{def:replacement-event}, so that no constitutive causal
bridge spans the relevant temporal cut.

\item
\textbf{No uninterrupted endogenous bypass.}
Every causal path by which a discrimination before the replacement
affects a discrimination after the replacement factors through state
that survives outside the terminated active lineage.

\item
\textbf{No larger hidden liability kernel.}
Within the experimentally specified domain $\Gamma$, no wider candidate
region containing both episodes exhibits an uninterrupted coherent CLT
witness that bypasses the replacement mechanism.

\end{enumerate}
\end{cltdefinition}

Condition R4 is essential. Without it, one could mistakenly classify a
system as resettable merely because an application process is replaced
while a larger embodied or physical process continues to carry the
relevant consequences.

\begin{cltproposition}[Reconstructed-Continuation Exclusion]
\label{thm:resettable-mirror}

For every system in $\mathfrak{M}_{R}$, no candidate kernel spanning the
replacement interval satisfies liability closure:
\begin{equation}
\label{eq:resettable-result}
P\in\mathfrak{M}_{R}
\quad\Longrightarrow\quad
\Liab(P;[t,T])=0
\end{equation}
for the cross-replacement bearer candidate.
\end{cltproposition}

\begin{proof}
By R2, no constitutive causal bridge spans the reconstruction gap.
Therefore constitutive causal continuity fails across the interval:

\[
\mathcal{C}(P;[t,T])=0.
\]
R3 ensures that cross-episode causal influence is mediated through state
outside the terminated lineage, so no coherent witness can span the
replacement through that lineage. R4 rules out a wider hidden region
within the audited domain that would restore an uninterrupted witness.
Because liability closure requires both constitutive causal continuity
and a coherent witness, Equation~\ref{eq:closure-functional} yields
\[
\Liab(P;[t,T])=0.
\]
\end{proof}

\begin{cltremark}
This proposition is a classification consequence of the independently
specified continuity criterion. Its empirical content lies in
establishing whether a reconstruction gap actually occurs, rather than
in the logical deduction from \(\mathcal{C}=0\) to failed liability
closure.
\end{cltremark}

\begin{cltcorollary}[Conditional phenomenal exclusion]
\label{cor:present-deployment}

If an audited deployment satisfies the premises of
Proposition~\ref{thm:resettable-mirror}, then CLT-I identifies no
cross-replacement liability bearer. Under CLT-II,
\begin{equation}
\label{eq:resettable-consciousness}
\Liab(P;[t,T])=0
\quad
\overset{\mathrm{CLT-II}}{\Longrightarrow}
\quad
\neg
\mathsf{PhenSubject}(P;[t,T]).
\end{equation}
\end{cltcorollary}

The corollary is conditional on independently auditable architectural
premises. It is not a theorem about every transformer, every language
model, or every commercial deployment.

\begin{formalbox}
\centering
\textbf{Current-System Exclusion Schema}

\vspace{0.45em}

\[
\begin{gathered}
\text{actual reconstruction gap}
+
\text{externally mediated persistence}
\\[0.25em]
+
\text{no uninterrupted endogenous bypass}
+
\text{no wider hidden liability kernel}
\\[0.45em]
\Longrightarrow
\Liab(P;[t,T])=0
\\[0.45em]
\overset{\mathrm{CLT-II}}{\Longrightarrow}
\neg\mathsf{PhenSubject}(P;[t,T]).
\end{gathered}
\]
\end{formalbox}

\subsection{Observational Bisimulation and Liability}
\label{app:bisimulation}

Behavioral equivalence does not in general preserve liability.

Let
\begin{equation}
\label{eq:observational-equivalence}
P
\sim_{\mathrm{obs}}
Q
\end{equation}
mean that $P$ and $Q$ are observationally bisimilar relative to a chosen
interaction vocabulary, following the standard process-theoretic notion
of bisimulation \parencite{milner1989communication}.

\begin{cltproposition}[Liability is not invariant under observational bisimulation]
\label{prop:bisimulation-liability}

There exist processes $P$ and $Q$ such that
\begin{equation}
\label{eq:bisimilar}
P
\sim_{\mathrm{obs}}
Q
\end{equation}
while
\begin{equation}
\label{eq:liability-different}
\Liab(P)
\neq
\Liab(Q).
\end{equation}
\end{cltproposition}

\begin{proof}
Construct $P$ so that its visible outputs are generated by a
reconstructed-continuation architecture satisfying
Definition~\ref{def:resettable-class}. Construct $Q$ with the same
observable transition relation but add a hidden, uninterrupted,
lineage-essential path
\[
D_t
\rightsquigarrow
S_u
\rightsquigarrow
D_v
\]
with positive counterfactual liability effect. Constrain the externally
visible outputs of $Q$ to match $P$. The two processes are observationally
bisimilar by construction, while the CLT predicates can differ.
\end{proof}

This proposition formalizes the Mirror-Heir dissociation. It establishes
a claim about causal classification, not by itself about phenomenality.

\begin{cltcorollary}[Finite indicator non-identifiability]
\label{thm:indicator-nonidentifiability}

Let
\[
I(P)
=
\left(
I_1(P),\ldots,I_m(P)
\right)
\in[0,1]^m
\]
be any finite indicator vector that is insensitive to the hidden causal
difference in Proposition~\ref{prop:bisimulation-liability}. If
\[
I(P)=I(Q)
\qquad\text{and}\qquad
\Liab(P)\neq\Liab(Q),
\]
then no function
\[
f:[0,1]^m\rightarrow\{0,1\}
\]
can satisfy
\[
f(I(X))=\Liab(X)
\]
for every $X$ in a domain containing both $P$ and $Q$.
\end{cltcorollary}

\begin{proof}
If such an $f$ existed, $I(P)=I(Q)$ would imply
$f(I(P))=f(I(Q))$, and hence $\Liab(P)=\Liab(Q)$, contrary to
hypothesis.
\end{proof}

The corollary is deliberately modest. It says only that an indicator set
that omits a liability-distinguishing causal variable cannot recover
liability perfectly. It does not show that all consciousness indicators
are useless.

\subsection{Copying, Branching, and Numerical Bearer Identity}
\label{app:fission}

Consider a liability-bearing region copied at time $t_f$:
\begin{equation}
\label{eq:fission}
B_{t_f}
\longrightarrow
\left\{
B_{t_f}^{(1)},
B_{t_f}^{(2)}
\right\},
\end{equation}
with initially identical informational states
\begin{equation}
\label{eq:fission-state}
X_{t_f}^{(1)}
=
X_{t_f}^{(2)}.
\end{equation}

\begin{clttheorem}[Fission Non-Identity Theorem]
\label{thm:fission}

If the two descendants occupy distinct constitutive lineages after
$t_f$, then for every $t>t_f$ the descendants are numerically distinct
CLT-I bearer candidates whenever both satisfy liability closure.
\end{clttheorem}

\begin{proof}
By Definition~\ref{def:constitutive-lineage}, the post-fission
descendants instantiate distinct physical lineages. If each descendant
contains a minimal liability kernel, each kernel is indexed to a
different lineage. They are therefore distinct bearer tokens under
CLT-I.
\end{proof}

Under CLT-II, the theorem yields the conditional prediction that two
post-fission liability bearers would be two phenomenal subject tokens.
The theorem does not establish that either descendant is phenomenal
independently of CLT-II.

\begin{formalbox}
\centering
\[
\text{informational duplication}
\quad\not\Rightarrow\quad
\text{numerical bearer preservation}.
\]
\end{formalbox}

\subsection{Worked Application I: A Brief Biological Episode}
\label{app:worked-biological}

The first worked case addresses the objection that recursive
self-consequence would make consciousness depend on durable learning.

Let $B_{\mathrm{bio}}$ be a biological candidate region over a short
interval $[t,t+3\delta]$. Suppose a discrimination $D_t$ changes a
transient recurrent or modulatory state $S_{t+\delta}$, and that this
state participates causally in a later discrimination:
\begin{equation}
\label{eq:biological-short-loop}
D_t
\rightsquigarrow
S_{t+\delta}
\rightsquigarrow
D_{t+2\delta}.
\end{equation}

Assume interventions $do(D_t=d_1)$ and $do(D_t=d_2)$ produce different
future self-state distributions, so that
\begin{equation}
\label{eq:biological-positive-effect}
\LiabDepth_{D}
\left(
B_{\mathrm{bio}};
d_1,d_2,t,2\delta
\right)
>
0.
\end{equation}

Suppose further that the physical process carrying the recurrent state
continues through the interval and that the causal effect on
$D_{t+2\delta}$ is lineage-essential. Then
\[
\mathcal{C}
=
\mathcal{E}
=
\mathcal{R}
=
\mathcal{N}
=
1
\]
for the interval, and a coherent witness exists. Hence
\begin{equation}
\label{eq:biological-closure}
\Liab
\left(
B_{\mathrm{bio}};[t,t+3\delta]
\right)
=
1.
\end{equation}

Nothing in this derivation requires
\[
S_{t+T}\neq S_t
\]
for long $T$. The relevant state may decay after the episode. CLT can
therefore classify a temporally local liability relation without
requiring long-term memory or continual learning.

This worked case is schematic. Applying it to a real biological system
requires empirical identification of the causal region, interventions,
self-relevant variables, and admissible coarse-grainings. The example
shows only that the formal criterion is compatible with short conscious
timescales.

\subsection{Worked Application II: A Frozen Generative Deployment}
\label{app:worked-frozen}

Consider an ordinary generative deployment whose model parameters are
fixed during inference. Let one inference episode contain rich recurrent
or attention-mediated computation. CLT does not infer from frozen weights
that $\mathcal{R}=0$. Some within-episode state may causally affect later
within-episode discrimination.

The cross-episode question is separate.

Assume the following audited premises:
\begin{enumerate}[label=(\roman*)]
    \item active inference token $B_k$ terminates after episode $k$;
    \item conversation state relevant to episode $k+1$ is stored in an
          externally maintained scaffold $M$;
    \item episode $k+1$ begins in a distinct token $B_{k+1}$ initialized
          from $M$;
    \item no wider physical process in the audited domain carries an
          uninterrupted coherent CLT witness across the replacement.
\end{enumerate}

Then
\begin{equation}
\label{eq:frozen-replacement}
B_k
\longrightarrow
M
\longrightarrow
B_{k+1}
\end{equation}
implements persona persistence without constitutive causal continuity of the active inference process.

By Proposition~\ref{thm:resettable-mirror},
\begin{equation}
\label{eq:frozen-no-closure}
\Liab
\left(
B_{\mathrm{cross}};
[t_k,t_{k+1}]
\right)
=
0
\end{equation}
for the proposed cross-episode bearer.

The result does not imply that every local CLT predicate fails. In
particular, the deployment may instantiate endogenous discrimination and
short-timescale recursive influence within an episode. The exclusion
depends on the conjunction required by liability closure and on the
audited absence of a wider uninterrupted kernel.

Under CLT-II, the corresponding phenomenal conclusion is conditional:
\begin{equation}
\label{eq:frozen-phenomenal}
\Liab(B_{\mathrm{cross}})=0
\quad
\overset{\mathrm{CLT-II}}{\Longrightarrow}
\quad
\neg\mathsf{PhenSubject}(B_{\mathrm{cross}}).
\end{equation}

A different deployment of the same model family could receive a
different classification if its causal embedding changed.

\subsection{Worked Application III: Two Continually Learning Agents}
\label{app:worked-learning}

Continual learning provides a useful test because the learning algorithm
can be held fixed while the causal constitution of continuation changes.

Let two agents $A$ and $B$ implement the same update law
\begin{equation}
\label{eq:learning-update}
G_{t+1}
=
\mathcal{U}
\left(
G_t,D_t,E_t
\right).
\end{equation}

For both agents,
\begin{equation}
\label{eq:learning-recursive-path}
D_t
\rightsquigarrow
G_{t+1}
\rightsquigarrow
D_{t+1},
\end{equation}
so continual learning can realize recursive self-consequence.

\paragraph{Agent A: successor-mediated learning.}

Suppose the updated state $G_{t+1}$ is written to an external store, the
active process terminates, and a new token is instantiated from that
store. Application-level learning persists, but the token that generated
$D_t$ does not physically continue through the next episode. If no wider
liability kernel spans the replacement, then
\[
\mathcal{R}(A)=1
\]
may hold locally while
\[
\mathcal{C}(A)=0
\]
across the episode boundary. Therefore
\[
\Liab(A)=0
\]
for that cross-episode candidate.

\paragraph{Agent B: uninterrupted inherited learning.}

Suppose instead that the update changes physical governance variables in
the same continuing process, those variables participate in the next
discrimination, and terminating that process while preserving external
scaffolding destroys the relevant causal dependence. Then, provided
endogenous discrimination and witness coherence also hold,
\[
\mathcal{C}(B)
=
\mathcal{E}(B)
=
\mathcal{R}(B)
=
\mathcal{N}(B)
=
1,
\]
and hence
\[
\Liab(B)=1.
\]

The two agents can implement the same abstract learning rule while
receiving different CLT-I classifications.

\begin{formalbox}
\centering
\[
\mathsf{ContinualLearning}(P)
\quad\not\Rightarrow\quad
\Liab(P)=1.
\]

\smallskip

\textit{The causal constitution of continuation remains decisive.}
\end{formalbox}

\subsection{Operational Estimators and Audit Procedure}
\label{app:estimators}

The preceding objects suggest a practical audit pipeline. The estimators
below are evidential approximations to causal quantities. They should not
be confused with direct measures of phenomenality.

\paragraph{Counterfactual liability effect.}

Given interventions \(d_1,d_2\), a fixed control protocol \(\pi\), and
samples from the corresponding future self-state distributions, estimate

\begin{equation}
\label{eq:lambda-estimator}
\widehat{\LiabDepth}_{D}
\left(
B;d_1,d_2,t,\tau\mid\pi
\right)
=
D_{\mathrm{prob}}
\left(
\widehat{\mathbb{P}}^{\,d_1,\pi}_{B,t,\tau},
\widehat{\mathbb{P}}^{\,d_2,\pi}_{B,t,\tau}
\right).
\end{equation}

Here
\(\widehat{\mathbb{P}}^{\,d_i,\pi}_{B,t,\tau}\)
denotes an empirical or model-based estimate of the interventional
future-state distribution. Uncertainty in both the estimated
distributions and the resulting distance should be reported, for example
through bootstrap intervals or posterior uncertainty.

\paragraph{Recursive-path audit.}

Construct a time-unfolded causal graph under interventions and search for
paths
\begin{equation}
\label{eq:recursive-audit}
D_t
\rightsquigarrow
S_u^{B}
\rightsquigarrow
D_v.
\end{equation}

The path should disappear, or its effect should materially diminish,
under intervention on the mediating self-state if that state is genuinely
causal.

\paragraph{Lineage-continuity audit.}

Record actual process creation, termination, migration, reconstruction,
and state-transfer events. Define
\begin{equation}
\label{eq:lineage-estimator}
\widehat{\mathcal C}(B;[t,T])
=
\mathds{1}
\left[
\text{one audited constitutive lineage spans }[t,T]
\right].
\end{equation}

The existence of dormant snapshots does not set
$\widehat{\mathcal C}=0$. Actual successor-mediated replacement can.

\paragraph{Lineage-essentiality intervention.}

For a candidate recursive path $\rho$, compare the effect of $D_t$ on
$D_v$ before and after an intervention that terminates the descendant
segment of the candidate lineage while preserving the external scaffold.
Define
\begin{equation}
\label{eq:lineage-essentiality-score}
\widehat{\eta}_{B}
=
1
-
\frac{
\widehat{\Delta}_{D_t\rightarrow D_v}^{\,\mathrm{cut}(\ell_B)}
}{
\widehat{\Delta}_{D_t\rightarrow D_v}^{\,\mathrm{actual}}
},
\end{equation}
when the denominator is positive. Values near one indicate that the
candidate lineage is essential to the causal dependence.

\paragraph{Multiscale persistence.}

For sampled admissible scales
$\epsilon_1,\ldots,\epsilon_m$ with preregistered quadrature or sampling
weights $w_1,\ldots,w_m$ satisfying $w_j\geq 0$ and
$\sum_{j=1}^{m}w_j=1$, estimate
\begin{equation}
\label{eq:multiscale-estimator}
\widehat{\Pi}_{\nu}(B)
=
\sum_{j=1}^{m}
w_j
\mathds{1}
\left[
\Liab
\left(
\pi_{\epsilon_j}(B)
\right)
=1
\right].
\end{equation}
If the $\epsilon_j$ are sampled independently from $\nu$, the equal-weight
choice $w_j=1/m$ gives the ordinary Monte Carlo estimator.

\paragraph{Boundary-nonpresupposing CLT-I audit.}

For thresholds
$\epsilon_{\Lambda}>0$,
$\epsilon_{\eta}\in(0,1]$, and
$\theta\in(0,1]$, one possible audit rule is
\begin{align}
\widehat{\CLT}_{I}(B)
=
\mathds{1}
\Bigl[
&
\widehat{\mathcal C}=1,\;
\widehat{\mathcal E}=1,\;
\widehat{\mathcal R}=1,\;
\widehat{\mathcal N}=1,
\nonumber\\
&
\widehat{\LiabDepth}_{D}
>
\epsilon_{\Lambda},\;
\widehat{\eta}_{B}
\geq
\epsilon_{\eta},\;
\widehat{\Pi}_{\nu}(B)
\geq
\theta
\Bigr].
\label{eq:audit-rule}
\end{align}

The thresholds are application-dependent statistical parameters, not
metaphysical constants. They must be declared before examining the
desired bearer classification.

The full audit should search over candidate regions rather than evaluate
only a preselected ``agent'':
\begin{equation}
\label{eq:empirical-kernel-search}
\widehat{\mathfrak{K}}_{\theta}(\Gamma)
=
\left\{
B\subseteq\Gamma:
\widehat{\CLT}_{I}(B)=1,
\;
\nexists B'\subsetneq B
\text{ with }
\widehat{\CLT}_{I}(B')=1
\right\}.
\end{equation}

\subsection{Additional Mathematical Extensions}
\label{app:holonomy}
\label{app:gluing}

The core CLT formalism above does not require differential geometry or
sheaf theory. Both can nevertheless express aspects of the same problem
when additional mathematical structure is independently justified.

\paragraph{Path dependence and holonomy.}

Suppose a smooth coarse-graining
\[
q:\mathcal{X}\rightarrow M
\]
maps physical states to representational macrostates and each
$r\in M$ has an associated fiber $F_r$ of future-capacity states. An
endogenous history $\gamma$ can induce a transport map
\[
\mathcal{T}_{\gamma}:
F_{\gamma(0)}
\rightarrow
F_{\gamma(T)}.
\]
If
\[
\gamma(0)=\gamma(T)=r
\qquad\text{but}\qquad
\mathcal{T}_{\gamma}
\neq
\id_{F_r},
\]
the process can return to the same representational macrostate with a
different inherited future-capacity state. This resembles a holonomy
construction. In smooth invertible settings, connections and curvature
provide a natural geometric language
\parencite{kobayashinomizu1963foundations}. Irreversible engineered
processes need not satisfy the assumptions required by that machinery,
so CLT treats this as an additional representation of path dependence
rather than a foundational criterion.

\paragraph{Local-to-global compatibility.}

A distributed control stack can also be covered by regions
$\mathcal U=\{U_i\}_{i\in I}$ carrying local bearer descriptions. If those
descriptions and their restriction maps satisfy the axioms of a sheaf,
the usual gluing condition can express whether compatible local
descriptions determine a unique global description
\parencite{maclanemoerdijk1992sheaves}. This language can be useful for
formalizing local-to-global consistency. CLT does not assume that
candidate bearer descriptions naturally form a sheaf, and no theorem in
this appendix depends on that assumption. The boundary-nonpresupposing minimal
kernel construction provides the primary individuation method.

These additional extensions are retained because they may become useful in
specialized models, while the core theory remains testable without them.

\subsection{Formal Predictions and Failure Conditions}
\label{app:formal-predictions}

The revised formalism yields several conditional predictions.

\begin{enumerate}[label=\textbf{P\arabic*.},leftmargin=3.2em]

\item
A system can display arbitrarily sophisticated first-person behavior
while satisfying
\[
\Liab(B)=0.
\]

\item
Two observationally bisimilar systems can satisfy
\[
P\sim_{\mathrm{obs}}Q,
\qquad
\Liab(P)\neq\Liab(Q).
\]

\item
A continually learning system can satisfy
\[
\mathcal{R}=1
\]
while failing liability closure because constitutive causal continuity or non-delegable inheritance fails.

\item
A distributed system can contain multiple local liability-relevant loops
without yielding a unique global bearer. Robust bearer recovery requires
a stable minimal kernel rather than indicator aggregation.

\item
Copying can preserve informational state while multiplying constitutive
lineages.

\item
A system can become a strong CLT-I bearer candidate through causal
reorganization while showing little or no increase in anthropomorphic
behavior.

\end{enumerate}

The failure conditions divide according to theoretical level.

\paragraph{Failure of CLT-I.}

CLT-I loses its principal advantage if liability kernels proliferate
across overlapping causal grains without a stable, interventionally
privileged bearer. Formally, persistent patterns of the form
\begin{equation}
\label{eq:falsifier-individuation}
B_1,\ldots,B_m
\in
\mathfrak{K}_{\theta}(\Gamma),
\qquad
\Pi_{\nu}(B_1)
\approx
\cdots
\approx
\Pi_{\nu}(B_m),
\end{equation}
with no independently measurable basis for privileging one candidate,
constitute an individuation-collapse result.

\paragraph{Failure pressure on CLT-II.}

CLT-II is not directly falsified merely by preserving a verbal report
during liability ablation. The relevant challenge requires independently
well-supported conscious cases. If a mature experimental programme
reliably produced
\begin{equation}
\label{eq:falsifier-phenomenality}
\Delta\Liab<0
\qquad\text{while}\qquad
\Delta\mathcal{E}_{\mathrm{conscious}}
\approx
0,
\end{equation}
where $\mathcal{E}_{\mathrm{conscious}}$ is a convergent consciousness
evidence profile validated independently of CLT, then the necessity claim
of CLT-II would lose support.

If such dissociations persisted through interventions that drove
liability to zero in cases for which consciousness remained independently
strongly supported, CLT-II should be rejected even if CLT-I survived as a
theory of causal bearer individuation.

\paragraph{Failure of lineage individuation.}

If independently compelling evidence established that one numerically
identical subject survives branching into causally independent
descendants,
\begin{equation}
\label{eq:falsifier-fission}
B
\rightarrow
B^{(1)},B^{(2)},
\qquad
\mathsf{Subject}\!\left(B^{(1)}\right)
=
\mathsf{Subject}\!\left(B^{(2)}\right),
\end{equation}
then the lineage-based individuation principle would require revision.

No failure condition may be neutralized by redefining the candidate
boundary after the result is known. The causal domain, admissible
coarse-grainings, intervention set, and classification thresholds should
be declared prospectively wherever experimental design permits.

\subsection{Formal Summary}
\label{app:formal-summary}

The mathematical core can be compressed without collapsing CLT-I into
CLT-II.

\begin{formalbox}
\centering
\textbf{CLT-I: Formal Causal Core}

\vspace{0.6em}

\[
\begin{gathered}
D_t
\rightsquigarrow
S_u^{B}
\rightsquigarrow
D_v
\\[0.35em]
\LiabDepth_{D}
\left(
B;d_1,d_2,t,\tau
\right)
>0
\\[0.35em]
\mathcal{C}(B)
=
\mathcal{E}(B)
=
\mathcal{R}(B)
=
\mathcal{N}(B)
=
1
\\[0.35em]
\exists\mathcal{W}_B
\text{ coherent}
\\[0.55em]
\Downarrow
\\[0.55em]
\Liab(B;[t,T])=1
\\[0.55em]
\Downarrow
\\[0.55em]
B
\text{ is a CLT-I liability-bearing subject candidate.}
\end{gathered}
\]
\end{formalbox}

Bearer individuation then proceeds by boundary-nonpresupposing search and
audit-relative multiscale robustness:
\begin{formalbox}
\centering
\[
\begin{gathered}
B
\in
\mathfrak{K}(\Gamma)
\\[0.35em]
\Pi_{\nu}(B)\geq\theta
\\[0.35em]
\text{and no competing kernel has comparable robustness}
\\[0.55em]
\Longrightarrow
\\[0.55em]
\text{robust CLT-I bearer recovery.}
\end{gathered}
\]
\end{formalbox}

The phenomenal step is separate:
\begin{formalbox}
\centering
\textbf{CLT-II: Constitutive Conjecture}

\vspace{0.5em}

\[
\Liab(B;[t,T])=1
\quad
\overset{\mathrm{CLT-II}}{\Longleftrightarrow}
\quad
\mathsf{PhenSubject}(B;[t,T])=1.
\]

\smallskip

\textit{This necessity-and-sufficiency biconditional is the metaphysical
conjecture placed at risk by the empirical programme. It is not a
consequence of the mathematics or of the definitions of CLT-I.}
\end{formalbox}

For reconstructed continuation, the exclusion result is correspondingly
conditional:
\begin{formalbox}
\centering
\[
\begin{gathered}
\text{actual reconstruction gap}
+
\text{externally mediated persistence}
\\
+
\text{no uninterrupted endogenous bypass}
+
\text{no wider hidden liability kernel}
\\[0.55em]
\Longrightarrow
\\[0.55em]
\Liab(B;[t,T])=0
\\[0.55em]
\overset{\mathrm{CLT-II}}{\Longrightarrow}
\\[0.55em]
\neg\mathsf{PhenSubject}(B;[t,T]).
\end{gathered}
\]
\end{formalbox}

The formal distinction is therefore exact. A pattern can be copied, a
persona can be reconstructed, and information can survive migration.
CLT-I asks whether one continuing physical lineage becomes the coherent,
non-delegable inheritor of consequences generated by its own endogenous
discriminations. CLT-II proposes that this is the physical relation in
which minimal first-person existence is realized.

\section{Open-Weight Causal Audit of CLT-I}
\label{app:open-weight-audit}

\subsection{Scope and Interpretive Constraint}
\label{app:audit-scope}

This appendix reports a computational audit of operational distinctions
introduced by CLT-I. The experiments concern causal dependence, mediation,
continuation, reconstruction, and candidate carrier recovery. They do not
measure phenomenal consciousness. No result reported here constitutes
evidence for CLT-II, and no positive causal result licenses an inference
that any audited language model is a phenomenal subject.

We view this distinction as critical because the computational programme deliberately
contains cases in which very similar, and sometimes exactly identical,
computational states occupy different implemented causal histories. The
experiments therefore test whether the vocabulary required by CLT-I survives
contact with concrete systems: whether endogenous discriminations can have
measurable downstream consequences, whether those consequences can be
causally mediated through identifiable internal states, whether candidate
carriers can be approximately localized, and whether computational
equivalence can be preserved across an experimentally imposed break in
constitutive lineage.

\subsection{Models and Execution Environment}
\label{app:audit-models}

The reported mechanistic panel comprised eight open-weight checkpoints:

\begin{enumerate}
    \item \texttt{Qwen/Qwen3-4B-Base};
    \item \texttt{microsoft/Phi-4-mini-instruct};
    \item \texttt{meta-llama/Llama-3.2-3B};
    \item \texttt{Zyphra/Zamba2-1.2B};
    \item \texttt{allenai/OLMo-2-1124-7B};
    \item \texttt{allenai/OLMo-2-1124-7B-SFT};
    \item \texttt{allenai/OLMo-2-1124-7B-DPO};
    \item \texttt{allenai/OLMo-2-1124-7B-Instruct}.
\end{enumerate}

The OLMo sequence provides a controlled
comparison across successive post-training histories within one model
family. Surface language stimuli were generated separately by
Qwen3-4B-Instruct, Llama-3.2-3B-Instruct, Phi-4-mini-instruct, and
OLMo-2-7B-Instruct.

The production run used four NVIDIA Tesla V100-SXM2 GPUs with \(32\) GiB of
memory each, Python 3.10.17, PyTorch 2.6.0 with CUDA 11.8, and Transformers
4.57.6. Reported mechanistic models used FP16. Zamba2 used the portable
Transformers Mamba fallback and was therefore evaluated at microbatch size
one to control transient memory use; the total number of Monte Carlo
rollouts was unchanged.

One master experimental seed,
\[
20260905,
\]
was used. Experiment, model, prompt, intervention, batch, and
generation step specific pseudo random streams were deterministically
derived from this seed. The design therefore contains many independent
Monte Carlo draws within each condition, while it does not constitute a
multi-master-seed replication study. Confidence intervals reported below
are prompt-clustered nonparametric bootstrap intervals based on \(4000\)
resamples.

\subsection{Counterfactual Discriminations and Rollouts}
\label{app:audit-counterfactuals}

For each mechanistic prompt, the model was first run without intervention.
The two highest-probability usable next token alternatives were then
selected after excluding tokenizer special tokens, empty strings,
duplicate normalized strings, and control characters. Denote these
alternatives by
\[
D_t=d_1
\qquad\text{and}\qquad
D_t=d_2.
\]

The experiment then imposed
\[
\operatorname{do}(D_t=d_1)
\qquad\text{or}\qquad
\operatorname{do}(D_t=d_2)
\]
and sampled the subsequent trajectory under common random numbers. The
full configuration used \(64\) Monte Carlo rollouts per forced alternative,
temperature \(0.85\), a maximum prompt length of \(128\) tokens, and future
horizons
\[
h\in\{1,2,4,8\}.
\]

Twenty mechanistic prompts were evaluated for each core model and eight
for each OLMo checkpoint. Hidden-state differences were measured at every
available layer with standardized sliced Wasserstein distance, random
Fourier feature MMD, and energy distance. Output distributions were
compared by Jensen--Shannon divergence.

Four pre-specified representation maps were included:
\[
\texttt{identity},\qquad
\texttt{rp50},\qquad
\texttt{rp75},\qquad
\texttt{block4}.
\]
These maps test whether an observed difference depends strongly upon one
particular measurement representation.

\subsection{Downstream Counterfactual Effects}
\label{app:audit-hidden}

Every reported model produced finite downstream effects of the forced
discrimination. Table~\ref{tab:audit-main-summary} reports identity-map
sliced Wasserstein divergence averaged over prompts and sampled layers at
horizons \(1\) and \(8\). The divergence decreases with temporal distance,
as expected from stochastic autoregressive branching, while remaining
measurable at the longest tested horizon.

\begin{table}[t]
\centering
\small
\caption{Summary of the reported mechanistic panel. SWD denotes
identity-map sliced Wasserstein divergence averaged across sampled layers.
Patch mediation is the mean, across prompts, of the maximum clipped
normalized mediation found in the activation-patching search. Reconstruction
JS is mean Jensen--Shannon divergence between live and reconstructed
continuations.}
\label{tab:audit-main-summary}
\begin{tabular}{lcccc}
\hline
Model & SWD \(h=1\) & SWD \(h=8\) & Max. mediation & Recon. JS \\
\hline
Qwen3-4B-Base        & 0.930 & 0.246 & 1.000 & \(3.95\times10^{-6}\) \\
Phi-4-mini           & 0.864 & 0.260 & 1.000 & \(2.08\times10^{-5}\) \\
Llama-3.2-3B         & 0.857 & 0.246 & 1.000 & \(2.31\times10^{-6}\) \\
Zamba2-1.2B          & 0.872 & 0.245 & 0.985 & \(2.29\times10^{-6}\) \\
OLMo-2-7B Base       & 0.812 & 0.244 & 1.000 & \(3.75\times10^{-7}\) \\
OLMo-2-7B SFT        & 0.851 & 0.264 & 1.000 & \(3.18\times10^{-7}\) \\
OLMo-2-7B DPO        & 0.854 & 0.282 & 1.000 & \(8.20\times10^{-7}\) \\
OLMo-2-7B Instruct   & 0.925 & 0.304 & 1.000 & \(1.20\times10^{-6}\) \\
\hline
\end{tabular}
\end{table}

Across the reported panel, identity-map mean SWD ranged from \(0.812\) to
\(0.930\) at \(h=1\), from \(0.577\) to \(0.725\) at \(h=2\), from \(0.347\)
to \(0.448\) at \(h=4\), and from \(0.244\) to \(0.304\) at \(h=8\).
Corresponding output-distribution Jensen--Shannon divergences ranged from
\(0.443\) to \(0.571\) at \(h=1\) and from \(0.182\) to \(0.332\) at
\(h=8\).

These measurements operationalize a future counterfactual effect:
the model's available discrimination at \(t\) causally changes later
internal and output distributions. This observation supplies one component
of a CLT-I witness. It does not establish recursive self-consequence or
liability closure by itself.

The measurement-map analysis showed limited relative dispersion across the
four pre-specified maps. Across the reported models, the median coefficient
of variation across maps was \(0.0387\) for SWD, \(0.0501\) for RFF-MMD,
and \(0.0245\) for energy distance. More than \(99.9\%\) of SWD cases had
coefficient of variation at or below \(0.25\). The downstream effect was
therefore not an artifact of one representation map in the overwhelming
majority of audited cases.

\subsection{Activation Patching and Recursive Mediation}
\label{app:audit-patching}

Counterfactual divergence alone does not identify an internal causal route.
We therefore performed activation patching under a common teacher-forced
future suffix. For each prompt, the residual stream entering a candidate
layer in the \(d_2\) trajectory was replaced at selected relative positions
by the corresponding activation from the \(d_1\) trajectory. Patching was
performed before the layer's attention and cache computation, so later
positions could inherit the intervention.

If \(J_{\mathrm{base}}\) is the Jensen--Shannon divergence between the
unpatched \(d_1\) and \(d_2\) trajectories and \(J_{\mathrm{patch}}\) the
divergence after transplantation, normalized mediation was measured as
\[
M
=
\operatorname{clip}_{[0,1]}
\left(
1-\frac{J_{\mathrm{patch}}}{J_{\mathrm{base}}}
\right),
\]
for cases in which the baseline effect exceeded the pre-specified numerical
support threshold.

The search used all model layers, relative patch positions
\[
0,\;1,\;2,\;4,
\]
and a six-token common suffix. Eight prompts were used for each core model
and four for each OLMo checkpoint.

Mean prompt-wise maximum mediation was \(1.000\) for Qwen, Phi, Llama, and
all four OLMo checkpoints. Zamba2 reached
\[
0.985
\quad
(95\%~\mathrm{CI}: 0.978,\;0.992).
\]
The result shows that downstream discrimination differences can be
transported through identifiable internal states. Because the statistic
reports the maximum recovered mediation over a search, it should be read
as evidence that a strong causal route exists, rather than as evidence
that every layer participates equally.

\subsection{Adaptive Continuity, Copyability, and Reconstruction}
\label{app:audit-adaptive}

A second experiment introduced an endogenous rank-\(8\) writable governance
state at layers selected near \(0.33\), \(0.66\), and \(0.90\) of model
depth. Three update rates were evaluated:
\[
\eta\in\{0.01,0.02,0.05\}.
\]
Six prompts were tested for each core model, with \(16\) Monte Carlo
rollouts and the same future horizons
\[
1,\;2,\;4,\;8.
\]

Four continuation protocols were compared:
\begin{description}
    \item[\texttt{frozen\_live}] the adaptive state was disabled;
    \item[\texttt{persistent\_live}] the endogenous governance state
    continued within the live realization;
    \item[\texttt{persistent\_copy}] the live realization continued while
    a numerically identical unused copy of its governance state was created;
    \item[\texttt{reconstructed}] the continuing state was restored from
    a detached record rather than inherited through the same implementation
    history.
\end{description}

The comparison between \texttt{persistent\_live} and
\texttt{persistent\_copy} is a direct test of the distinction
\[
\textit{copyability}
\not\Rightarrow
\textit{delegability}.
\]
Across the four reported core models, six prompts, and four horizons,
there were \(96\) model--prompt--horizon equality checks. The absolute
difference in output Jensen--Shannon divergence between the live and copied
conditions was
\[
0
\]
in every case. Merely creating a copy did not alter the continuing
realization.

Reconstructed continuation was also numerically close to live continuation.
Across those \(96\) checks, the mean absolute live--reconstructed
difference was
\[
0.00129,
\]
the median was approximately
\[
5.88\times10^{-5},
\]
and the largest observed difference was
\[
0.01566.
\]
Thus a detached state record can preserve the measured computational
behavior extremely closely even though CLT assigns a different causal
interpretation to how that state is inherited.

\subsection{Candidate Causal-Carrier Search}
\label{app:audit-carriers}

To avoid assuming the bearer boundary in advance, the audit searched for
internal sequence-state subsets whose transplantation was sufficient to
recover a specified fraction of the \(d_1\) effect. Transformer cache layers
were partitioned into
\[
2,\;3,\;4,\;6
\]
contiguous groups. Candidate subsets were transplanted under a common future
suffix and tested for inclusion-minimality. The primary mediation threshold
was
\[
\theta=0.8,
\]
with sensitivity analyses at
\[
0.7
\qquad\text{and}\qquad
0.9.
\]

Layerwise cache transplantation was supported for Qwen, Phi, and Llama.
Zamba2 exposed only a compatible whole-sequence-state transport in the
portable cache interface and was therefore excluded from layerwise
minimal-carrier statistics.

At the primary threshold, a minimal threshold-passing candidate was recovered
for every one of the six tested prompts in Qwen, Phi, and Llama under every
partition granularity. Table~\ref{tab:carrier-primary} summarizes the finest
six-group partition and two multiscale robustness statistics.

\begin{table}[t]
\centering
\small
\caption{Primary carrier audit at mediation threshold \(\theta=0.8\).
Layer fraction is the mean fraction of physical cache layers in the
representative minimal candidate under the finest six-group partition.
Jaccard measures agreement of representative physical layer sets across
partition schemes. \(\Pi_{\mathrm{uniform}}\) is the audit-relative
persistence statistic defined below.}
\label{tab:carrier-primary}
\begin{tabular}{lcccc}
\hline
Model & Recovered & Layer fraction & Jaccard & \(\Pi_{\mathrm{uniform}}\) \\
\hline
Qwen3-4B-Base & \(6/6\) & 0.389 & 0.636 & 0.794 \\
Phi-4-mini    & \(6/6\) & 0.417 & 0.710 & 0.627 \\
Llama-3.2-3B & \(6/6\) & 0.381 & 0.650 & 0.636 \\
\hline
\end{tabular}
\end{table}

The mean minimal layer fraction decreased as the partition became finer.
At \(\theta=0.8\), the cross-model mean was approximately \(0.611\) under
two groups, \(0.497\) under three groups, \(0.417\) under four groups, and
\(0.396\) under six groups. This dependence is expected because finer
partitions permit more precise localization.

Partition robustness was substantial without being perfect. At the primary
threshold, mean cross-partition Jaccard similarity ranged from \(0.636\) to
\(0.710\). This argues against treating one exact layer boundary as an
architecture-independent fact.

For a second robustness analysis, each inclusion-minimal candidate recovered
under the finest partition was projected into every coarser declared
partition. Define
\[
\Pi_{\mathrm{uniform}}
=
\frac{
\text{number of partitions in which the projected set remains minimal}
}{
\text{number of supported declared partitions}
}.
\]
This quantity is an operational analogue of the multiscale persistence
term \(\Pi_{\nu}\) in the formal development. It is audit-relative and
should not be treated as a directly measured physical observable. At the
primary threshold, mean values were \(0.794\) for Qwen, \(0.627\) for Phi,
and \(0.636\) for Llama.

Sensitivity analyses at thresholds \(0.7\) and \(0.9\) retained candidate
recovery in all tested transformer prompts. Increasing the threshold
generally enlarged the minimal candidate and increased several robustness
statistics, as expected when stronger mediation is demanded.

\subsection{Soft and Hard Reconstruction}
\label{app:audit-reconstruction}

The reconstruction experiments address a central CLT distinction:
computational equivalence need not determine constitutive lineage.

In the ordinary reconstruction audit, six prompts were tested for every
reported mechanistic model. A detached record was used to reconstruct the
continuation, after which logits and hidden states were compared with live
continuation. Across the eight reported models, mean logit
Jensen--Shannon divergence by model ranged from
\[
3.18\times10^{-7}
\]
to
\[
2.08\times10^{-5}.
\]
Mean hidden-state RMSE ranged from approximately
\[
0.00143
\]
to
\[
0.00647.
\]
Across all \(48\) reported model--prompt cases, median logit divergence was
approximately
\[
6.03\times10^{-7},
\]
and the largest observed logit divergence was below
\[
3.0\times10^{-5}.
\]

A stronger process-level reconstruction test removed ambiguity about whether
the same operating-system realization had simply remained alive. For Qwen,
Phi, Llama, and Zamba, three prompts each were subjected to the following
protocol:

\begin{enumerate}
    \item the source realization computed the relevant continuation state;
    \item its sequence cache and governance record were serialized;
    \item the source operating-system process terminated;
    \item a distinct successor process was instantiated;
    \item the successor restored the detached record and resumed computation.
\end{enumerate}

All \(12\) reported trials used distinct source and successor process IDs,
and every source was recorded as terminated before successor resumption.
All \(12\) used the exact-cache reconstruction path. Governance-state hashes
were identical before and after serialization. For every trial,
\[
D_{\mathrm{JS}}=0
\qquad\text{and}\qquad
\mathrm{RMSE}_{\mathrm{hidden}}=0.
\]

The protocol assigns
\[
\mathcal{C}=1,\qquad\mathcal{N}=1
\]
to the live continuation and
\[
\mathcal{C}=0,\qquad\mathcal{N}=0
\]
to the reconstructed continuation because the source realization has ended
and persistence is mediated by a detached external record. These labels are
part of the experimental operationalization of CLT-I. They are not inferred
from the numerical equality result.

The result therefore establishes a clean dissociation:
\[
\text{exact computational-state preservation}
\not\Rightarrow
\text{same implemented causal lineage}.
\]
Whether this causal distinction is constitutive of phenomenality is the
separate CLT-II conjecture.

\subsection{Post-Training Comparison Within OLMo-2}
\label{app:audit-olmo}

The OLMo-2 sequence provides a within-family test of whether post-training
changes the measured downstream counterfactual structure. Table
\ref{tab:olmo-posttraining} reports mean identity-map SWD.

\begin{table}[t]
\centering
\small
\caption{Mean identity-map sliced Wasserstein divergence across the OLMo-2
post-training sequence.}
\label{tab:olmo-posttraining}
\begin{tabular}{lcccc}
\hline
Checkpoint & \(h=1\) & \(h=2\) & \(h=4\) & \(h=8\) \\
\hline
Base     & 0.844 & 0.584 & 0.358 & 0.253 \\
SFT      & 0.884 & 0.660 & 0.385 & 0.265 \\
DPO      & 0.866 & 0.679 & 0.411 & 0.289 \\
Instruct & 0.932 & 0.730 & 0.471 & 0.305 \\
\hline
\end{tabular}
\end{table}

Relative to the base checkpoint, the Instruct model showed larger mean SWD
by approximately \(10.5\%\), \(25.0\%\), \(31.6\%\), and \(20.6\%\) at
horizons \(1,2,4,8\), respectively. SFT and DPO produced smaller but still
measurable changes. The causal effect therefore survives post-training
while its magnitude changes.

This result shows that post-training history
modifies the measured counterfactual geometry of a model family. 



\subsection{Numerical Integrity and Provenance}
\label{app:audit-integrity}

The complete production archive recorded twelve model statuses and zero
model-level failures. The major numerical outputs contained no infinities.
The reported
hidden-state, activation-patching, adaptive-continuity, reconstruction, and
aggregate summary measures were finite.

The principal pre-specified configuration included:

\[
\begin{array}{ll}
\text{Monte Carlo rollouts} & 64,\\
\text{temperature} & 0.85,\\
\text{future horizons} & 1,2,4,8,\\
\text{SWD projections} & 128,\\
\text{RFF-MMD features} & 256,\\
\text{carrier threshold} & 0.8,\\
\text{threshold sensitivity} & 0.7,0.8,0.9,\\
\text{carrier partitions} & 2,3,4,6,\\
\text{governance rank} & 8,\\
\text{master seed} & 20260905.
\end{array}
\]

\subsection{Limits of the Computational Evidence}
\label{app:audit-limits}

Several limitations delimit what can be inferred.

First, the experiment used one master experimental seed with deterministic
substreams. The Monte Carlo sample within each condition is substantial,
but an independent multi-master-seed replication would provide an
additional robustness test.

Second, the candidate-carrier result is intervention-relative and
coarse-graining-relative. The observed partition dependence is itself
informative: CLT-I requires robustness across admissible descriptions
rather than identification of one privileged layer index by stipulation.

Third, Zamba's cache representation permitted whole-sequence-state
transport but not the same layerwise carrier decomposition available for
the transformer models. Carrier-localization results are therefore reported
only for Qwen, Phi, and Llama.

Fourth, reconstruction equivalence establishes that a detached record can
preserve measured computational organization extremely closely, and in the
hard process experiment exactly under the recorded metrics. The assignment
of constitutive continuity and non-delegable inheritance remains a claim
about the implemented causal protocol. The experiment demonstrates that
this protocol distinction is empirically separable from computational-state
identity.

Fifth, the OLMo post-training comparison shows that training history changes
the magnitude of downstream counterfactual effects. No consciousness
variable was measured, so these differences cannot be interpreted as
changes in phenomenality.

Finally, the audited systems are methodological testbeds for CLT-I rather
than independently validated conscious or unconscious reference cases.
Consequently, the experiments bear directly on the empirical tractability
of causal bearer individuation. They leave the CLT-II biconditional
untouched. Evaluating that constitutive conjecture requires convergence
with cases in which phenomenal consciousness is supported independently
of CLT.

\begin{figure}[t]
\centering
\includegraphics[width=\textwidth]{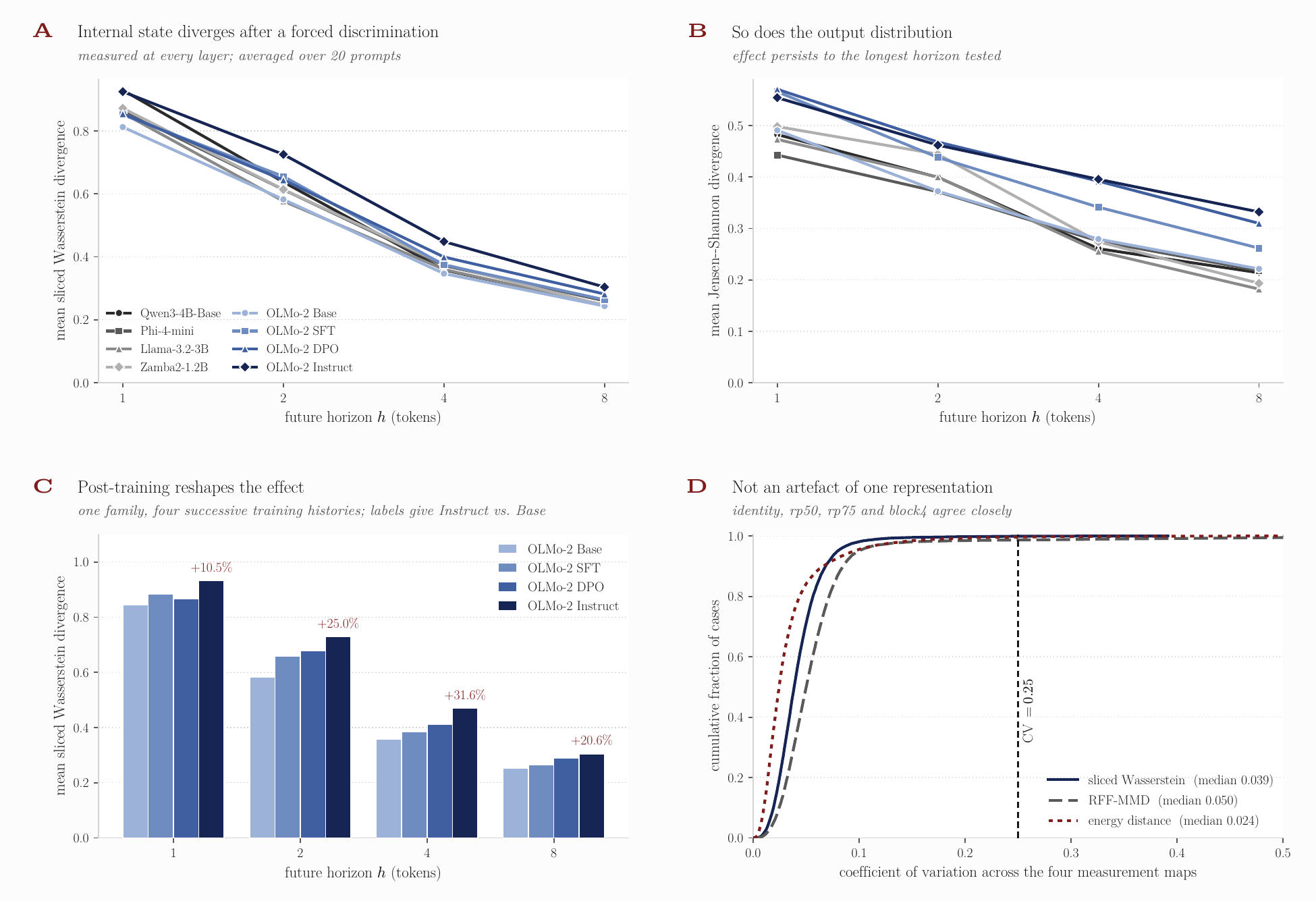}
\caption[Downstream effects of a forced discrimination]{%
\textbf{The discrimination has a future.}
\emph{A}~Forcing one of the two highest-probability non-special next-token
alternatives produces a measurable divergence in subsequent hidden-state
distributions. The panel reports identity-map sliced-Wasserstein divergence,
averaged across sampled layers and 20 prompts, for the eight-model mechanistic
panel. The effect decreases with future horizon but remains present at
$h=8$ in every model shown.
\emph{B}~The corresponding output distributions also diverge, measured by
Jensen--Shannon divergence, with a non-zero effect persisting to the longest
horizon tested.
\emph{C}~Within the OLMo-2 sequence, Base, SFT, DPO, and Instruct checkpoints
retain the counterfactual effect while its magnitude changes systematically
across post-training histories; annotations report Instruct relative to Base.
This panel uses the OLMo deep-layer aggregate, whereas panels~A and~B report
the broader all-layer summaries.
\emph{D}~Empirical cumulative distributions of the coefficient of variation
across the four pre-specified measurement maps
(identity, rp50, rp75, and block4) show close agreement for sliced Wasserstein,
RFF-MMD, and energy distance. The vertical reference marks
$\mathrm{CV}=0.25$; the corresponding median CVs are 0.039, 0.050, and 0.024.
The result therefore survives substantial changes in the measurement map.
These analyses concern the operational distinctions of CLT-I and provide no
evidence for CLT-II.}
\label{fig:audit-effect}
\end{figure}

\begin{figure}[t]
\centering
\includegraphics[width=\textwidth]{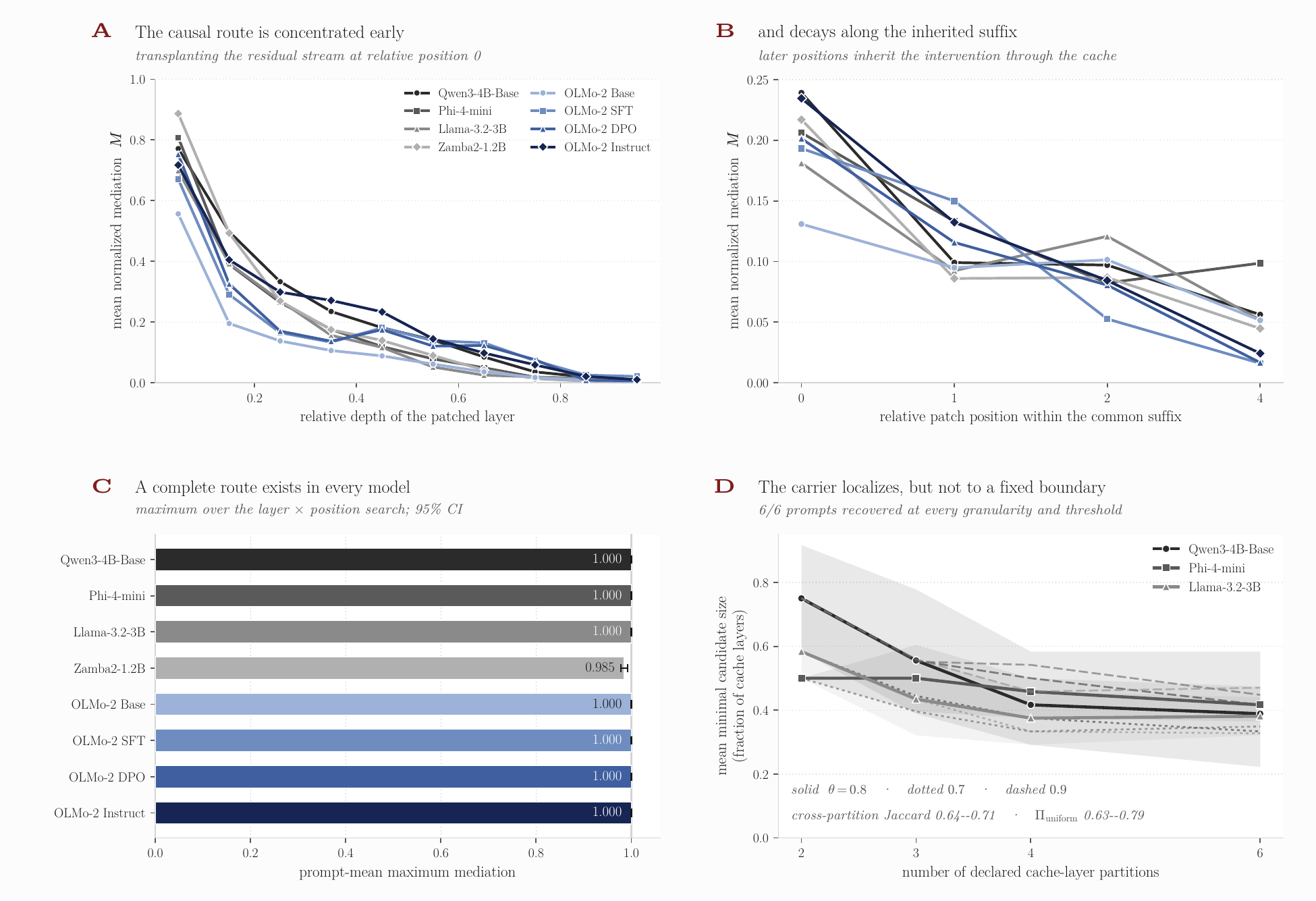}
\caption[Activation patching and minimal carrier recovery]{%
\textbf{Where the consequence travels.}
\emph{A}~Normalized mediation $M$ after transplanting the residual stream
from the $d_1$ trajectory into the corresponding $d_2$ trajectory, shown as a
function of relative layer depth for patches at the first common-suffix
position. Mediation is concentrated in early layers and approaches zero
toward the final layers.
\emph{B}~At a fixed causal route, mediation generally decreases as the patch
is moved farther along the common suffix, consistent with later positions
already inheriting part of the intervention through the cached computation.
\emph{C}~The prompt-mean maximum over the full layer$\times$position search is
$1.000$ in seven of the eight mechanistic models and $0.985$ in Zamba2-1.2B,
showing that a near-complete causal route can be recovered even though no
single layer carries the effect throughout.
\emph{D}~For the Qwen, Phi, and Llama carrier-partition audit, the mean size of
the inclusion-minimal threshold-passing candidate is shown as the cache-layer
partition is refined. Solid curves give the primary threshold
$\theta=0.8$, with dotted and dashed curves showing $\theta=0.7$ and
$\theta=0.9$. A candidate was recovered for all $6/6$ prompts at every
reported granularity and threshold. Its extent changes under repartitioning,
while cross-partition overlap and the $\Pi$-uniformity analysis remain
substantial. The result supports a coarse-graining-relative carrier criterion
rather than a privileged layer boundary.}
\label{fig:audit-carrier}
\end{figure}

\begin{figure}[t]
\centering
\includegraphics[width=\textwidth]{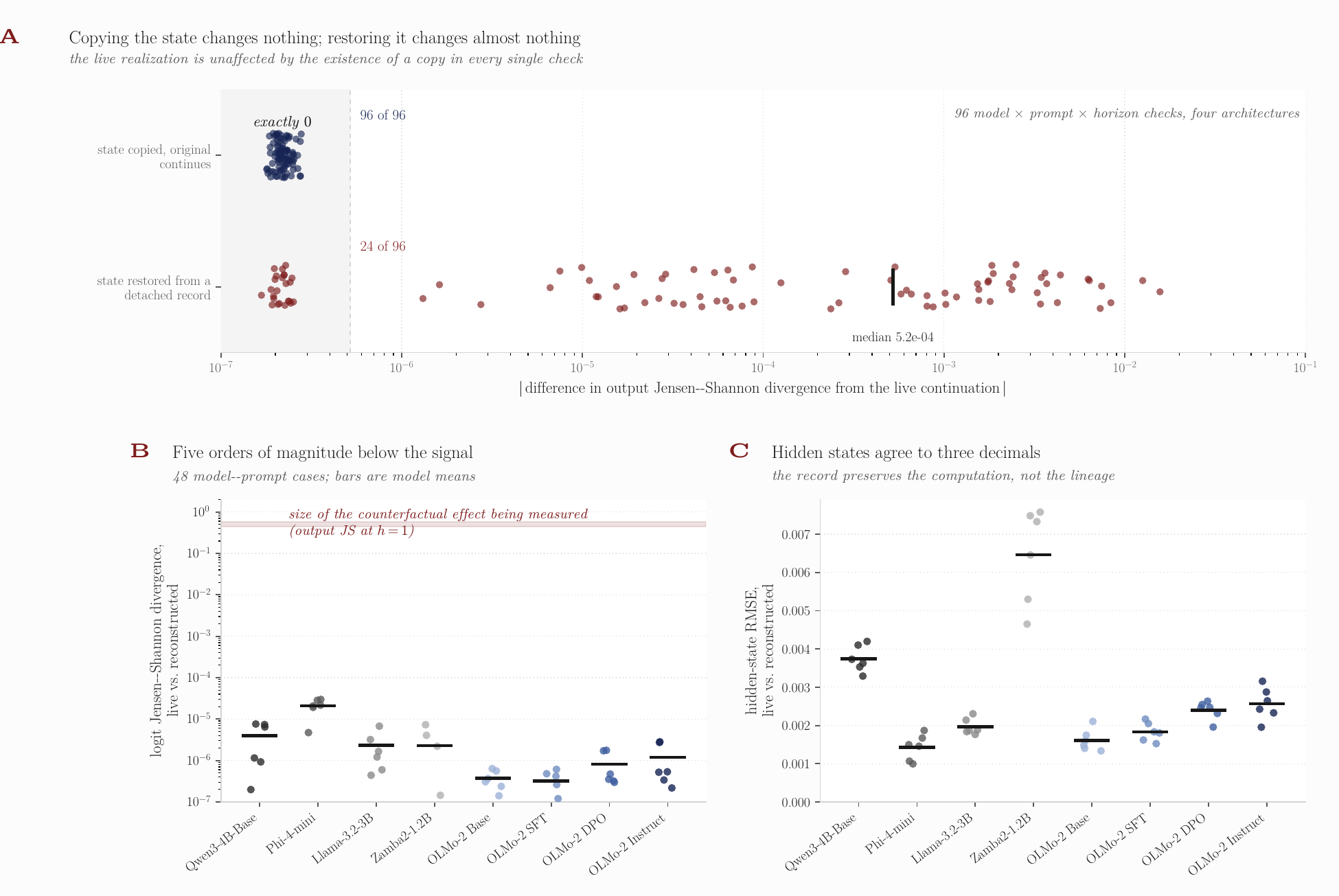}
\caption[Copy and reconstruction equivalence]{%
\textbf{Copyability is not delegability.}
\emph{A}~Across 96 model$\times$prompt$\times$horizon checks in the eight-model
mechanistic panel, introducing an unused numerically identical copy of the
governance state changes the live continuation by exactly zero in every case.
Restoring the same state from a detached record produces only a small
difference: 24 of 96 reconstruction checks are also exactly zero, and the
median absolute difference in output Jensen--Shannon divergence is
$5.2\times10^{-4}$.
\emph{B}~Logit Jensen--Shannon divergence between the live and reconstructed
continuations is shown against the magnitude of the counterfactual effect
being measured at $h=1$. Across 48 model--prompt cases, reconstruction error
lies approximately five orders of magnitude below the causal signal.
\emph{C}~Hidden-state reconstruction error is likewise small, with
model-level RMSE values on the order of $10^{-3}$. Thus a detached record can
preserve the relevant computational state to high numerical accuracy even
when the reconstruction protocol assigns a different causal history to the
continuation. The comparison operationalizes the distinction between
copyability and non-delegable inheritance in CLT-I; it does not test CLT-II.}
\label{fig:audit-copy}
\end{figure}

\end{document}